\documentclass{article}
\usepackage[final]{neurips_2024}

\input{packages.tex}
\newtheorem{assump}{Assumption}

\crefname{assump}{assumption}{assumptions}

\theoremstyle{plain}
\newtheorem{theorem}{Theorem}[section]

\newtheorem{lemma}[theorem]{Lemma}
\newtheorem*{lemma-non}{Lemma}

\newtheorem*{theorem-non}{Theorem}

\theoremstyle{definition}

\theoremstyle{remark}

\DeclareMathOperator*{\argmin}{arg\,min}

\DeclareMathOperator{\msum}{\medmath\sum}
\DeclareMathOperator{\mprod}{\medmath\prod}
\DeclareMathOperator{\mint}{\medmath\int}

\newcommand{\argdot}{{\,\vcenter{\hbox{\tiny$\bullet$}}\,}}%{\bullet}
\newcommand{\tagaligneq}{\refstepcounter{equation}\tag{\theequation}}

\newtheorem*{remark*}{Remark}

\usepackage{nccmath,enumitem,tikz}

\newenvironment{proplist}{\begin{enumerate}[
    leftmargin=2.5em,
    labelwidth=0em,
    label=(\roman{enumi}),
    topsep=.1em,
    partopsep=0em,
    itemsep=-.3em
    ]}
{\end{enumerate}}

\def\P{{\mathbb P}}
\def\E{{\mathbb E}}

\def\Cov{\text{\rm Cov}}

\def\Tr{\text{\rm Tr}}

\def\E{{\mathbb E}}

\def\N{\mathbbm{N}}

\def\P{{\mathbbm P}}

\def\R{\mathbbm{R}}

\def\cA{\mathcal{A}}

\def\cD{\mathcal{D}}

\def\cF{\mathcal{F}}

\def\cP{\mathcal{P}}

\def\cX{\mathcal{X}}

\newcommand{\msup}{\sup\nolimits}
\newcommand{\minf}{\inf\nolimits}
\newcommand{\mmax}{\max\nolimits}
\newcommand{\mmin}{\min\nolimits}
\newcommand{\ind}{\mathbb{I}}

\newcommand{\mean}{\mathbb{E}}

\title{Near-Optimality of Contrastive Divergence Algorithms}

\author{%
  Pierre Glaser \qquad Kevin Han Huang \qquad Arthur Gretton \\
Gatsby Computational Neuroscience Unit, University College London\\
  \texttt{pierreglaser@gmail.com, han.huang.20@ucl.ac.uk, arthur.gretton@gmail.com} \\
}

\begin{document}

\maketitle

\begin{abstract}
We perform a non-asymptotic analysis of the contrastive divergence (CD) algorithm, a training method for unnormalized models. While prior work has established that (for exponential family distributions) the CD iterates asymptotically converge at an $O(n^{-1 / 3})$ rate to the true parameter of the data distribution, we show, under some regularity assumptions, that CD can achieve the parametric rate $O(n^{-1 / 2})$. Our analysis provides results for various data batching schemes, including the fully online and minibatch ones. We additionally show that CD can be near-optimal, in the sense that its asymptotic variance is close to the Cramér-Rao lower bound.
% We illustrate our results on a set of synthetic examples.
\end{abstract}

\section{Introduction} 
Describing data using probability distributions is a central task in multiple scientific and industrial disciplines \cite{cranmer_frontier_2020,radford2021learning,radford2019language}.
Since the true distribution of the data is generally unknown, such a task requires finding an estimator  of the true distribution
among a model class that best describes the available data.  
An estimator can be characterized at multiple levels of granularity:
at the highest level lies \emph{consistency} \cite{van2000asymptotic}, a property which states that as the number of available data
points increases, a given estimator will converge
to the one best describing the data distribution. 
At a lower level, a consistent estimator can be further characterized by its convergence rate, a quantity upper--bounding its distance
to the true distribution as a function of the number of samples. A convergence rate can be either \emph{asymptotic}, e.g. hold only
in the limit of an infinite sample size, or \emph{non-asymptotic}, in which case the rate also holds for finite sample sizes. 
In their simplest form, convergence rates are provided in big--$ O $ notation, %e.g. $ O(n^{-\alpha}) $ from some $ \alpha > 0 $,
discarding finer grained information such as asymptotically dominated quantities %by $ O(n^{-\alpha}) $
as well as multiplicative constants. These constants play a role in the so--called \emph{asymptotic variance} of the estimator,
which is a precise descriptor of an estimator's statistical efficiency.
Convergence rates and asymptotic variances have been the subject of extensive research in the statistical literature;
in particular, well--known lower bounds exists regarding both the best possible (asymptotic) convergence rate of an estimator
and its best possible asymptotic variance. These results set a clear frame of reference to interpret individual convergence rates,
and are routinely present in the analysis of modern statistical algorithms such as noise-contrastive estimation
\cite{lee2023pitfalls,gutmann2010noise} or score matching \cite{hyvarinen2005estimation,koehler2022statistical,pabbaraju2024provable}.

In this work, we focus on cases where (1) the true data distribution admits a density with respect to some  known base measure,
and (2) the model class is parametrized by a finite-dimensional parameter. In this setting, provided that the true distribution belongs to the model class,
a celebrated result in statistical estimation states that the model maximizing the average log-likelihood
both achieves the best possible asymptotic convergence rate (called the \emph{parametric rate}) and the best possible asymptotic variance,
called the Cramér-Rao bound (see, e.g. \cite{casella2021statistical}). While this result shows that Maximum Likelihood Estimators (MLE) are asymptotically optimal,
fitting them is complicated by computational hurdles when using models with intractable normalizing constants. 
Such \emph{unnormalized models} are common in the Machine Learning literature due to their high flexibility \cite{le2008representational,du_implicit_2020};
their weakness however lies in the fact that expectations under 
these models have no unbiased approximation.
For this reason,
popular approximation algorithms such as unbiased gradient-based stochastic optimization of the empirical log-likelihood
cannot \emph{a priori} be used, as the gradient of the normalizing constant is given by an expectation under the model distribution.

The Contrastive Divergence (CD) algorithm \cite{hinton2002training} is a popular approach that circumvents this issue by using a 
Markov Chain Monte Carlo (MCMC) algorithm to approximate the gradient of the log-likelihood. Unnormalized models trained with Contrastive Divergence have been shown
to reach competitive performance in high-dimensional tasks such as image \cite{nijkamp2020anatomy, nijkamp2019learning, du_improved_2021}, text \cite{qin2022cold},
and protein modeling \cite{kong2023molecule,varnai2013efficient}, or neuroscience \cite{glaser2022maximum}. 
A consistency analysis of the Contrastive Divergence algorithm is delicate, however:
indeed, the \emph{optimization error} e.g. the difference between the estimate returned by CD and the MLE, is likely to be non-negligible
as compared with
\emph{statistical error} -- the distance between the MLE and the true distribution -- and thus cannot be discarded,
as often done when analyzing estimators that minimize tractable objectives \cite{koehler2022statistical, lee2023pitfalls}.
Recent work \cite{jiang2018convergence} elegantly established asymptotic $ O(n^{-1 / 3}) $--consistency of the CD estimator
for unnormalized exponential families when using only a \emph{finite} number of MCMC steps.
Key to their argument is the fact that the bias of the CD gradient estimate decreases as iterates approach the data distribution.
However, as noted by the authors, their work left open the question of whether and under what conditions CD might achieve
$ O(n^{-1 / 2}) $--consistency.

\vspace{-0.5em}

\paragraph{Contributions}

In this work, 
we answer this question by providing a non-asymptotic analysis of the 
CD  % Contrastive Divergence 
algorithm for unnormalized
exponential families. While existing convergence bounds \cite{jiang2018convergence} were derived for the ``full batch'' setting, where
the CD gradient is estimated using the full dataset at each iteration, our analysis covers both the online setting
(where data points are processed one at a time without replacement), and the offline setting with multiple data reuse strategies
(including full batch).

In the online case (Section \ref{sec:online_cd}), we show, under a restricted set of assumptions compared to 
Jiang et al.~\cite{jiang2018convergence},
that the CD iterates can converge to the true distribution at the parametric $ O(n^{-1 / 2}) $ rate.
Our analysis reveals that CD contains two sources of approximation: a bias term, and a variance term. These sources
are almost independent of each other, in the sense that decreasing the bias by increasing the number of MCMC steps will not decrease the variance.
The impact of these two sources of approximation transparently propagates in our resulting bounds: in particular, as the bias
of the CD algorithm goes to 0, our bounds recover well-known results in online stochastic optimization \cite{moulines2011non}. Finally, we study the
asymptotic variance of an estimator obtained by averaging the CD iterates, a classic acceleration technique in stochastic optimization \cite{polyak1992acceleration}.
We show that provided that the number of steps $ m $ is sufficiently large, the \emph{asymptotic} variance of this estimator matches 
(up to a factor 4) the Cramér-Rao bound.

Next, we study the offline setting (Section \ref{sec:offline:SGD}),
where the CD gradient is estimated by reusing (potentially random) subsets of a finite dataset.
We show that a similar result to the online setup holds, up to an additional correlation term that arises from data reuse, and present several approaches to control this term. 
We improve over the results of \cite{jiang2018convergence} by showing a non-asymptotic and 
near-parametric rate at $ O((\log n)^{1/2} n^{-1/2})$ under their conditions, and also illustrate how different rates can be obtained under a variety of conditions.
Our results also show an interesting tradeoff between the effect of initialization and the statistical error as a function of batch size.

In summary, we establish the near--optimality of a variety of Contrastive Divergence algorithms for unnormalized exponential families in the so called ``long-run'' regime,
where the number of MCMC steps is high enough to ensure that the CD gradient bias is sufficiently offset by the convexity of the negative log-likelihood.

\section{Contrastive Divergence in Unnormalized Exponential Families}
\vspace{-0.5em}
{\bfseries Unnormalized Exponential Families}
Exponential families (EF) \cite{brown1986fundamentals,wainwright2008graphical} form a well-studied class of probability distributions, given by
\begin{equation} \label{eq:expfam}
\begin{aligned}
	\textstyle p_{\psi}(\mathrm{d}x) \coloneqq  e^{\psi^{\top}\phi(x) - \log Z(\psi)}c(\mathrm{d}x), \quad Z(\psi) \coloneqq \int_{ \mathcal  X }^{  } e^{\psi^{\top}\phi(x)}c(\mathrm{d}x).
\end{aligned}
\end{equation}
Here, $ \mathcal X \ni x $ is the \emph{data} or \emph{sample space}, which we set to be a subset of $ \mathbb{R}^d $ for some $ d \in \mathbb{N}^* $,
although our results are readily extendable to more general measurable spaces.
$ c $ is a measure on $ \mathcal  X $ called the \emph{base} or \emph{carrier} measure. When $ \mathcal X \subseteq \mathbb{R}^d $,
$ c $ is often set to be the corresponding Lebesgue measure. $ \psi \in \Psi \subseteq \mathbb{R}^p $ is a finite-dimensional parameter called the \emph{natural parameter},
and $ \phi: \mathbb{R}^d  \longmapsto \mathbb{R}^p $ is a function called the \emph{sufficient statistics}, which, alongside with the base measure, fully describes an 
exponential family. Finally, $ \log Z(\psi) $, the \emph{log--normalizing} (or \emph{cumulant}) function, is a quantity ensuring that $ p_{\psi} $ integrates to $ 1 $ over $ \mathcal  X $. Crucially, we will not assume that $ \log Z(\psi) $ admits a closed form  expression for all $ \psi $. The latter fact provides the practitioner with a great deal of
flexibility in designing the model class: indeed, the only requirement that should be satisfied prior to performing statistical estimation is to have $ Z(\psi)< +\infty $
for all $ \psi $, something that can be readily verified and is often the case in practice. The drawback of unnormalized EFs is the fact that sampling
(and thus approximating expectations under the model) cannot usually be performed in an unbiased manner. Instead, inference in unnormalized EFs is often performed using tools from
the Bayesian Inference literature, such as MCMC \cite{geyer2011introduction}.
Unnormalized EFs belong to the larger class of
\emph{unnormalized models} \cite{lecun2006tutorial,song2021train,hyvarinen2005estimation,gutmann2010noise}, of the form $ e^{-E_{ \psi}(x) - \log Z(\psi)} c(\text{d}x),\,\,Z(\psi)= \smallint e^{-E_{ \psi}(x)} c(\text{d}x) $,
for some parametrized function $E_{ \psi}: \mathbb{R}^d  \longmapsto \mathbb{R}$ referred to as the \emph{energy}.
Unnormalized models thus take the flexibility of unnormalized EFs one step further by allowing the (negative) unnormalized log--density to be an arbitrary function $ E_{ \psi} $
of $ x $ and $ \psi $, instead of requiring a linear dependence on $ \psi $ as in Equation \ref{eq:expfam}.
We focus in this work on unnormalized EFs due to the multiple computational benefits they provide, as explained in the next section,
but we believe that extending our analysis to more general unnormalized models is an interesting avenue for future work.

{\bfseries Statistical Estimation in Unnormalized Exponential Families using Contrastive Divergence}
We now review the Contrastive Divergence algorithm, an algorithm used to fit unnormalized models, and our main object of study in this work.
The general setting is the following: we assume access to $ n $ i.i.d.~samples  $ (X_1, \dots, X_n)$ drawn from some unknown distribution
$ {p}^{\star}$, which we assume belongs to $ \mathcal  P_{ \psi} $, e.g. $ {p}^{\star} = p_{\psi^{\star}} $ for some $ \psi^{\star} \in \Psi $.
Given these samples, we aim to perform \emph{statistical estimation}, e.g. find a parameter $ \psi_n $ within $ \Psi $ that should
approach $ {\psi}^{\star} $ as $ n $ grows.

The starting point of the Contrastive Divergence algorithm is the unfortunate realization that Maximum Likelihood Estimation, which corresponds to minimizing the cross-entropy 
$ \mathcal L(\psi) \coloneqq -\mathbb{ E }_{ p_n} \log \mathrm{d} p_{ \psi} / \mathrm{d} c$ between the model $ p_{\psi} $ and the empirical data distribution
$ p_n \coloneqq 1 / n \sum_{ i=1 }^{ n } \delta_{X_i} $, cannot be performed using exact (possibly stochastic) gradient-based optimization,
as the gradient $ \nabla_{ \psi } \mathcal  L(\psi) $ of $ \mathcal  L $ with respect to the parameter $ \psi $ contains an expectation under the model distribution $ p_{\psi} $.
Indeed, the cross entropy and its gradient are given by
\begin{equation} \label{eq:nll-and-gradient}
\begin{aligned}
\begin{cases}
	\textstyle \mathcal  L(\psi) &= -\frac{1}{n}\sum_{ i=1 }^{ n } \phi(X_i)^{\top}\psi + \log Z(\psi) \\
	\textstyle \nabla_{ \psi } \mathcal  L(\psi) &= -\frac{1}{n}\sum_{ i=1 }^{ n }\phi(X_i) + \mathbb{ E } \left \lbrack \phi(X^{\psi}) \right \rbrack, \quad X^{\psi} \sim p_{\psi}\,.
\end{cases} 
\end{aligned}
\end{equation}
The second line follows from the well known identity $ \nabla_{ \psi } \log Z(\psi) \coloneqq \mathbb{ E } \left \lbrack \phi(\cramped{X^{\psi}}) \right \rbrack $;
we refer to \citep[Proposition~3.1]{wainwright2008graphical} for a proof. The Contrastive Divergence algorithm circumvents this
issue by running approximate stochastic gradient descent (SGD) on $ \mathcal  L $, where the intractable expectation in $ \nabla_{ \psi } \log Z $ is 
estimated using an MCMC algorithm initialized at the empirical data distribution.
In more details, given a number of epochs $ T $, a sequence of data batches 
$ B_{t, j} $ of size $ B $ 
(e.g. $ \textstyle B_{t, j} \in [\![1, n]\!]^{B}, 1 \leq t \leq T, 1 \leq j \leq N \lceil n / B \rceil $),
and a family of \emph{Markov kernels} $ \{  k_{\psi}, \psi \in \Psi \} $ each with invariant distribution $ p_{\psi} $,
at the $ j^{th} $ minibatch of epoch $ t $, % for the where each
$ \nabla_{ \psi } \log Z(\psi_{t, _{j-1}})$ is approximated by $ \frac{1}{B} \sum_{ i \in B_{t, j} } \phi(\tilde{X}^{m}_i) $, 
where $ \tilde{X}_{i}^{m} $ is produced by running the recursion $ \tilde{X}_i^{k} \sim k_{\psi_t}(\tilde{X}_i^{k-1}, \cdot),\,\, \tilde{X}_i^{0} = X_i $ up to $ k = m $.
Throughout the paper, we will refer to the conditional distribution of $ \tilde{X}^{m}_i $ given $ X_{i}$ as 
$ k^{m}_{\psi}(X_{i}, \cdot) $.
The resulting gradient estimate arising from combining this approximation with the other (tractable) sum over the data samples 
present in $\nabla_{ \psi } \mathcal  L(\psi)$, which we refer to as the \emph{CD gradient} and denote as $ h_t $, 
is thus 
% given by:
\begin{equation} \label{eq:generic_cd_gradient}
\begin{aligned}
	h_{t, j} \coloneqq \mfrac{1}{B}\msum_{ i \in B_{t, j} } \phi(X_i) - \mfrac{1}{B}\msum_{ i \in B_{t, j} }\phi(\tilde{X}^{m}_i) = \mfrac{1}{B} \msum_{ i \in B_{t,j} } \left ( \phi(X_i) - \phi(\tilde{X}^{m}_i) \right ).
\end{aligned}
\end{equation}
Key to the behavior and analysis of the CD algorithm is the strategy employed to generate minibatches $ B_{t, j} $.
The case where $ T = 1 $, $ B = 1 $, and  $ B_{1, j} = \{ j \} $ 
will be referred to as \emph{online} CD, while the variant where $ T > 1 $, and each batch $ B_{t, j} $ draws $ B $
indices (with or without replacement) from $ [\![1, n]\!] $ will be referred to as \emph{offline} CD.
In online CD, each data point is present in one and one batch only, while in offline CD, data points are reused across batches.
From a statistical perspective, we will see that online CD can be analyzed in a remarkably simple way, while offline
CD introduces additional correlations that require care to be controlled.
Both settings come with their advantages and drawbacks, as we will see in the next section.
The CD algorithms we study will employ decreasing step size schedules $ (\eta_t)_{t \geq  0} $ of the form $ \eta_t = C t^{-\beta} $,
where $ C > 0 $ is the initial leaning rate and $ \beta \in [0, 1]$.
We lay out online CD and offline CD in Algorithms \ref{alg:online_cd} and \ref{alg:offline_cd}.
Note that our algorithms include a projection step on the parameter space $ \Psi $ to account for the case where
$ \Psi $ is compact. In the case $ \Psi = \mathbb{R}^p $, this step can be omitted.
Next we depart from the setting of \cite{jiang2018convergence} and start by analyzing online CD.

\vspace{-1em}

\begin{minipage}{0.46\textwidth}
\begin{algorithm}[H]
   \caption{Online CD}
   \label{alg:online_cd}
\begin{algorithmic}
   \STATE {\bfseries Input:} $(X_1, \dots, X_n) \stackrel{\text{i.i.d.}}{\sim} p_{ {\psi}^{\star}}$
   \STATE \textbf{Parameters:}  
   Model class $ \{ p_{\psi}, \psi \in \Psi \} $,
   Markov kernels $ \{ k_{ \psi}, \psi \in \Psi \} $,
   number of MCMC steps $ m $,
   learning rate schedule $ \eta_t \coloneqq C t^{-\beta} $, $ \beta \in \lbrack 0, 1 \rbrack  $,
   $ C > 0 $, initial parameter $ \psi_0 $
   \FOR{$ t = 1, \dots, n $}  %\texttt{//loop over epochs}
	\STATE \texttt{//Approx.~sample from} $ p_{\psi_{t-1}} $
	\STATE $ \tilde{X}^m_t \sim k_{\psi_{t-1}}^m(X_t, \cdot) $
	\STATE $ h_t \coloneqq \phi(X_t) - \phi(\tilde{X}^m_t) $
	\STATE $ \psi_{t} \leftarrow \psi_{t-1} - \eta_t h_t $
	\STATE $ \psi_{t} \leftarrow \text{Proj}_{\Psi}(\psi_{t}) $
	% \ENDIF
   \ENDFOR
   \RETURN $ \psi_n $
\end{algorithmic}
\end{algorithm}
\end{minipage}
\hfill
\begin{minipage}{0.46\textwidth}
\begin{algorithm}[H]
   \caption{Offline CD}
   \label{alg:offline_cd}
\begin{algorithmic}
   \STATE {\bfseries Input:} $(X_1, \dots, X_n) \stackrel{\text{i.i.d.}}{\sim} p_{ {\psi}^{\star}}$
   \STATE \textbf{Parameters:}  
   Same as Algorithm \ref{alg:online_cd},
   plus number of epochs $ T $, batch size $ B $, batching schedule $ B_{t, j} $, initial parameter $\psi_{1, 0}$
   \FOR{$ t = 1, \dots, T $}
	\FOR{$ j = 1, \dots \lceil n / B \rceil $}
	\STATE $\lbrack \tilde{X}^{m}_{t,i,j}\sim k^{m}_{\psi_{t-1}}(X_i, \cdot)\hspace{.3em}\texttt{for\hspace{.3em}i\hspace{.3em}in}\hspace{.3em}B_{t,j}   \rbrack  $ 
	\vspace{-.8em}
	\STATE $ \textstyle h_{t, j} \coloneqq \frac{1}{B}\sum_{i \in B_{t, j}} (\phi(X_i) - \phi(\tilde{X}^{m}_{t,i,j})) $
	\vspace{-.9em}
	\STATE $ \psi_{t, j} \leftarrow \psi_{t, j-1} - \eta_t h_{t, j} $
	\STATE $ \psi_{t, j} \leftarrow \text{Proj}_{\Psi}(\psi_{t, j}) $
   \ENDFOR
   \STATE $ \psi_{t+1, 0} \leftarrow \psi_{t, \lceil n / B \rceil} $
   \ENDFOR
   \RETURN $ \psi_{T, \lceil n / B \rceil} $
\end{algorithmic}
\end{algorithm}
\end{minipage}

\section{Non-asymptotic analysis of Online CD}\label{sec:online_cd}

\vspace{-0.7em}
\subsection{Preliminaries and Assumptions} \label{sec:preliminaries}

\vspace{-0.5em}
Recall that the chi-squared divergence between two probability measures $ p $ and $ q $ is defined as
$\chi^2( p, q) \coloneqq  \int ( \frac{ \mathrm{d}p }{ \mathrm{d}q }(x) - 1)^2 q(\mathrm{d}x) $ if 
$p \ll q$, and $+ \infty$ otherwise.
Here, $ p \ll q $ denotes that $ p $ is absolutely continuous with respect to $ q $ and $ \mathrm{d}p / \mathrm{d}q $
is the Radon-Nikodym derivative \cite{doob2012measure} of $ p $ with respect to $ q $.
Let $ L^{2}(p_{\psi})$ be the space of square-integrable functions with respect to $ p_{\psi} $.
For a function $ f \in L^{2}(p_\psi) $, we define
\vspace{-0.5em}
\begin{equation} \label{eq:alpha_f_psi}
\begin{aligned}
\alpha(f, \psi) =  
\mfrac{  
	\big(  \int \big(  \int \left (f- \mathbb{ E } \left \lbrack f(\cramped{X^{\psi}}) \right \rbrack \right )(y)  \, k_{\psi}(x, \mathrm{d}y) \big)^{2} p_{\psi}(\text{d}x) \big)^{1 / 2}
}{  
\big(  \int  \left (f - \mathbb{ E } \left \lbrack f(\cramped{X^{\psi}}) \right \rbrack  \right )(x)^{2} p_{\psi}(\text{d}x) \big)^{1 / 2}
}
\end{aligned}
\end{equation}
which is a measure of how quick a Markov chain with kernel $k_\psi$
mixes, relative to the function $f$ \cite{levin2017markov}.
With these definitions in hand, we now state the assumptions required by our analysis of online CD. These assumptions
form a strict subset of the assumptions considered in prior work \cite{jiang2018convergence},
which required additional regularity and tail conditions on the Markov kernels $ k_{\psi} $.
\begin{assump}\label{asst:SGD-A1}
	$ \mathcal  P_{\psi}  $ is 
	a subset of 
	a \emph{regular} and \emph{minimal}
	\cite[Section 3.2]{wainwright2008graphical} exponential family with natural parameter domain $ \mathcal  D \subseteq \mathbb{R}^p $,
 $ \Psi $ is a \emph{convex and compact} subset of $ \mathcal  D$,
 and ${\psi}^{\star}$ lies in the interior of $\Psi$.
\end{assump}
\begin{assump}\label{asst:SGD-A4}
	There exists a constant $ C_{\chi}  > 0 $ such that $ \chi^2( p_{ {\psi}^{\star}}, p_{\psi}) \leq C_{\chi}^2 \| \psi - {\psi}^{\star} \|^2 $
\end{assump}
\begin{assump}\label{asst:SGD-A5}
	$ \alpha \coloneqq \sup \{  \alpha(f, \psi),\, f \in \{ \phi_i \}_{i=1}^{p} \cup \{ \phi_i \phi_j \}_{i,j=1}^{p},\, \psi \in \Psi \} < 1$, where
	$ \phi_i $  is the $ i $-th component of the function $ \phi $, and $ \phi_i^2 $ is the $ i $-th component of the function $ x  \longmapsto \phi(x)^2 $.
\end{assump}

\noindent 
A well known property of EFs \citep[Proposition~3.1]{wainwright2008graphical} is that their negative cross-entropy (against any other measure) is $ C^{\infty} $, convex,
and strictly so if the exponential family is minimal (meaning that the set of sufficient statistic functions $ \phi_i $ are not linearly dependent). Leaving aside the issue of
intractable expectations, this convexity suggests that $ \mathcal L $ can be efficiently minimized using stochastic approximation algorithms \cite{nesterov2013introductory, moulines2011non}.
The compactness of $ \Psi $ provided by Assumption \ref{asst:SGD-A1} thus ensures, by the extreme value theorem \cite{harris1987shorter}, the existence of finite positive constants 
$ \mu$ and $ L $ defined as:

\begin{equation} \label{eq:mu_L}
\begin{aligned}
\mu \coloneqq \mmin_{ \psi \in \Psi }  \lambda_{\min_{  }} \left ( \nabla_{ \psi }^{2 } \mathcal  L(\psi) \right ), \quad L \coloneqq \mmax_{ \psi \in \Psi }  \lambda_{\max_{  }} \left ( \nabla_{ \psi }^{2 } \mathcal  L(\psi) \right ),
\end{aligned}
\end{equation}
where $ \nabla_{ \psi }^{2} \mathcal  L $ is the Hessian of $ \mathcal  L $ with respect to $ \psi $.
$ \mu $ (called the \emph{strong convexity} constant) and $ L $ (a bound controlling the smoothness of the problem) play a critical role in the analysis of convex optimization
algorithms \cite{nesterov2013introductory}. 
While it is possible to obtain convergence rates in non-smooth or non-strongly-convex settings, our analysis follows the spirit of \cite{jiang2018convergence} by leveraging the strong
convexity of the problem to compensate for the bias introduced by using CD gradients instead of unbiased stochastic gradients.

Assumption \ref{asst:SGD-A4} allows link variations in distribution space to variations in parameter space, and will be instrumental
to control the bias of the CD gradient. Note that since $ \chi^2(p_{ {\psi}^{\star}}, p_{\psi}) = e^{ \log Z(2\psi - {\psi}^{\star}) - (2 \log Z(\psi) - \log Z({\psi}^{\star})) } - 1 $
provided that  $ 2  \psi - {\psi}^{\star} \in \mathcal D$ (see \cite[Lemma 1]{nielsen2013chi}), we expect Assumption \ref{asst:SGD-A4} to hold in many cases of interests.
On the other hand, the possible exponential scaling of $ C_{\chi} $ w.r.t $ \log Z $ suggests that this constant may be large in some instances.

Assumption \ref{asst:SGD-A5}  is a \emph{restricted} spectral gap condition: it guarantees that the time required by
the MCMC algorithm to estimate expectations of $ \phi $ and $ \phi^2 $ under $ p_{\psi} $ will be uniformly bounded.
This assumption is weaker than the (unrestricted) uniform spectral gap condition of \cite{jiang2018convergence},
which requires that $ \alpha $ controls the convergence rate of \emph{all} functions in $ L^2(p_\psi) $.
% taking the supremum over all functions  $ L^2{p_{\psi}}(\mathbb{R}) $
Note that standard results in stochastic analysis \cite{bakry2014analysis} guarantee that $ \alpha \leq 1 $: thus,
it only remains to ensure that $ \alpha $ is strictly less than $ 1 $.
Spectral gaps are strongly dependent on two properties of distribution: their tail behavior and their multimodality.
While multimodality poses the risk of pushing the constant $ \alpha $ close to $ 1 $, very heavy tails distributions
may not verify the spectral gap condition at all.

\subsection{Results}

\vspace{-0.5em}

\subsubsection{Parametric convergence of online CD}\label{sec:online-cd-parametric}

\vspace{-0.5em}
In this section, we show that under the assumptions stated in Section \ref{sec:preliminaries}, the iterates $ \psi_t $
produced by the online CD algorithm described in Algorithm \ref{alg:online_cd} will converge to the true parameter $ \psi^{\star} $
at the parametric rate $ O(n^{-1/2}) $. To do so, we follow a well known paradigm in convex optimization \cite{moulines2011non}
by deriving a recursion on the quantity $ \delta_{t} \coloneqq \mathbb{ E } \left \| \psi_t - {\psi}^{\star} \right \|^2$, which will allow, after unrolling, to obtain convergence rates for the iterates $ \psi_t $.
We aim to characterize precisely the impact of performing CD as opposed to performing online SGD on $ \mathcal  L $,
which would consist of replacing the CD gradient $ h_t $ of Algorithm \ref{alg:online_cd} by the unbiased (stochastic) gradient, given by:
\begin{equation} \label{eq:online-mle-gradient}
\begin{aligned}
g_t(\psi) \coloneqq - \phi(X_t) + \nabla_{ \psi } \log Z(\psi)
\end{aligned}
\end{equation}
which satisfies $\E\, g_t(\psi) = \nabla_{\psi} \mathcal L(\psi)$.
The only stochasticity in $ g_t $ comes from the sampling of a single data point $ X_t $ from the true distribution, which is unavoidable
in the online setting, and we have $\E \| g_t(\psi^\star) \|^2 =  \textrm{Tr}(\Cov \left \lbrack \phi(\cramped{X_1}) \right \rbrack) \eqqcolon \sigma_\star^2 $.
$\sigma_\star^2$ plays a key role in the analysis of Stochastic Gradient Descent \cite{moulines2011non}.
We expect that replacing $ g_t $ by $ h_t $ will introduce two sources of approximation: a bias term
coming from using a finite number of MCMC steps $ m $, and an \emph{additional} variance term, coming from using a single sample
$ \tilde{X}^{m}_t $ to estimate $ \nabla_{ \psi } \log Z(\psi_t) $. With that in mind, we derive a recursion on $ \delta_{t} $ in the
following lemma.

\begin{lemma}\label{lem:online-cd-recursion}
  Let $(\psi_t)_{0 \leq t \leq n}$ be the iterates from
  Algorithm \ref{alg:online_cd}. Denote $ \delta_{t} = \mathbb{ E } \| \psi_{t} - {\psi}^{\star} \|^2 $,
  $ \sigma_{\star} = (\mathbb{ E } \| \phi(X_1) - \mathbb{ E }  \left \lbrack \phi(X_1) \right \rbrack  \|^2)^{1 / 2} $, and
  $ \sigma_{t} = (\mathbb{ E } \| \phi(X^{\psi_t}) - \mathbb{ E }  \left \lbrack \phi(X^{\psi_t}) \right \rbrack  \|^2)^{1 / 2} $.
  Then, under 
  \ref{asst:SGD-A1}, \ref{asst:SGD-A4} and \ref{asst:SGD-A5}, for all $  t \geq  1 $, 
  \begin{equation} \label{eq:lem:online-cd-recursion}
  \begin{aligned}
	  \delta_{t} \leq (1 - 2 \eta_t \tilde{\mu}_{m, t-1} + 2 \eta_t^2 L^2) \delta_{t-1} + 2 \eta_t^2 \tilde{\sigma}_{m, t-1}^2  +  
	  4 \alpha^{m / 2} \eta_t^2  \left \| \log Z \right \|_{3, \infty } C_{ \chi} \delta_{t-1}^{1 / 2}
  \end{aligned}
  \end{equation}
  where $ \left \| \log Z \right \|_{3, \infty} $  is a
  constant,
  $ \tilde{\mu}_{m, t} \coloneqq \mu - \alpha^{m} \sigma_{t} C_{\chi} $,
  and $ \tilde{\sigma}_{m, t} \coloneqq ( \sigma_{\star}^2 + \sigma_t^2 + 2\sigma^2_{t} \alpha^{2m})^{1 / 2}$.
\end{lemma}
Lemma \ref{lem:online-cd-recursion} is proved in Appendix \ref{app:sgd-style-recursion},
which details the form of $ \left \| \log Z \right \|_{3, \infty} $, a constant that we expect to scale roughly  as $ d L $.
Loosely speaking, this recursion suggests that as the learning rate $ \eta_t $ goes to $ 0 $,
the two terms scaling in $ \eta_t^2 $ will be negligible, in which case we will have: $ \delta_{t} \leq (1 - 2\eta_t \tilde{\mu}_{m, t-1})\delta_{t-1} < \delta_{t-1} $,
yielding convergence of $ \delta_t $ to $ 0 $.
We make these arguments formal in the next theorem.
The reader familiar with the convex optimization literature will note the similarities between this recursion
and the one derived in \cite{moulines2011non}, which would apply as is to online SGD on $ \mathcal  L $ using $ g_t $.
The difference between the two recursions is that the roles of the strong convexity constant $\mu$ and the noise $\sigma_\star$ are now played respectively by

\begin{equation*}
\begin{aligned}
\tilde{\mu}_{m, t-1} 
\;=&\; \mu - \alpha^{m} \sigma_{t-1} C_{\chi} 
&\text{ and }&&
\tilde{\sigma}_{m, t-1}^2 \;=&\; \sigma_{\star}^2 + \sigma_t^2 + 2\sigma^2_{t} \alpha^{2m}
\end{aligned}
\end{equation*}

These two modifications respectively characterize the impact of the bias and the additional variance
introduced by the CD gradient. The last term in Equation \ref{eq:lem:online-cd-recursion}, scaling in $ \alpha^{m / 2} \eta_t^2 \sqrt { \delta_t} $, is a residual higher order mixed term coming from relating the variance of
the Markov chain sample $ \tilde{x}^{m}_t $ to $ \sigma_t^2 $. This term can be easily controlled as done next, and disappears as $ m  \to \infty $.
Investigating the impact of $ m $ in the recursion, we notice that as $ m  \to \infty $, $ \tilde{\mu}_{m, t}   \to \mu $.
As we will see later, this ensures that CD will converge for a sufficiently high $ m $. On the other hand,  in that same regime,
$ \tilde{\sigma}_{m, t}$ \emph{does not} converge to $ \sigma_{ \star} $, but rather to 
$ (\sigma_{ \star}^2 + \sigma_t^2)^{1 / 2} $, showing the
irreducible impact of the variance term. While we precisely investigate the impact of the residual variance term in the next section,
we now unify $ \sigma_{\star} $ and $ \sigma_{t} $ by introducing
\begin{equation} \label{eq:sigmq}
\begin{aligned}
\sigma \coloneqq \msup_{ \psi \in \Psi} (\mathbb{ E } \left \lbrack \| \phi(X^{\psi}) - \mathbb{ E } \left \lbrack \phi(X^{\psi}) \right \rbrack \|^2 \right \rbrack)^{1 / 2}.
\end{aligned}
\end{equation}
$ \sigma $ is an upper bound on the noise induced \emph{both} by the CD gradient and by the online setup, and was used in prior work \cite{jiang2018convergence}.
Note that by the properties of $ \log Z $, $ \sigma^2 $ also equals $ \sup_{ \psi \in \Psi } \mathrm{tr} (\nabla_{ \psi }^2 \mathcal  L(\psi)) $, where $ \mathrm{tr}(A) $ is the trace of
$ A \in \mathbb{R}^{p \times p}$, and thus finite by the extreme value theorem.
The following theorem is obtained by invoking standard unrolling arguments in the convex optimization literature.
In the next result, we use the function $ \varphi_{\gamma}(t) $, defined as $ \varphi_{\gamma}(t) = \frac{ t^{\gamma} - 1 }{ \gamma }$ if $ \gamma \neq 0 $, and $ \log t $ if $ \gamma = 0 $.
\begin{theorem}\label{thm:cd-sgd}
	Fix $ n \geq  1 $. Let $ (\psi_t)_{0 \leq t \leq n}  $ be the iterates produced by Algorithm \ref{alg:online_cd}, and define
	$ \delta_{t} \coloneqq \mathbb{ E } \left \| \psi_{t} - {\psi}^{\star} \right \|^2 $. Moreover,
	assume that $m > \frac{\log (\sigma C_\chi / \mu )}{\log |\alpha|}$, i.e.~$\tilde \mu_m \coloneqq \mu - \alpha^{m} \sigma C_\chi > 0$.
	Then under Assumptions \ref{asst:SGD-A1}, \ref{asst:SGD-A4} and \ref{asst:SGD-A5},
	for $ \eta_t = C t^{-\beta} $ with $ C > 0$, we have:
	  \begin{equation*} 
	  \begin{aligned}
	  \delta_n \leq
	  \begin{cases}
		  2 \exp \left(4 \tilde{L}_m C^2 \varphi_{1-2 \beta}(n)\right) \exp \left(-\frac{ \tilde{\mu}_m C}{4} n^{1-\beta}\right)\left(\delta_0+\frac{\tilde{\sigma}_m^2}{\tilde{L}_m^2}\right)+\frac{4 C \tilde{\sigma}_m^2}{\tilde{\mu}_m n^\beta}, & \text { if } 0 \leq  \beta<1 \\ 
	  	\frac{\exp \left(2 \tilde{L}_m^2 C^2\right)}{n^{\tilde{\mu} C}}\left(\delta_0+\frac{ \tilde{\sigma}_m^2}{\tilde{L}_m^2}\right)+2 \tilde{\sigma}_m^2 C^2 \frac{\varphi_{\tilde{\mu}_m C / 2-1}(n)}{n^{\tilde{\mu}_m C / 2}}, & \text { if } \beta=1\;,
	  \end{cases}
	  \end{aligned}
	  \end{equation*}
	where 
 $\tilde \sigma_m  = \sigma^2(2 + 2 \alpha^{2m})  + \alpha^{m / 2} \left \| \log Z \right \|_{3, \infty }^{2} C_{ \chi}^{2}$ and $ \tilde{L}_m = (L^2 + \alpha^{m / 2})^{1 / 2}$
 .
 Consequently, if $ \eta_n = \frac{C}{n} $ with an initial learning rate $C > 2 \tilde \mu_m^{-1}$,
 we have
 $ \sqrt{\delta_n}
 \leq 
 2 \tilde{\sigma}_m C \sqrt \frac{ \tilde{\mu}_m C }{ \tilde{\mu}_m C - 2 } \frac{1}{\sqrt {n} } + o\big(\frac{1}{\sqrt{n}}\big) \;.
 $
\end{theorem}
Theorem \ref{thm:cd-sgd} is proved in Appendix \ref{app:online-cd-convergence}. It shows that the iterates produced by online CD will converge to the true parameter
$ \psi^{\star} $ at the rate $ O(n^{-1/2}) $ provided that the number of steps $ m $ is sufficiently large, improving over the asymptotic $ O(n^{-1 / 3}) $ rate of \cite{jiang2018convergence},
while imposing slightly weaker conditions on the number of steps $ m $ (see \cite[Theorem 2.1]{jiang2018convergence}).
This proves that online CD can be asymptotically competitive with other methods for training unnormalized models, such as Noise Contrastive Estimation \cite{gutmann2010noise}, or
Score Matching \cite{hyvarinen2005estimation}. However, the asymptotic variance of $ \psi_t $ (e.g. the multiplicative factor in front of the $ O(n^{-1/2}) $ term)
is likely to be suboptimal, e.g. much larger than the Cr\'amer-Rao bound, given by the trace of the inverse of the Fisher information matrix \cite{wainwright2008graphical}.
Given the statistical optimality of MLE, and the fact that CD in an approximate MLE method, this motivates the further goal or obtaining a CD estimator with near-optimal statistical properties.
In the next section, we achieve this goal by showing that averaging the iterates $ \psi_t $ will produce a near statistically-optimal estimator, in a sense that we will make precise.

\subsubsection{Towards statistical optimality with averaging}\label{sec:online-cd-averaging}
Polyak-Ruppert averaging \cite{polyak1992acceleration} is a simple yet surprisingly effective way to
construct an asymptotically optimal estimator $ \bar{ \psi }_n \coloneqq \frac{1}{n} \sum_{ t=1 }^{ n } \psi_i $
from a sequence of iterates
$ (\psi_t)_{0 \leq t \leq n} $ obtained by running a standard online SGD algorithm \cite{moulines2011non}.
As shown in \cite{moulines2011non}, when the objective is the cross-entropy of a model, and assuming the unbiased stochastic gradients are available,
averaging yields an estimator $ \overline{ \psi } $ with the asymptotic variance
$ \mathrm{tr}(\mathcal  I( {\psi}^{\star})^{-1}) / n $, where 
$ \mathcal  I(\psi) \coloneqq \mathrm{Cov} \left \lbrack \phi(X_1) \right \rbrack $ is the Fisher information matrix of the
data distribution $ p_{\psi^{\star}} $. $ \mathcal  I( {\psi}^{\star})^{-1} $ being the Cram\'er-Rao \emph{lower
bound} on asymptotic variances of statistical estimators \cite{casella2021statistical},
this estimator $ \bar{ \psi }_n $ is asymptotically optimal.
The following theorem shows conditions under which averaging CD iterates can give rise to a near-optimal estimator.
\begin{theorem}[Contrastive Divergence with Polyak-Ruppert averaging]\label{thm:online-cd-averaging}
	Let $ (\psi_t)_{t \geq  0} $ the sequence of iterates obtained by running the CD algorithm
	with a learning rate $ \eta_t = C t^{-\beta} $ for $ \beta \in (\frac{1}{2}, 1) $.
	Define $ \bar{ \psi }_n \coloneqq \frac{1}{n}\sum_{ i=1 }^{ n } \psi_i $. Then, under the same assumptions as 
	Theorem \ref{thm:cd-sgd}, and assuming additionally that 
	$ m \coloneqq m(n) > \frac{(1 - \beta) \log n }{2|\log \alpha|}$, we have, for all $ n \geq  1$, \vspace{-.5em}
	\begin{equation*} 
	\begin{aligned}
	  \big (\mathbb{ E } \left \| \overline{\psi}_n - {\psi}^{\star} \right \|^{2} \big )^{1 / 2} \;\leq\; 2 \sqrt{ \mfrac{ \mathrm{tr}(\mathcal  I( {\psi}^{\star})^{-1}) }{  n} } + o(n^{-1 / 2})
	\end{aligned}
	\end{equation*} \vspace{-.5em}
	Consequently, we have that $ \limsup_{n  \to \infty} n \mathbb{ E }(  \left \| \overline{ \psi }_n - {\psi}^{\star} \right \|^2) \leq 4 \times \mathrm{tr}(\mathcal  I( {\psi}^{\star})^{-1})  $.
\end{theorem}
Theorem \ref{thm:online-cd-averaging}, alongside with a statement
which includes the asymptotic order of the residual term,
is proved in Appendix \ref{app:proof-online-cd-averaging}.
It shows that at the cost of an increase in \emph{computational} complexity
of the entire algorithm from $ O(n) $ to $ O(n \log n) $, $ \bar{ \psi }_n $
will be a near-optimal statistical estimator of $ \psi^{\star} $.
While this increase in complexity emerges from the bias of CD, the additional
variance of CD results in an asymptotic variance inflated by a factor of $ 4 $
compared to the Cram\'er-Rao bound.

\vspace{-0.5em}

\section{Non-asymptotic analysis of offline CD} \label{sec:offline:SGD}

\vspace{-0.5em}
In practice, CD gradient approximation schemes are commonly used within an offline stochastic gradient descent (SGD) algorithm, where one is given the full size-$n$ dataset upfront and each update uses some stochastic subset of the data. We study CD under offline SGD with replacement (SGDw), i.e.~\cref{alg:offline_cd} with batches $B_{t,j}$ being i.i.d.~uniform draws of size-$B$ subsets of $[n]$, and include SGD without replacement in \cref{appendix:SGDo:results}. To do so, we follow the setting of prior work on offline CD \cite{jiang2018convergence}, which established its asymptotic $ O(n^{-\frac{1}{3}}) $ consistency. We show that by slightly strengthening a moment assumption used in \citep{jiang2018convergence}, the offline CD iterates converge to the true parameter at a near-parametric $ O((\log n)^{\frac{1}{2}} n^{-\frac{1}{2}}) $ rate. Our proof proceeds by controlling a ``tail probability'' term specific to the offline setting which characterizes the strength of the correlations between the offline CD iterates and the training data. While, as we show, the assumptions of \cite{jiang2018convergence} provide a tail control sufficient to obtain a near-parametric rate, other strategies are possible to obtain convergence guarantees. In particular, we show that non-asymptotic convergence can be obtained by either (1) relaxing assumptions on the Markov kernel required by prior work, or (2) making a specific mixing assumption the Markov chain.

\subsection[Background: Asymptotic  consistency of offline CD]{Background: Asymptotic consistency of offline CD in subexponential settings}

Prior work \cite{jiang2018convergence} has established \emph{asymptotic} $ O(n^{-\frac{1}{3}}) $ consistency of the (averaged) offline CD iterates in the full-batch case. We summarize their results and assumptions below.
In the following, we write $K^m_\psi(x) \sim k^m_\psi(x, \argdot)$ to keep the dependence of the MCMC sample on $ \psi $ and $ x  $explicit.
%\textcolor{red}{mention how they studied a special case of offline CD}
\begin{assump}\label{asst:Markov:noise} There exists $\nu \geq 2$ s.t.~for all $m \in \N$, there is $\kappa_{\nu;m} < \infty$ s.t.
\begin{align*}
    \msup_{x \in \cX} \, \msup_{\psi \in \Psi} \, 
    \big( \mean \big\| \phi(K^m_\psi(x)) - \mean[\phi(K^m_\psi(x)) ] \big\|^{\nu} \big)^{1/\nu}
    \;\leq\; 
    \kappa_{\nu;m}
    \;.
\end{align*}
\end{assump}

\begin{assump} \label{asst:Jiang} There exists some $C_m > 0$ such that, for all $\psi_1, \psi_2 \in \Psi$, $\sup_{x \in \cX} \| \mean[ \phi(K^m_{\psi_1}(x))] - \mean[ \phi(K^m_{\psi_2}(x))] \| \leq C_m \|\psi_1 - \psi_2 \|$.
\end{assump}

\begin{assump} \label{asst:subexp} There exist some $\sigma_m, \zeta_m > 0$ such that, for any $z \in \R^p$ with $\|z \| \leq \zeta_m$, $\mean[ e^{z^\top(\phi(K^m_{\psi^*}(X_1)) - \mean[\phi(K^m_{\psi^*}(X_1))] ) } ]\leq e^{\sigma_m^2 \| z \|^2 /2}$.
\end{assump}

\begin{theorem}[Theorem 2.1 of \citep{jiang2018convergence}]\label{thm:offline_cd_jiang}
    Assume assumptions \ref{asst:SGD-A1},\ref{asst:SGD-A4}, \ref{asst:SGD-A5}, \ref{asst:Markov:noise} (for $ \nu=2 $), \ref{asst:Jiang} and \ref{asst:subexp}. Let $ \psi_{t, 1} $ be the $t$-th iterate of offline CD with full-batch gradient descent and constant stepsize $ \eta_t = C $, i.e the iterates produced by Algorithm \ref{alg:offline_cd} using $ B_{t, 1} = [\![1, n]\!] $. Then for any learning rate $ C $ and number of Markov kernel steps $ m $ satisfying $ \mu - \alpha^m \sigma C_{\chi}  - \frac{C}{2} (L + \alpha^m \sigma C_{\chi} )^2 > 0 $, we have, for some $A_m > 0$,
    \begin{equation*}
    \begin{aligned}
        \lim_{ n  \to \infty } \mathbb{ P } \left (  \limsup_{T  \to \infty}  \left \|  \frac{1}{T}\sum\limits_{ t=1 }^{ T } \psi^{\rm SGDw}_{t,1}   - {\psi}^{\star} \right \| > A_{m} n^{-\frac{1}{3}}\right ) = 0
    \end{aligned}
    \end{equation*}
\end{theorem}
This result shows convergence of the \emph{averaged} full-batch CD iterates to the true parameter in the large $ n $ and $ T $ limit. As discussed, this result is asymptotic both in $ n $ and $ T $: the probability of the error exceeding $ A_{m} n^{-\frac{1}{3}} $ goes to 0 as $ n \to \infty $ and $ T \to \infty $, but at an unknown rate. Moreover, the $ O(n^{-\frac{1}{3}}) $ does not match the optimal $ O(n^{-\frac{1}{2}}) $ rate.

\vspace{-0.5em}
\subsection[Nonasymptotic consistency of offline CD]{Sharpening offline CD bounds in subexponential settings} 

\vspace{-0.5em}
\subsubsection[Non-asymptotic  near-optimal consistency]{Non-asymptotic $ \tilde{O}(n^{-1/2}) $-consistency}
As a first result, we show that under the assumptions of \cite{jiang2018convergence} (except for a slightly stronger $\nu >2$ moment assumption in \ref{asst:Markov:noise}), $\psi^{\rm SGDw}_{T,N}$ in fact achieves a near-parametric rate.
The most general version of our result holds for any learning rate schedule of the form $ Ct^{-\beta}, \beta \in \left \lbrack 0, 1 \right \rbrack$, and for offline SGD with arbitrary batch sizes $ B $, with data drawn either with or without replacement across batches. For simplicity, we first present our result assuming full batch ($ B=n, N=1, \psi^{\rm SGDw}_{t,j}=\psi^{\rm SGDw}_{t,1} $ for $t \ge 1$) SGD with constant step sizes $ \eta_t = C $, which is the setting of \cite{jiang2018convergence}. Analogue bounds holding for the other mentioned batching and step sizes schedules can be found in Appendix \ref{sec:appendix:offline:SGD}.

\begin{theorem} \label{lem:tail:prob:subexp} 
    Assume the setup of Theorem \ref{thm:offline_cd_jiang}, except that Assumption \ref{asst:Markov:noise} holds for some $ \nu > 2 $, and that
    $\tilde \mu_m \;=\; \mu- \alpha^m \sigma C_\chi \;>\; 4 C L^2\;$.
    Let $\delta^{\rm SGDw}_{t, j} \coloneqq \mean \| \psi^{\rm SGDw}_{t,j}  - \psi^* \|^2$. Then, we have:
    \begin{equation} \label{eq:tail:prob:subexp:nonasymptotic:upperbound}
    \begin{aligned}
       \sqrt{\delta^{\rm SGDw}_{T, 1}} \leq E_1^{T,1} \sqrt{\delta^{\rm SGDw}_{0, 0}} \,+\,
        C'(p,\nu,m,\Psi) \Big( \mfrac{\sqrt{\log n}}{\sqrt{n}} + \mfrac{1}{\sqrt{n}} \Big)
        \Big( \mfrac{e^{\frac{\tilde \mu_m C}{2}}}{\tilde \mu_m C} + \mfrac{E_2^{T,1}}{L^2C^2}  \,  \Big)
    \end{aligned}
    \end{equation}
    where $E_1^{T,1}$, $E_2^{T,1}$ are functions decreasing exponentially  in $ T  $, and $ C'(p,\nu,m,\Psi)$
    is a constant in $ n, T $. Consequently, 
    \begin{align*}
        \lim_{T \rightarrow \infty} \sqrt{\delta^{\rm SGDw}_{T,1}} 
        \leq
        \mfrac{e^{\frac{\tilde \mu_m C}{2}}}{\tilde \mu_m C} C'(p,\nu,m,\Psi,\beta) \Big( \mfrac{\sqrt{\log n}}{\sqrt{n}} + \mfrac{1}{\sqrt{n}} \Big)\;.
    \end{align*}
\end{theorem}

The precise values of all the constants can be found in \cref{thm:SGDw:full} (for $E_1^{T,1}$, $E_2^{T,1}$) and \cref{lem:tail:prob:subexp:full}
(for $ C'(p,\nu,m,\Psi,\beta) $), including their expressions for $ N>1 $ and  $ \beta \in [0, 1] $.
We comment on the main differences between our result and the one of \cite{jiang2018convergence}.
First our bound holds for \emph{any} epoch $ T $ and number of samples $ n $. 
Second, fixing $ n $ but taking $ T  \to \infty  $, the final bound matches the parametric $ O(\sqrt {n} ) $ up to a $ \sqrt{\log(n)} $ factor, a significant improvement over the $ O(n^{-\frac{1}{3}}) $ rate of \cite{jiang2018convergence}. Finally, we control an $L_2$ error, which is a stronger control than a high probability bound by Markov's inequality; we hypothesize this is the reason why a slightly stronger moment assumption is required for our setup, compared to the one used for the high probability bound in \cite{jiang2018convergence}.

Inspecting Equation \ref{eq:tail:prob:subexp:nonasymptotic:upperbound}, we notice the presence of two \emph{transient} terms, and a \emph{stationary term}, reminiscent of the structure of upper bound of Theorem \ref{thm:cd-sgd}. The transient terms (i.e. the ones containing $ E_1^{T, 1} $ and $ E_2^{T, 1} $) vanish exponentially fast in the total number of CD updates $ T $. However, unlike in online CD where the number of updates and the number of samples are tied (e.g. $ T=n $), these two values are now \emph{decoupled}, and these terms can be made arbitrarily small by increasing the number of gradient steps $ T $ without having to collect more samples $ n $.
The stationary term, which is the only one remaining in the limit of $ T \to \infty$, decreases with $ n $ at a rate that is independent of hyperparameters like the step size $ C $ or the learning rate schedule $ \beta $ (see \cref{lem:tail:prob:subexp:full}). In that sense, offline CD compares favorably to online CD, whose rate is sensitive to $ \beta $ and $ C $, and averaged online CD, whose bound contains higher-order (in $ n $) terms which can be large in the moderate $ n $ regime. On the other hand, the stationary term in offline CD is asymptotically suboptimal: its rate is larger (while only up to a log factor) than the best-case $ O(\sqrt{n}) $ one achieved by online CD algorithms, and the leading constant does not match the optimal one.

\subsubsection[Proof of main theorem]{Proof of \cref{lem:tail:prob:subexp}}
The high-level proof of \cref{lem:tail:prob:subexp} follows a similar strategy as the online one: first, derive a recursion for the quantity $\delta^{\rm SGDw}_{t,1} \coloneqq \mean \| \psi^{\rm SGDw}_{t,1}  - \psi^* \|^2$, 
then unroll it explicitly to obtain a final bound on $\delta^{\rm SGDw}_{T, 1}$. The main difference to online CD is the presence of an additional offline-specific correlation between the iterates and the data. We thus break down the proof into three steps: (1) deriving a controllable, uniform-in-time upper bound of the data-iterate correlations, (2) deriving and unrolling a recursion on $\delta^{\rm SGDw}_{t,1}$ containing this new term, and (3) controlling that term to obtain a final bound on $\delta^{\rm SGDw}_{T, 1}$.
\paragraph{Step 1:characterizing the data-iterate correlations in offline CD}
In offline CD, at each epoch $ t \geq  1 $, the iterate $ \psi^{\rm SGDw}_{t-1,1} $ and the data samples $ X_{i} $ are correlated: this is because these samples may have been used in previous epochs $ t' < t-1 $ to obtain the $ \psi^{\rm SGDw}_{t',1} $, which themselves influenced $ \psi_{t-1,1} $. With such correlations, we now have $ \mathbb{ P }\lbrack X_i \in \bullet | \psi^{\rm SGDw}_{t-1, 1}\rbrack \ne \mathbb{ P }\lbrack X_i \in \bullet \rbrack  $, preventing us from obtaining an unrollable recursion on $ \delta^{\rm SGDw}_{t,1}  $ by first marginalizing $ X_i $ out to obtain an upper bound of $ \mathbb{ E } \left \lbrack  \left \| \psi^{\rm SGDw}_{t,1}  - \psi^* \right \|\vphantom{a}^2 |  \psi^{\rm SGDw}_{t-1,1} \right \rbrack  $ that only depends on $  \left \| \psi^{\rm SGDw}_{t-1,1} - {\psi}^{\star} \right \| $, and then marginalizing over $ \psi^{\rm SGDw}_{t-1,1} $ to obtain a recursion as in Lemma \ref{lem:online-cd-recursion}. As this problem would not have occurred had we used ``fresh samples'' (e.g. i.i.d copies of $ X_i$ not present in the training data) to perform our update, the core of the proof lies in controlling the following quantity:
\begin{align*}
        \Delta(\psi^{\rm SGDw}_{t,1})
        \coloneqq&
             \Big\| \,\mfrac{1}{n} \msum_{i \leq n} 
            \big( 
                \mean\big[ 
                    \phi\big( K^m_{i;\psi^{\rm SGDw}_{t,1}}(X_i) \big)
                \big| \psi^{\rm SGDw}_{t,1}, X_i \big]
                - 
                \mean\big[ \phi\big( K^m_{i;\psi^{\rm SGDw}_{t,1}}(X'_1) \big) \,\big|\, \psi^{\rm SGDw}_{t,1} \big] 
            \big) \Big\|
\end{align*}
where $X'_1$ is an i.i.d.~copy of $X_1$. 
$ \Delta(\psi^{\rm SGDw}_{t,1}) $ is the expected (over the data and iterates) error between a quantity that allows to obtain a recursion (the rightmost term) and the one actually used by offline CD (the leftmost term).
To control it, we upper-bound it using a tail decomposition:
\begin{equation} \label{eq:varepsilon:new}
\begin{aligned}
    \hspace{-0.5em} \mean[\Delta(\psi^{\rm SGDw}_{t,1})^2] 
                                         & \leq \epsilon^2 \! + \! (\sup_{t}\mean[\Delta(\psi^{\rm SGDw}_{t,1})^\nu])^{2/\nu} \sup_{t} \P( \Delta(\psi^{\rm SGDw}_{t,1}) > \epsilon)^{\frac{\nu-2}{\nu}} \coloneqq \varepsilon^{\rm SGDw}_{n,m,T;\nu}(\epsilon)^2
                                         %\hspace{-1.5em}
\end{aligned}
\end{equation}
We invoke an additional assumption to ensure that $(\mean[\Delta(\psi^{\rm SGDw}_{t,1})^\nu])^{2/\nu}$ is finite; in the results of \cite{jiang2018convergence}, this is automatically implied by assumptions \ref{asst:Markov:noise} and \ref{asst:subexp}. For simplicity we assume the same bounding constant $\kappa_{\nu;m}$.
\begin{assump}\label{asst:Markov:noise:X1} There exists $\nu \geq 2$ s.t.~for all $m \in \N$, $\kappa_{\nu;m}$ from \ref{asst:Markov:noise} moreover verifies
    \vspace{-0.5em}
\begin{align*}
    \msup_{\psi \in \Psi} \, 
    \big( \mean \big\| \phi(K^m_\psi(X_1)) - \mean[\phi(K^m_\psi(X_1)) ] \big\|^{\nu} \big)^{1/\nu}
    \;\leq\; 
    \kappa_{\nu;m}
    \;.
\end{align*}
\end{assump}

\vspace{-0.3em}
Note the similarity of this assumption with \cref{asst:Markov:noise}: the only difference is that $ X_1 $ is now a random training point instead of an deterministic (arbitrary) one.
Ensuring \cref{asst:Markov:noise:X1} in addition to \cref{asst:Markov:noise} thus requires controlling a $ \nu $-th order moment, instead of all moments as implied by \cref{asst:subexp}.
\paragraph{Step 2: Deriving and unrolling the recursion on $ \delta^{\rm SGDw}_{t,1} $}
The right-hand side  of \cref{eq:varepsilon:new} does not depend on $t$, allowing for the derivation of an ``unrollable'' recursion on $ \delta^{\rm SGDw}_{t,1} $ and its subsequent unrolling, which is performed in the following theorem.
For simplicity, we again assume $ \beta=0 $ and $ N=1 $ and defer the general case to \cref{thm:SGDw:full} in appendix.
\begin{theorem}[Convergence up to a tail control] \label{thm:SGDw} Assume \ref{asst:SGD-A1}, \ref{asst:SGD-A4}, \ref{asst:SGD-A5}, \ref{asst:Markov:noise} and \ref{asst:Markov:noise:X1}. Let $\eta_t = C $ for some $C > 0$, and assume that $\tilde \mu_m \;=\; \mu- \alpha^m \sigma C_\chi \;>\; 4 C L^2\;$. Then for any $\epsilon > 0$,
\begin{align*}
    \sqrt{\delta^{\rm SGDw}_{T, 1}}
    \hspace{-.1em}
    \leq
    \hspace{-.1em}
        E_1^{T,1}
        \sqrt{\delta^{\rm SGDw}_{0, 0}}
        \,+\,
        C \left (\varepsilon^{\rm SGDw}_{n,m,T;\nu}(\epsilon) + \mfrac{5 \sigma + 5\kappa_{\nu, m}}{\sqrt{n}} \right )
        \Big( 
            \mfrac{e^{\frac{\tilde \mu_m C}{2} }}{\tilde \mu_m C} 
            +
            \mfrac{E_2^{T,1}}{L^2C^2}  \,  
        \Big)
\end{align*}
where $ \varepsilon^{\rm SGDw}_{n,m,T;\nu}(\epsilon) $ is defined in \cref{eq:varepsilon:new}.
\end{theorem}
Note that in the general, non-full batch $ B \leq n $ case, $ \frac{ 5 \sigma + 5 \kappa_{\nu, m} }{ \sqrt {n}  } $ is replaced by $ \frac{ 5 \sigma + 5 \kappa_{\nu, m} }{ \sqrt {B}  } $ (see \cref{thm:SGDw:full}). Under our bounds, obtaining consistency thus requires setting $ B \equiv B(n)  \underset{n  \to \infty}{ \to} + \infty $.

\vspace{-0.5em}
\paragraph{Step 3: Controlling the tail probability term}
\cref{thm:SGDw} is just one step away from the final bound of \cref{lem:tail:prob:subexp}: it remains to control the tail term $\varepsilon^{\rm SGDw}_{n,m,T;\nu}(\epsilon)$.
Under the assumptions of \cite{jiang2018convergence}, minimizing $\varepsilon^{\rm SGDw}_{n,m,T;\nu}(\epsilon)$ over $ \epsilon $ yields the following result:

\begin{lemma} 
    \label{lem:varepsilon:subexp} 
    Assume the setup of \cref{lem:tail:prob:subexp}. Let $n \in \N$ be sufficiently large s.t.~$\frac{\log n}{n} < \frac{\sigma^2_m \zeta^2_m}{p + \nu-2}$. Denote $r_\Psi$ as the radius of the smallest sphere in $\R^p$ that contains $\Psi$, which is finite under \ref{asst:SGD-A1}.
    Then 
\begin{align*}
    \inf_{\epsilon > 0} & \varepsilon^{\rm SGDw}_{\nu;n,m,T}(\epsilon) \leq
    \\
    &\,
    \bigg( 
        \mfrac{3\sigma_m \sqrt{p ( (\nu-2)p +2 \nu)} }{\sqrt{\nu - 2}}
        +  
        \kappa_{\nu;m}
        2^{\frac{\nu-2}{2 \nu}}
        (r_\Psi)^{\frac{(\nu -2)p}{2 \nu}} 
        \Big( 1 + \mfrac{2 C_m (\nu-2)^{1/2}}{\sigma_m p^{1/2} ((\nu-2)p + 2\nu)^{1/2}} \Big)^{\frac{(\nu -2)p}{2 \nu}}
        \bigg) 
    \, \mfrac{\sqrt{\log n}}{\sqrt{n}}
    .
\end{align*}
\end{lemma}
To obtain this result, we control the moment term $ (\sup_{t}\mean[ \Delta(\psi^{\rm SGDw}_{t,1})^\nu]) $ using \ref{asst:Markov:noise:X1},
and we control the tail probability term $\sup_{t}\P(\Delta(\psi^{\rm SGDw}_{t,1}) > \epsilon)$ as in \citep[Lemma 3.1]{jiang2018convergence} using an union bound, a covering argument and \ref{asst:subexp}.
\cref{lem:tail:prob:subexp} then follows by plugging \cref{lem:varepsilon:subexp} into \cref{thm:SGDw}.

\subsection{Consistency of offline CD: beyond subexponential tails.}
As discussed above, the general unrolling result of \cref{thm:SGDw} holds without the subexponentiality \cref{asst:subexp}; this assumption was only used in 
\cref{lem:varepsilon:subexp} to control 
$\varepsilon^{\rm SGDw}_{n,m,T;\nu}(\epsilon)$.
We now discuss two alternative ways to control this quantity without requiring subexponential tails.
The first generalizes the idea of Jiang et al. \cite{jiang2018convergence}, while the second exploits mixing of the Markov chain $K^m_\psi(x)$ as $m \rightarrow \infty$.
As before we only state partial results (full batch, $ \beta=0 $) and defer the full explicit bounds to \cref{appendix:tail:results}.

\paragraph{Control via Markov Inequality} 
Our first alternative uses Markov Inequality to yield the following.
\begin{theorem} \label{lem:tail:prob:subopt} Assume the setup of \cref{thm:SGDw} and additionally that \ref{asst:Jiang} holds. Then
\begin{align*}
    \inf_{\epsilon > 0} \varepsilon^{\rm SGDw}_{\nu;n,m,T}(\epsilon)
    \;\leq&\;
   \tilde C(p,\nu,m, \Psi)
    \,
     n^{  - \frac{(\nu-2)\nu}{2(\nu^2+(\nu-2)p)} }
    \;,\;
    \\
    \lim_{T \rightarrow \infty} \sqrt{\delta^{\rm SGDw}_{T,1}}
    \;\leq&\;
    \tilde C'(p,\nu,m,\Psi) \big(  n^{  - \frac{(\nu-2)\nu}{2(\nu^2+(\nu-2)p)} } + \mfrac{1}{\sqrt{n}} \big)
    ,
\end{align*}
where $\tilde C$ and $\tilde C'$ are functions whose explicit expressions are given in \cref{lem:tail:prob:subopt:full}  in the appendix.
\end{theorem}
In the case $p=1$ and $\nu=3$, the sub-optimal error from \cref{lem:tail:prob:subopt} reads $O(n^{-3/20})$. \cref{lem:tail:prob:subexp,lem:tail:prob:subopt} reveal that, depending on the tail condition imposed on the noise introduced by the Markov kernel, the convergence rate of offline CD varies: A subexponential tail, as assumed in prior work, in fact leads to near-parametric rate. Meanwhile, consistency can be obtained without assuming subexponentiality, albeit at a sub-optimal rate. 

\textbf{Control via Markov chain mixing.}~Alternatively, notice that $\mean[\Delta(\psi^{\rm SGDw}_{t,1})^2]$ involves an average of~$\mean\big[ \phi\big(K^m_{\psi^{\rm SGDw}_{t,1}}(X_i)^2 \big) \big| X_i, \psi^{\rm SGDw}_{t,1} \big] - \mean\big[ \phi\big(K^m_{\psi^{\rm SGDw}_{t,1}}(X'_1) \big) \big| \psi^{\rm SGDw}_{t,1} \big]$. When $m \rightarrow \infty$, the effect of initialization vanishes, and one may expect the difference to converge to zero. We defer to \cref{lem:tail:prob:ergodicity} in the appendix to show that, under a $\phi$-discrepancy mixing condition (\cite{rabinovich2020function}) with a mixing coefficient $\tilde \alpha \in [0,1)$, 
\begin{align*}
    \inf_{\epsilon > 0} \varepsilon^{\rm SGDw}_{\nu;n,m,T}(\epsilon)
    \;=&\;
    O( \kappa_{\nu;m} \tilde \alpha^{\frac{(\nu - 2) m}{3\nu - 2}} )
    &\text{ and }&&
    \lim_{T \rightarrow \infty} \sqrt{\delta^{\rm SGDw}_{T,1}}
    \;=&\;
    O\Big( \kappa_{\nu;m} \tilde \alpha^{\frac{(\nu - 2) m}{3\nu - 2}} + \mfrac{\sigma+\kappa_{\nu;m} }{\sqrt{n}} \Big)
    \;.
\end{align*}
As $m \rightarrow \infty$, this recovers the parametric rate $O(n^{-1/2})$.

\vspace{.5em}

\begin{remark*}[Examples] In our main results (\cref{thm:cd-sgd,thm:online-cd-averaging,thm:SGDw}) and the tail condition for offline SGD (\cref{lem:tail:prob:subexp}), we employed a weaker set of assumptions than those in \cite{jiang2018convergence} (except for the mild $\nu > 2$ moment assumption in \eqref{asst:Markov:noise}). Consequently, our results apply to all three examples studied in \cite{jiang2018convergence}: A bivariate Gaussian model with unknown mean and random-scan Gibbs sampler, a fully visible Boltzmann machine with random-scan Gibbs sampler, and an exponential-family random graph model with a Metropolis-Hastings sampler.  
\end{remark*}

\vspace{-0.5em}
\section{Related Work}

\vspace{-0.5em}
Central to this paper is the prior work of Jiang et al.~\cite{jiang2018convergence}, which 
provided a rigorous theoretical foundation 
% pioneered
% a fertile ground 
to analyze the convergence of full-batch CD, and which we refine. 
The study of optimization with biased gradient descent has attracted a lot of attention in recent years \cite{karimi2019non,sun2018markov, doan2022finite, atchade2017perturbed}. These works, while closely connected to ours,
analyze algorithms with different implementation choices than the CD algorithm: i.i.d.~noise setup \cite{karimi2019non}, or setup where a persistent Markov chain is maintained through the iterations \cite{karimi2019non,sun2018markov, doan2022finite, atchade2017perturbed}. The latter is akin to a variant of the CD algorithm, called the persistent CD \citep{tieleman2008training}. In contrast, our analysis focus on the CD algorithm that restarts a batch of Markov chains from the data distribution at every iteration. Finally, there is a rich body of work on convergence guarantees for offline multi-pass SGD \citep{elisseeff2005stability,hardt2016train,lin2017optimal,pillaud2018statistical,mucke2019beating,zou2022risk}. A notable difference of our analysis is that we are primarily concerned with statistical errors associated with convergence to the true parameter $\psi^*$ in number of samples $n$, and not the commonly studied convergence rate in number of epochs $T$. Consequently, most of our work for the offline setup goes into handling the correlations that accumulate by reusing data across epochs.

\vspace{-0.5em}
\section{Discussion}\label{sec:discussion}

\vspace{-0.5em}
In this work, we provide a non-asymptotic analysis of the Contrastive Divergence algorithms,
showing, in the online setting, their potential to converge at the parametric rate and
to have near-optimal asymptotic variance, and proving a near-parametric rates in the
offline setting, 
significantly extending prior results.
Our results apply to
unnormalized exponential families: despite their flexibility, these models only
cover log-densities with linear relationships on the model parameters. We believe
that extending our results to more general forms of unnormalized models is an important direction for
future work.
\section*{Acknowledgments and Disclosure of Funding}
All authors acknowledge support from the Gatsby Charitable Foundation.

\clearpage
\bibliographystyle{unsrt}
\bibliography{biblio}
\newpage

\appendix
{\bfseries \Large Supplementary Material for ``Near-Optimality of Contrastive Divergence Algorithms''}\\

The supplementary material provides the proofs of the main results of the paper:

Section \ref{sec:appendix:offline:SGD} states full explicit bounds for the offline CD algorithm. 

Section \ref{appendix:tools} collects a list of useful tools for our proofs. These include the properties of $\varphi_\gamma$ introduced before \cref{thm:cd-sgd} in the main text, as well as several contraction and integrability results.

Section \ref{sec:proofs-online-CD} provides the proofs for the online CD algorithm. 

Sections \ref{appendix:L2:approx},
\ref{appendix:proofs:offline} and \ref{appendix:proofs:SGDw:tail} contain the proofs about the offline CD algorithm and the tail control.

\section{Notations}\label{sec:appendix:notations}

Throughout the proofs, we will denote by $ \mathrm{P}^{m}_{\psi} $ the
following operator from $ L^{2}(p_\psi)$ to itself:
\begin{equation} \label{eq:Pm-def}
\begin{aligned}
  \mathrm{P}^{m}_{\psi}f(x) \coloneqq \mint_{  }^{  } k^{m}(x, x') f(x') p_{\psi}(x') \text{d}x'.
\end{aligned}
\end{equation}
Here, $ k^{m}(x, x') $ is the $ m $-iterated version of some Markov transition kernel $ k_{\psi} $, e.g.: 

\begin{equation} \label{eq:km-def}
\begin{aligned}
	k_{\psi}^m(x, x') \coloneqq \mint_{  }^{  } 
	k_{\psi}(x, x_1) \dots k_{\psi}(x_{m-2}, x_{m-1}) \dots {k_{\psi}(x_{m-1}, x')} \mathrm{d}x_1 \dots \mathrm{d}x_{m-1}.
\end{aligned}
\end{equation}

$ \mathrm{Proj}_{\Psi}: \mathbb{R}^p  \longmapsto \Psi $ denotes the projection operator onto the convex
set $ \Psi $, e.g. 
\begin{equation*} 
\begin{aligned}
\mathrm{Proj}_{\Psi}(\psi) \coloneqq \argmin_{ \psi' \in \Psi } \left \| \psi - \psi' \right \|.
\end{aligned}
\end{equation*}

We also frequently use the following function, used in standard convex optimization results \cite{moulines2011non}.
\begin{equation*} 
\begin{aligned}
\varphi_{\gamma}(t) = 
\begin{cases}
 \frac{ t^{\gamma} - 1 }{ \gamma } & \text{ if } \gamma \neq 0 \\
 \log t & \text{ if } \gamma = 0
\end{cases} 
\end{aligned}
\end{equation*}
which is defined on $ \mathbb{R}_{+} \setminus \{0 \} $.
\vspace{1em}

\noindent
Finally,   we recall the notation $K^m_\psi(x) \sim k^m_\psi(x, \argdot)$. We will note, in the proofs regarding online CD,
$ X^{\psi} $ a sample drawn independently of $X_1, \ldots, X_n$ and on a given $\psi \in \Psi$, as 
\begin{align*}
    X^{\psi}  \;{\sim}\; p_{\psi}\;.
\end{align*}
\paragraph{Notations for Offline CD} 
For offline CD, we will note by $ X^{\psi}_1, \ldots, X^{\psi}_n $ 
the i.i.d.~samples, drawn independently of $X_1, \ldots, X_n$ and on a given $\psi \in \Psi$, as 
\begin{align*}
    X^{\psi}_1, \ldots, X^{\psi}_n 
    \;\overset{\rm i.i.d.}{\sim}\;
    p_{\psi}\;.
\end{align*}
We also write, for $m \in \N \cup \{\infty\}$,
\begin{align*}
    K^m_{1;\psi}(x) \,,\, \ldots \,,\, K^m_{n;\psi}(x)
    \;\overset{\rm i.i.d.}{\sim}  k^m_\psi(x, \argdot)
\end{align*}

\section{Additional results for offline SGD}  \label{sec:appendix:offline:SGD}

In this section, we provide the full statements on error bounds for SGD with replacement (SGDw), SGD with reshuffling (SGDo) and tail moment bounds, which complement the results in \cref{sec:offline:SGD}. Proofs are deferred to \cref{appendix:proofs:offline}, which make use of $L_2$ approximation by auxiliary gradient updates derived in \cref{appendix:L2:approx}.

\paragraph{Notations}
Denote the SGDw iterates by $(\psi^{\rm SGDw}_{t,j})_{t \in \N, j \leq N}$ and let $X'_1$ be an i.i.d.~copy of $X_1$.
Throughout the remaining of the appendix, we define given $\epsilon > 0$ and $n \in \N$
\begin{align*}
    \vartheta^{\rm SGDw}_{n,m,T} (\epsilon) 
    \hspace{-.2em}
    &\coloneqq
    \hspace{-.2em}
    \sup_{\substack{t \in [T] \\ j \in [N]}} 
    \hspace{-.2em}
    \P\Bigg(  \mfrac{ \Big\|  \sum_{i=1}^n \Big( \mean\Big[ \phi\Big(K^m_{\psi^{\rm SGDw}_{t-1, j}}(X_i) \Big) \Big| X_i, \psi^{\rm SGDw}_{t-1,j} \Big]  
    - 
    \mean\Big[ \phi\Big(K^m_{\psi^{\rm SGDw}_{t-1,j}}(X'_1)\Big)  \Big| \psi^{\rm SGDw}_{t-1,j} 
    \Big] \Big) \Big\| }{n}
    \hspace{-.2em}
    > \epsilon
    \hspace{-.2em}
    \Bigg). \\
    &= \sup_{t, j} \P( \Delta(\psi^{\rm SGDw}_{t,j}) > \epsilon)
\end{align*}
and $\vartheta^{\rm SGDo}_{n,m,T} (\epsilon)$ analogously.
Using these notations, we can redefine the quantity  $\varepsilon^{\rm SGDw}_{n,m,T;\nu}(\epsilon)$ in the main as
$
    \varepsilon^{\rm SGDw}_{n,m,T;\nu}(\epsilon)
    \coloneqq
    \sqrt{\epsilon^2 
    + 
    \kappa_{\nu;m}^2
    \big(\vartheta^{\rm SGDw}_{n,m,T}(\epsilon)\big)^{\frac{\nu-2}{\nu}}}
$.
%, consider \cref{asst:Markov:noise:X1}

\subsection{An explicit finite-sample bound for SGDw}

% Recall that $\varphi_{\beta}(t) = \frac{ t^{\beta} - 1 }{ \beta }$ if $ \beta \neq 0 $, and $ \log t $ if $ \beta = 0 $. 
In the result below, we write $\delta^{\rm SGDw}_{t, j} \coloneqq \mean \, \big\| \psi^{\rm SGDw}_{t,j}  - \psi^* \big\|^2$ and, for a fixed $\epsilon > 0$, the quantity 
\begin{align*}
    \sigma^{\rm SGDw}_{n,T} 
    \;=&\; 
    \varepsilon^{\rm SGDw}_{n,m,T;\nu}(\epsilon) + \mfrac{5 \sigma + 5\kappa_m}{\sqrt{B}}
    \;=\;
    \sqrt{\epsilon^2 
    + 
    \kappa_m^2
    \big(\vartheta^{\rm SGDw}_{n,m,T}(\epsilon)\big)^{\frac{\nu-2}{\nu}}}
    + \mfrac{5 \sigma + 5\kappa_m}{\sqrt{B}}
    \;.
\end{align*}

\begin{theorem}\label{thm:SGDw:full} 
Assume \ref{asst:SGD-A1} (where $\Psi$ may be non-compact), \ref{asst:SGD-A4}, \ref{asst:SGD-A5}, \ref{asst:Markov:noise} and \ref{asst:Markov:noise:X1}. Let $\eta_t = C t^{-\beta}$ for some $\beta \in [0,1]$ and $C > 0$, and assume that $m > \frac{\log (\sigma C_\chi / \mu )}{\log |\alpha|}$ s.t.~$\tilde \mu_m  = \mu - \alpha^m \sigma C_\chi > 0$ as in \cref{thm:cd-sgd}. Then for any $\epsilon > 0$, $\sqrt{\delta^{\rm SGDw}_{T, N}}$ is upper bounded by
\begin{align*}
    \begin{cases}
        E_1^{T,N}
        \sqrt{\delta^{\rm SGDw}_{0, 0}}
        \,+\,
        C \sigma^{\rm SGDw}_{n,T}
        \Big( 
            \mfrac{4 e^{\frac{\tilde \mu_m C  N}{(T+1)^{1/2}}}}{\tilde \mu_m C}
            +
            2 N ( 1 + \tilde \mu_m C)^{N-1} 
            \varphi_{\frac{1}{2}-L^2C^2 N}(T+1)
            \, E_2^{T,N}
        \Big)
        \\
        \hspace{34.4em}
        \text{ for } \beta = \mfrac{1}{2} 
        \;,
        \\[.8em]
        E_1^{T,N}
        \sqrt{\delta^{\rm SGDw}_{0, 0}}
        +
        C \sigma^{\rm SGDw}_{n,T}
        \Big(
        \mfrac{4}{\tilde \mu_m C}
        +
        \mfrac{3 N \big(1+\frac{L^2C^2}{2}\big)^{N-1} \, e^{2 L^2 C^2 N}  \, \log(T+1)}{(T+1)^{(\tilde \mu_m C N)/2}}
        \Big)
        \hspace{4em}
        \text{ for } \beta = 1
        \;,
        \\[.8em]
        E_1^{T,N}
        \sqrt{\delta^{\rm SGDw}_{0, 0}}
        \,+\,
        C \sigma^{\rm SGDw}_{n,T} 
        \Big( 
            \mfrac{2^{2\beta+1}}{\tilde \mu_m C} 
            e^{\frac{\tilde \mu_m C}{2(1-\beta)} \frac{N}{(T+1)^\beta}}
            +
            \mfrac{3^\beta (1+\tilde \mu_m C)^{N-1} (T+2)^{\beta} }{L^2C^2}  \,  E_2^{T,N}
        \Big)
        \hspace{.2em}
        \text{ otherwise}
        \;,
    \end{cases}
\end{align*}
where $E_1^{T,N}$ and $E_2^{T,N}$ are two decreasing functions in $T$ defined by
\begin{align*}
    E_1^{T,N}
    \;\coloneqq&\;
    \exp\Big( 
        1 
        - 
        N \tilde \mu_m C \varphi_{1-\beta}(T+1)
        +
        \mfrac{ N L^2 C^2}{2} \varphi_{1-2\beta}(T+1)
    \Big)
    \;,
    \\
    E_2^{T,N} 
    \;\coloneqq&\; 
    \exp\Big( 
        - \mfrac{N \tilde \mu_m C}{2} \varphi_{1-\beta}(T+1)
        + 2 N L^2 C^2 \varphi_{1-2\beta}(T+1)
     \Big)
     \;.
\end{align*}
\end{theorem}

We emphasize that the full result above holds for any $\beta \in [0,1]$, which in particular includes the constant step size $\beta=0$ regime considered by \cite{jiang2018convergence}. When $\beta=0$, for $E^{T,N}_1$ and $E^{T,N}_2$ to decay to zero as $T \rightarrow \infty$, we additionally need the condition 
\begin{align*}
    \tilde \mu_m \;=\; \mu- \alpha^m \sigma C_\chi \;>\; 4 C L^2\;.
\end{align*}
This is almost identical to the condition used in \cite[Equation 2.5, Theorem 2.1]{jiang2018convergence}, except that $4L^2$ gets replaced by $\frac{1}{2}(L + \alpha^m \sigma C_\chi)^2$. Notably this says that an additional step size condition is needed for our results to hold in the constant step size regime, but not necessary for a decreasing step size.

\subsection{Results for SGDo}
\label{appendix:SGDo:results}

SGD with reshuffling (SGDo, also called SGD without replacement) is an optimization scheme that is also widely used in practice compared to SGDo and online SGD. In the context of CD, it corresponds to \cref{alg:offline_cd} with batches chosen as 
\begin{align*}
    (B_{t,1}, \ldots, B_{t,N}) 
    \;=&\; 
    \pi(\{1, \ldots, n\})\;,
\end{align*}
where $\pi$ is a uniform draw of the permutation group on $n$ elements. We denote the iterates of SGDo $(\psi^{\rm SGDo}_{t,j})_{t \in \N, j \in [N]}$. Analogously to $\vartheta^{\rm SGDw}_{n,m,T}$, we define, for $X'_1$ an i.i.d.~copy of $X_1$, $\epsilon > 0$ and $n \in \N$, the tail probability term
\begin{align*}
    \vartheta^{\rm SGDo}_{n,m,T} (\epsilon) 
    \hspace{-.2em}
    \coloneqq
    \hspace{-.2em}
    \sup_{\substack{t \in [T] \\ j \in [N]}} 
    \hspace{-.2em}
    \P\Bigg(  \mfrac{ \Big\|  \sum_{i=1}^n \Big( \mean\Big[ \phi\Big(K^m_{\psi^{\rm SGDo}_{t-1, j}}(X_i) \Big) \Big| X_i, \psi^{\rm SGDo}_{t-1,j} \Big]  
    - 
    \mean\Big[ \phi\Big(K^m_{\psi^{\rm SGDo}_{t-1,j}}(X'_1)\Big)  \Big| \psi^{\rm SGDo}_{t-1,j} 
    \Big] \Big) \Big\| }{n}
    \hspace{-.2em}
    > \epsilon
    \hspace{-.2em}
    \Bigg)\;.
\end{align*}
Also denote $\varepsilon^{\rm SGDo}_{n,m,T;\nu}(\epsilon) = \sqrt{\epsilon^2 
+ 
\kappa_m^2
\big(\vartheta^{\rm SGDo}_{n,m,T}(\epsilon)\big)^{\frac{\nu-2}{\nu}}}$ 
and 
$
\sigma^{\rm SGDw}_{n,T} 
=
\varepsilon^{\rm SGDo}_{n,m,T;\nu}(\epsilon) + \mfrac{5 \sigma + 5\kappa_m}{\sqrt{B}}
$.
The following result says that $\psi^{\rm SGDo}_{t,j}$ enjoys exactly the same convergence guarantee as $\psi^{\rm SGDo}_{t,j}$ in \cref{thm:SGDw}. The statement is identical to that of \cref{thm:SGDw:full} and is stated in full for completeness; see \cref{appendix:proof:SGDo} for the proof, which is a slight adaptation of the proof for \cref{thm:SGDw:full}. As before we write $\delta^{\rm SGDo}_{t, j} \coloneqq \mean \, \big\| \psi^{\rm SGDo}_{t,j}  - \psi^* \big\|^2$.

\begin{theorem}[Convergence of CD-SGDo]  \label{thm:SGDo} Assume \ref{asst:SGD-A1} (where $\Psi$ may be non-compact), \ref{asst:SGD-A4}, \ref{asst:SGD-A5}, \ref{asst:Markov:noise} and \ref{asst:Markov:noise:X1}. Let $\eta_t = C t^{-\beta}$ for some $\beta \in [0,1]$ and $C > 0$, and assume that $m > \frac{\log (\sigma C_\chi / \mu )}{\log |\alpha|}$ s.t.~$\tilde \mu_m  = \mu - \alpha^m \sigma C_\chi > 0$ as in \cref{thm:cd-sgd}. Then for any $\epsilon > 0$, $\sqrt{\delta^{\rm SGDo}_{T, N}}$ is upper bounded by
    \begin{align*}
        \begin{cases}
            E_1^{T,N}
            \sqrt{\delta^{\rm SGDo}_{0, 0}}
            \,+\,
            C \sigma^{\rm SGDo}_{n,T}
            \Big( 
                \mfrac{4e^{\frac{\tilde \mu_m C  N}{(T+1)^{1/2}}}}{\tilde \mu_m C} 
                +
                2 N ( 1 + \tilde \mu_m C)^{N-1} 
                \varphi_{\frac{1}{2}-L^2C^2N}(T+1)
                \, E_2^{T,N}
            \Big)
            \\
            \hspace{34.3em}
            \text{ for } \beta = \mfrac{1}{2} 
            \;,
            \\[.8em]
            E_1^{T,N}
            \sqrt{\delta^{\rm SGDo}_{0, 0}}
            +
            C \sigma^{\rm SGDo}_{n,T}
            \Big(
            \mfrac{4}{\tilde \mu_m C}
            +
            \mfrac{3 N \big(1+\frac{L^2C^2}{2}\big)^{N-1} \, e^{2 L^2 C^2 N}  \, \log(T+1)}{(T+1)^{(\tilde \mu_m C N)/2}}
            \Big)
            \hspace{4.1em}
            \text{ for } \beta = 1
            \;,
            \\[.8em]
            E_1^{T,N}
            \sqrt{\delta^{\rm SGDo}_{0, 0}}
            \,+\,
            C \sigma^{\rm SGDo}_{n,T} 
            \Big( 
                \mfrac{2^{2\beta+1}}{\tilde \mu_m C} 
                e^{\frac{\tilde \mu_m C}{2(1-\beta)} \frac{N}{(T+1)^\beta}}
                +
                \mfrac{3^\beta (1+\tilde \mu_m C)^{N-1} (T+2)^{\beta} }{L^2C^2}  \,  E_2^{T,N}
            \Big)
            \hspace{.4em}
            \text{ otherwise}
            \;,
        \end{cases}
    \end{align*}
    where $E_1^{T,N}$ and $E_2^{T,N}$ are two decreasing functions in $T$ defined by
    \begin{align*}
        E_1^{T,N}
        \;\coloneqq&\;
        \exp\Big( 
            1 
            - 
            N \tilde \mu_m C \varphi_{1-\beta}(T+1)
            +
            \mfrac{ N L^2 C^2}{2} \varphi_{1-2\beta}(T+1)
        \Big)
        \;,
        \\
        E_2^{T,N} 
        \;\coloneqq&\; 
        \exp\Big( 
            - \mfrac{N \tilde \mu_m C}{2} \varphi_{1-\beta}(T+1)
            + 2 N L^2 C^2 \varphi_{1-2\beta}(T+1)
         \Big)
         \;.
    \end{align*}
\end{theorem}

\vspace{.5em}

\begin{remark*} We also remark that existing works \cite{nagaraj2019sgd,rajput2020closing} show that the standard SGDo typically gives a faster convergence rate in $T$ than SGDw. An analogous result for the CD setup would involve additional technical hurdles of jointly controlling the correlations across minibatches and from reusing data samples, and we defer this to future work.
\end{remark*}

\subsection{Explicit tail control}
\label{appendix:tail:results}

We now provide the full explicit tail control bounds. All results in this section hold directly for $\varepsilon^{\rm SGDo}_{\nu;n,m,T}(\epsilon)$ and $\delta^{\rm SGDo}_{T,N}$, and we omit them here. In the result below, we denote $r_\Psi$ as the radius of the smallest sphere in $\R^p$ that contains $\Psi$, which is finite under \ref{asst:SGD-A1}.

\begin{lemma} 
    \label{lem:tail:prob:subexp:full} 
    Assume \ref{asst:Jiang} and \ref{asst:subexp}. Let $n \in \N$ be sufficiently large s.t.~$\frac{\log n}{n} < \frac{\sigma^2_m \zeta^2_m}{p + \nu-2}$. Then 
\begin{align*}
    \inf_{\epsilon > 0} & \varepsilon^{\rm SGDw}_{\nu;n,m,T}(\epsilon)
    \;\leq\;
    \\
    &\, 
    \bigg( 
        \mfrac{3\sigma_m \sqrt{p ( (\nu-2)p +2 \nu)} }{\sqrt{\nu - 2}}
        +  
        \kappa_{\nu;m}
        2^{\frac{\nu-2}{2 \nu}}
        (r_\Psi)^{\frac{(\nu -2)p}{2 \nu}} 
        \Big( 1 + \mfrac{2 C_m (\nu-2)^{1/2}}{\sigma_m p^{1/2} ((\nu-2)p + 2\nu)^{1/2}} \Big)^{\frac{(\nu -2)p}{2 \nu}}
        \bigg) 
    \, \mfrac{\sqrt{\log n}}{\sqrt{n}}
    \;.
\end{align*}
In particular, if we additionally assume the conditions of \cref{thm:SGDw:full}, we have
\begin{align*}
    \lim_{T \rightarrow \infty} \sqrt{\delta^{\rm SGDw}_{T,N}} 
    \;\leq\; 
    C'(p,\nu,m,\Psi,\beta) \Big( \mfrac{\sqrt{\log n}}{\sqrt{n}} + \mfrac{1}{\sqrt{B}} \Big)
\end{align*}
where
\begin{align*}
    &C'(p,\nu,m,\Psi,\beta)
    \;\coloneqq\; 
    \mfrac{8(1+ 5 \sigma + 5\kappa_m)}{\tilde \mu_m}
    \\
    &\;\times\;
    \bigg( 
        \mfrac{3\sigma_m \sqrt{p ( (\nu-2)p +2 \nu)} }{\sqrt{\nu - 2}}
        +  
        \kappa_{\nu;m}
        2^{\frac{\nu-2}{2 \nu}}
        (r_\Psi)^{\frac{(\nu -2)p}{2 \nu}} 
        \Big( 1 + \mfrac{2 C_m (\nu-2)^{1/2}}{\sigma_m p^{1/2} ((\nu-2)p + 2\nu)^{1/2}} \Big)^{\frac{(\nu -2)p}{2 \nu}}
        \bigg) 
    \;.
\end{align*}

\end{lemma}

\begin{lemma} 
    \label{lem:tail:prob:subopt:full}
     Assume the conditions of \cref{thm:SGDw} and additionally that \ref{asst:Jiang} holds. Then
\begin{align*}
    \inf_{\epsilon > 0} \varepsilon^{\rm SGDw}_{\nu;n,m,T}(\epsilon)
    \;\leq&\;
    \big(
    3 C_m 
    + 
   \kappa_{\nu;m}^{\nu / 2}
    (r_\Psi)^{\frac{(\nu -2) p}{2 \nu}}
    C_m^{- \frac{\nu-2}{2}}
    3^{\frac{(\nu-2) p}{2 \nu}}
    \big) 
    \,
      n^{  - \frac{(\nu-2)\nu}{2(\nu^2+(\nu-2)p)} }
    \;,
    \\
    \lim_{T \rightarrow \infty} \big(\delta^{\rm SGDw}_{T,N}\big)^{1/2} 
    \;\leq&\; 
    \tilde C'(p,\nu,m,\Psi,\beta) \big( n^{  - \frac{(\nu-2)\nu}{2(\nu^2+(\nu-2)p)} } + B^{-1/2} \big)
    \;,
\end{align*}
where 
\begin{align*}
    \tilde C'(p,\nu,m,\Psi)
    \;\coloneqq\; 
    \mfrac{8(1+ 5 \sigma + 5\kappa_m)}{\tilde \mu_m}
    \big(
    3 C_m 
    + 
   \kappa_{\nu;m}^{\nu / 2}
    (r_\Psi)^{\frac{(\nu -2) p}{2 \nu}}
    C_m^{- \frac{\nu-2}{2}}
    3^{\frac{(\nu-2) p}{2 \nu}}
    \big) 
    \;.
\end{align*}
\end{lemma}

The next result considers a $\phi$-discrepancy mixing condition (\cite{rabinovich2020function}), which is a mixing assumption on $K^m_\psi$ but with respect to a specific test function $\phi$, and we impose it uniformly over $\psi \in \Psi$. We also recall that $X_1^\psi \sim p_\psi$.

\begin{lemma} 
    \label{lem:tail:prob:ergodicity} 
    Assume that there exist $\tilde \alpha \in [0,1)$ and $\tilde C_K > 0$ such that, for all $\psi \in \Psi$ and $x \in \cX$, $\| \mean[ \phi(K^m_\psi(x)) ] - \mean[\phi(X_1^\psi)] \| \leq \tilde C_K \tilde \alpha^m$. Then
\begin{align*}
    \minf_{\epsilon > 0} \, \varepsilon^{\rm SGDw}_{\nu;n,m,T}(\epsilon)
    \;\leq\;
    \big( 1
    +
    2^{\frac{\nu-2}{2\nu}}
    \kappa_{\nu;m}
    (\tilde C_K)^{\frac{\nu - 2 }{2\nu}}
    \big) \tilde \alpha^{\frac{(\nu - 2) m}{3\nu - 2}}
    \;.
\end{align*} 
In particular, if we additionally assume the conditions of \cref{thm:SGDw}, we have
\begin{align*}
    \lim_{T \rightarrow \infty} \sqrt{\delta^{\rm SGDw}_{T,N}} 
    \;\leq\; 
    \mfrac{8}{\mu - \alpha^m \sigma C_\chi} 
    \Big(  
    \big( 1
    +
    2^{\frac{\nu-2}{2\nu}}
    \kappa_{\nu;m}
    (\tilde C_K)^{\frac{\nu - 2 }{2\nu}}
    \big) \tilde \alpha^{\frac{(\nu - 2) m}{3\nu - 2}}
    + \mfrac{5 \sigma + 5 \kappa_m}{\sqrt{B}} \Big)
    \;.
\end{align*}
\end{lemma}

\section{Auxiliary Tools} \label{appendix:tools}

\subsection{Properties of \texorpdfstring{$\varphi_\gamma$}{ϕᵧ}} \label{appendix:prop:varphi}

The following lemma collects some identities used in \cite{moulines2011non}.

\begin{lemma}\label{lemma:phi-properties}
$ \varphi_\gamma $ satisfies the following properties:
\begin{proplist}
	\item $ \varphi_\gamma $ is increasing on $ \mathbb{R}_{+} $ for all $ \gamma $\;;
	\item $ \varphi_{\gamma}(t) \leq \frac{t^{\gamma}}{\gamma} $ for $ \gamma > 0 $, and $ \varphi_{\gamma}(t)\leq -\frac{1}{\gamma} $ for $ \gamma < 0 $\;;
	\item $ \varphi_{1-\beta}(t) \geq  t^{1 - \beta} $ for $ \beta \in (0, 1] $\;;
	\item $\varphi_{\gamma}(t) - \varphi_{\gamma}(\frac{t}{2}) \geq  \frac{1}{2}x^{\gamma}$ for $ \gamma \in (0, 1]$\;.
\end{proplist}
\end{lemma}

\vspace{.5em}

The next lemma provides some additional results on $\varphi_\gamma$.

\begin{lemma}  \label{lem:varphi} $\varphi_\gamma$ satisfies the following properties:
\begin{proplist}
    \item $\varphi_\gamma$ is positive on $t > 0$ and increasing for every $\gamma \in \R$;
    \item For $1 \leq t_1 \leq t_2$ and $\beta \geq 0$, we have 
    \begin{align*}
            \varphi_{1-\beta}(t_2+1)
            -
            \varphi_{1-\beta}(t_1)
            \;\leq\; 
            \msum_{t=t_1}^{ t_2 } t^{-\beta} 
            \;\leq\; 
            2 \,
            \big(
            \varphi_{1-\beta}(t_2+1)
            -
            \varphi_{1-\beta}(t_1)
            \big)
            \;.
        \end{align*}
    If instead $\beta < 0$, we have 
    \begin{align*}
            \mfrac{1}{2}
            \big(
            \varphi_{1-\beta}(t_2+1)
            -
            \varphi_{1-\beta}(t_1)
            \big)
            \;\leq\; 
            \msum_{t=t_1}^{ t_2 } t^{-\beta} 
            \;\leq\; 
            \varphi_{1-\beta}(t_2+1)
            -
            \varphi_{1-\beta}(t_1)
            \;;
        \end{align*}
    \item For $1 \leq t_1 \leq t_2$ and $\gamma \neq 0$, we have 
    \begin{align*}
        t_1^{\gamma-1}
        \;\leq&\;
        \varphi_{\gamma}(t_2)
        -
        \varphi_{\gamma}(t_1) 
        \;\leq\;
        t_2^{\gamma-1}
        \quad
        &&\text{ if } \gamma \geq 1\;,
        \\
        t_2^{-(1-\gamma)}
        \;\leq&\;
        \varphi_{\gamma}(t_2)
        -
        \varphi_{\gamma}(t_1) 
        \;\leq\;
        t_1^{-(1-\gamma)}
        \quad
        &&\text{ if } \gamma \leq 1\;;
    \end{align*}
    \item Let $1 \leq t_1 < t_2$ and $\kappa, \beta \geq 0$. If $\kappa \neq 1$ and $a > 0$, we have
    \begin{align*}
        \msum_{t=t_1}^{t_2} (t+1)^{-\beta} \exp\big( a \, \varphi_{1-\kappa}(t-1) \big)
        \;\leq\;
        \mfrac{(t_2+1)^{\max\{\kappa-\beta,0\}}}{a} \,  \exp\big(  a \, \varphi_{1-\kappa}(t_2+1)  \big)\;,
    \end{align*}
    and if $\kappa \neq 1$ and $a < 0$, we have 
    \begin{align*}
        \msum_{t=t_1}^{t_2} (t+1)^{-\beta}  \exp\big( a \, \varphi_{1-\kappa}(t) \big)
        \;\leq\;
        \mfrac{(t_2+1)^{\max\{\kappa-\beta,0\}}}{(-a)} \,
        \exp\big(  a \, \varphi_{1-\kappa}(t_1)  \big)
        \;.
    \end{align*}
\end{proplist}    
\end{lemma}

\begin{proof}[Proof of \cref{lem:varphi}] (i) follows from checking $\gamma > 0$, $\gamma=0$ and $\gamma < 0$ respectively. The first set of bounds in (ii) follow by noting that $t \mapsto t^{-\beta}$ is decreasing for $\beta \geq 0$:
\begin{align*}
    \msum_{t=t_1}^{t_2} t^{-\beta} 
    \;\geq&\; 
    \mint_{t_1}^{t_2+1} t^{-\beta} dt 
    \;=\;
    \varphi_{1-\beta}(t_2+1)
    -
    \varphi_{1-\beta}(t_1)
    \;,
    \\
    \msum_{t=t_1}^{t_2} t^{-\beta} 
    \;\leq&\; 
    2
    \msum_{t=t_1}^{t_2} (t+1)^{-\beta} 
    \;\leq\;
    2
    \mint_{t_1 - 1}^{t_2} (t+1)^{-\beta} dt 
    \;=\;
    2 
    \big(
    \varphi_{1-\beta}(t_2+1)
    -
    \varphi_{1-\beta}(t_1)
    \big)
    \;.
\end{align*}
The second set of bounds follows from noting that $t \mapsto t^{-\beta}$ is increasing for $\beta < 0$: 
\begin{align*}
    \msum_{t=t_1}^{t_2} t^{-\beta} 
    \;\geq&\;
    \mfrac{1}{2} \msum_{t=t_1}^{t_2} (t+1)^{-\beta} 
    \;\geq\; 
    \mfrac{1}{2} \mint_{t_1-1}^{t_2} (t+1)^{-\beta} dt 
    \;=\;
    \mfrac{1}{2}
    \big(
    \varphi_{1-\beta}(t_2+1)
    -
    \varphi_{1-\beta}(t_1)
    \big)
    \;,
    \\
    \msum_{t=t_1}^{t_2} t^{-\beta} 
    \;\leq&\; 
    \mint_{t_1}^{t_2+1} t^{-\beta} dt 
    \;=\;
    \varphi_{1-\beta}(t_2+1)
    -
    \varphi_{1-\beta}(t_1)
    \;.
\end{align*}
For (iii), we note that for $\gamma \neq 0$,
\begin{align*}
    \varphi_{\gamma}(t_2)
    -
    \varphi_{\gamma}(t_1)
    \;=&\;
    \mfrac{t_2^\gamma - t_1^\gamma}{\gamma}\;,
\end{align*}
so by the mean value theorem,
\begin{align*}
    \minf_{t_1 \leq t \leq t_2} t^{\gamma-1}
    \;\leq\;
    \varphi_{\gamma}(t_2)
    -
    \varphi_{\gamma}(t_1) 
    \;\leq\;
    \msup_{t_1 \leq t \leq t_2} t^{\gamma-1}
    \;.
\end{align*}
The desired bounds then follow from an explicit computation of the infimum and the maximum in each of the two cases $\gamma \geq 1$ and $\gamma \leq 1$. 

\vspace{.5em}

For (iv), we first consider the case $\kappa \neq 1$ and $a > 0$. Then 
\begin{align*}
    \msum_{t=t_1}^{t_2} (t+1)^{-\beta}  \exp\big( a \,& \varphi_{1-\kappa}(t-1)  \big)
    \;=\;
    e^{-\frac{a}{1-\kappa}}
    \msum_{t=t_1}^{t_2} (t+1)^{-\beta} \exp\Big( \mfrac{a t^{1-\kappa}}{1-\kappa}  \Big)
    \\
    \;\leq&\; 
    e^{-\frac{a}{1-\kappa}}
    \mmax_{t_1 \leq t \leq t_2} \Big( \mfrac{(t+1)^\kappa}{(t+1)^\beta} \Big)
    \msum_{t=t_1}^{t_2} (t+1)^{-\kappa} \exp\Big( \mfrac{a t^{1-\kappa}}{1-\kappa} \Big)
    \\
    \;\leq&\;
    e^{-\frac{a}{1-\kappa}}
    (t_2+1)^{\max\{\kappa-\beta, 0 \}}
    \msum_{t=t_1}^{t_2} (t+1)^{-\kappa} \exp\Big( \mfrac{a t^{1-\kappa}}{1-\kappa}  \Big)
    \;.
\end{align*}
Since, for $x \geq 0$, $x \mapsto (x+1)^{-\kappa}$ is decreasing and $x \mapsto \exp( ax^{1-\kappa} / (1-\kappa))$ is increasing, we have that for $t_1 \leq t \leq t_2$ and $x \in [t,t+1]$,
\begin{align*}
    (t+1)^{-\kappa} \;\leq&\; x^{-\kappa}
    &\text{ and }&&
    \exp( a t^{1-\kappa} / (1-\kappa)) \;\leq&\; \exp( ax^{1-\kappa} / (1-\kappa))\;.
\end{align*}
This implies that
\begin{align*}
    \msum_{t=t_1}^{t_2} (t+1)^{-\beta}  \exp\big( a \,  & \varphi_{1-\kappa}(t-1) \big)
    \\
    \;\leq&\;
    (t_2+1)^{\max\{\kappa-\beta, 0 \}} e^{-\frac{a}{1-\kappa}}
    \msum_{t=t_1}^{t_2}  \mint_t^{t+1} x^{-\kappa} \exp\Big( \mfrac{a x^{1-\kappa}}{1-\kappa}  \Big) \,dx 
    \\
    \;=&\;
    (t_2+1)^{\max\{\kappa-\beta, 0 \}} e^{-\frac{a}{1-\kappa}}
    \mint_{t_1}^{t_2+1} x^{-\kappa} \exp\Big( \mfrac{a x^{1-\kappa}}{1-\kappa}   \Big) \,dx
    \\
    \;\leq&\;
    \mfrac{(t_2+1)^{\max\{\kappa-\beta,0\}}}{a} \, e^{-\frac{a}{1-\kappa}} \, e^{ \frac{a (t_2+1)^{1-\kappa}}{1-\kappa} }
    \\
    \;=&\;
    \mfrac{(t_2+1)^{\max\{\kappa-\beta,0\}}}{a} \,  \exp\big(  a \, \varphi_{1-\kappa}(t_2+1)  \big)
    \;.
\end{align*}
The main difference in the case $\kappa \neq 1$ and $a < 0$ is that we now use $ x \mapsto \exp( a (x+1)^{1-\kappa} / (1-\kappa))$ is decreasing to obtain, for $t_1 \leq t \leq t_2$ and $x \in [t,t+1]$,
\begin{align*}
    \exp( a (t+1)^{1-\kappa} / (1-\kappa) )
    \;\leq\;
    \exp( a x^{1-\kappa} / (1-\kappa) )
    \;.
\end{align*}
A similar argument then yields 
\begin{align*}
    &\; 
    \msum_{t=t_1}^{t_2} (t+1)^{-\beta}  \exp\big( a \, \varphi_{1-\kappa}(t) \big)
    \\
    \;\leq&\;
    (t_2+1)^{\max\{\kappa-\beta, 0 \}}
    e^{-\frac{a}{1-\kappa}}
    \msum_{t=t_1}^{t_2} (t+1)^{-\kappa} \exp\Big( \mfrac{a (t+1)^{1-\kappa}}{1-\kappa} \Big)
    \\
    \;\leq&\;
    (t_2+1)^{\max\{\kappa-\beta, 0 \}}
    e^{-\frac{a}{1-\kappa}}
    \mint_{t_1}^{t_2+1}
    x^{-\kappa}
    \exp\Big( \mfrac{a x^{1-\kappa}}{1-\kappa} \Big)
    dx
    \\
    \;=&\;
    \mfrac{(t_2+1)^{\max\{\kappa-\beta,0\}}}{a} \,  \big( 
        \exp\big(  a \, \varphi_{1-\kappa}(t_2+1)  \big)
        -
        \exp\big(  a \, \varphi_{1-\kappa}(t_1)  \big)
    \big)
    \\
    \;\leq&\;
    \mfrac{(t_2+1)^{\max\{\kappa-\beta,0\}}}{(-a)} \,
    \exp\big(  a \, \varphi_{1-\kappa}(t_1)  \big)
    \;.
\end{align*}
\end{proof}

We also need the following lemma, which is useful for controlling the accumulation of errors from the noise terms over iterations. 

\vspace{.5em}

\begin{lemma} \label{lem:error:accumulate} For any $a, b \geq 0$, $T, N \in \N$ and $\kappa, \beta \geq 0$ such that $ bt^{-\beta} - at^{-\kappa}  \leq 1$ for all $1 \leq t \leq T$, we have that
\begin{align*}
    \mprod_{t=1}^T ( 1 - b t^{-\beta} + a t^{-\kappa}   )^N
    \;\leq\;
    \exp\big( 
        - b N \, \varphi_{1-\beta}(T+1)
        + a N \, \varphi_{1-\kappa}(T+1) 
    \big)
    \;.
\end{align*}
Moreover, for any $\zeta \geq 0$, we have that
\begin{align*}
    \msum_{t=1}^T  t^{-\zeta} \, \Big( \msum_{j=1}^N ( 1 - b t^{-\beta} + a t^{-\kappa} ) &  \Big) \, \mprod_{s=t+1}^T ( 1 - b s^{-\beta} + a s^{-\kappa}  )^N 
    \\
    \;\leq&\;
    Q_1 
    + 
    \exp\Big( - \mfrac{b N}{2} \, \varphi_{1-\beta}(T+1)  + 4a N \,\varphi_{1-\kappa}(T+1) \Big) \, Q_2
    \;,
\end{align*}
where 
\begin{align*}
    Q_1 \;\coloneqq&\;
    \begin{cases}
        \mfrac{2^{2\zeta+1} (T+3)^{\max\{\beta-\zeta, 0\}}}{b}
            \exp\Big( \mfrac{b N}{2(1-\beta) (T+1)^\beta}
        \Big) 
        &
        \text{ if } \beta \neq 1 \,,\, b > 0\;,
        \\[.8em]
        2 N \varphi_{1-\zeta+bN/2}(T+1)
        \exp\Big( - \mfrac{b N}{2} \, \varphi_{1-\beta}(T+1) \Big)
        &
        \text{ if } \beta = 1 \text{ or } b = 0\;,
    \end{cases}
\end{align*}
and 
\begin{align*}
    Q_2 \;\coloneqq&\;
    \begin{cases}
        \mfrac{3^{\zeta} (1+a)^{N-1}}{2a}
    (T+2)^{\max\{\kappa - \zeta, 0 \}}
        &
        \text{ if }
         \kappa \neq 1 \text{ and } a > 0 
        \;,
        \\[.8em] 
        2 N (1+a)^{N-1}
        \,\varphi_{1-\zeta-2a N }(T+1) 
        &
        \text{ if }
        \kappa = 1 \text{ or } a = 0 
        \;.
    \end{cases}
\end{align*}
In the special case $\zeta = \beta = 1 < \kappa$\,, we have
\begin{align*}
    \msum_{t=1}^T t^{-\zeta} \, \Big( \msum_{j=1}^N  ( 1 - b t^{-\beta} + a t^{-\kappa} )^{j-1}  \Big) \,  \mprod_{s=t+1}^T ( 1 & - b s^{-\beta} + a s^{-\kappa} )^N
    \\
    &\;\leq\;
    \mfrac{4}{b}
    +
    \mfrac{3 N (1+a)^{N-1} \, e^{\frac{4aN}{\kappa-1}}  \, \log(T+1)}{(T+1)^{\frac{bN}{2}}}
    \;.
\end{align*}
\end{lemma}

\begin{proof}[Proof of \cref{lem:error:accumulate}] By assumption, $ bt^{-\beta} - at^{-\kappa}   \leq 1$ for all $1 \leq t \leq T$. Since $0 \leq 1-x \leq e^{-x}$ for all $x \leq 1$, we have that for any $1 \leq t_1 \leq t_2 \leq T$, 
\begin{align*}
    \mprod_{t=t_1}^{t_2} ( 1 - b t^{-\beta} + a t^{-\kappa}  )^N
    \;\leq\; 
    \exp\bigg( 
        - b N \msum_{t=t_1}^{t_2} t^{-\beta} 
        + a N \msum_{t=t_1}^{t_2} t^{-\kappa} 
    \bigg)
    \;.
    \tagaligneq \label{eq:prod:to:exp}
\end{align*}
Applying this to the first quantity of interest followed by noting that $a, b \geq 0$ and using \cref{lem:varphi}(ii), we obtain the first bound that 
\begin{align*}
    \mprod_{t=1}^{T} ( 1 - b t^{-\beta} + a t^{-\kappa}  )^N
    \;\leq&\;
    \exp\bigg( 
        - b N \msum_{t=1}^{T} t^{-\beta} 
        + a N \msum_{t=1}^{T} t^{-\kappa} 
    \bigg)
    \\
    \;\leq&\;
    \exp\big( 
        - b N  \, \varphi_{1-\beta}(T+1)
        + a N  \, \varphi_{1-\kappa}(T+1) 
    \big)\;.
\end{align*}
For the second bound, we define 
\begin{align*}
    t_0 \;\coloneqq\; \sup\Big\{ t \leq T \;\Big|\; 
        \mfrac{b}{2} \leq a t^{-(\kappa - \beta)}
        \Big\}\;.
\end{align*}
Then by noting that $1 - bt^{-\beta} + at^{-\kappa}  \geq 0$ for all $1 \leq t \leq T$ again,we can bound the quantity of interest as 
\begin{align*}
    & \msum_{t=1}^T t^{-\zeta} \, \Big( \msum_{j=1}^N ( 1 - b t^{-\beta} + a t^{-\kappa}  )^{j-1}  \Big) \,  \mprod_{s=t+1}^T ( 1 - b s^{-\beta} + a s^{-\kappa}  )^N
    \\
    &=\;
    \msum_{t=t_0+1}^{T} t^{-\zeta}  \, \Big( \msum_{j=1}^N ( 1 - b t^{-\beta} + a t^{-\kappa} )^{j-1}  \Big) \,  \mprod_{s=t+1}^T ( 1 - b s^{-\beta} + a s^{-\kappa}   )^N
    \\
    &\quad
    +
    \Big(  \mprod_{s=t_0+1}^T ( 1 - b s^{-\beta} + a s^{-\kappa}   )
    \Big)
    \\
    &\hspace{2em} \times
    \Big(
    \msum_{t=1}^{t_0} t^{-\zeta} \, \Big( \msum_{j=1}^N ( 1 - b t^{-\beta} + a t^{-\kappa}   )^{j-1}  \Big) \,
    \mprod_{s=t+1}^{t_0} ( 1 - b s^{-\beta} + a s^{-\kappa}   )^N
    \Big)
    \\
    &\leq\;
    \msum_{t=t_0+1}^{T} t^{-\zeta}    
    \, \Big( \msum_{j=1}^N \Big( 1 - \mfrac{b}{2} \, t^{-\beta} \Big)^{j-1} \Big)\, 
    \mprod_{s=t+1}^T \Big( 1 - \mfrac{b}{2} \, s^{-\beta} \Big)^N
    \\
    &\quad
    +
    \Big(  \mprod_{s=t_0+1}^T \Big( 1 - \mfrac{b}{2} \, s^{-\beta}  \Big)^N
    \Big)
    \Big(
    \msum_{t=1}^{t_0} t^{-\zeta}  
    \,\Big( \msum_{j=1}^N  ( 1 + a t^{-\kappa}  )^{j-1} \Big)\,
    \mprod_{s=t+1}^{t_0} ( 1 + a s^{-\kappa}  )^N
    \Big)
    \\
    &\leq\;
    N 
    \times 
    \underbrace{
    \msum_{t=t_0+1}^{T} t^{-\zeta}    
    \, 
    \mprod_{s=t+1}^T \Big( 1 - \mfrac{b}{2} \, s^{-\beta} \Big)^N
    }_{\eqqcolon S_1}
    \\
    &\quad
    +
    N (1+a)^{N-1} \times
    \underbrace{ \Big(  \mprod_{s=t_0+1}^T \Big( 1 - \mfrac{b}{2} \, s^{-\beta}  \Big)^N
    \Big)}_{\eqqcolon S_3} 
    \times
    \underbrace{
    \Big(
    \msum_{t=1}^{t_0} t^{-\zeta}  
    \,
    \mprod_{s=t+1}^{t_0} ( 1 + a s^{-\kappa}  )^N
    \Big)
    }_{\eqqcolon S_2}
    \;.
\end{align*}
In the last line, we have used that $0 \leq 1- \frac{b}{2}t^{-\beta} \leq 1$ for $t \geq t_0+1$ and $1 + a t^{-\kappa}  \leq 1 + a$. To control the three quantities, we first note that by \eqref{eq:prod:to:exp} , we have
\begin{align*}
    S_3
    \;\leq&\; 
    \exp\Big( - \mfrac{b N}{2} \msum^T_{s=1} s^{-\beta}   \Big) 
    \exp\Big( \mfrac{b N}{2} \msum^{t_0}_{s=1} s^{-\beta}   \Big) 
    \\
    \;\overset{(a)}{\leq}&\; 
    \exp\Big( - \mfrac{b N}{2} \msum^T_{s=1} s^{-\beta}   \Big) 
    \exp\Big( 
        a N \msum^{t_0}_{s=1} s^{-\kappa} 
    \Big) 
    \\
    \;\overset{(b)}{\leq}&\; 
    \exp\Big( - \mfrac{b N}{2} \, \varphi_{1-\beta}(T+1) + 2a N \,\varphi_{1-\kappa}(T+1) \Big)
    \;.
\end{align*}
In $(a)$ above, we have noted that $\frac{b}{2} \leq as^{-(\kappa-\beta)}$ for $s \leq t_0$; in $(b)$, we have used $t_0 \leq T$ and \cref{lem:varphi}(ii) with $a,b \geq 0$. In the special case $\beta=1 < \kappa$, the above yields 
\begin{align*}
    S_3 
    \;\leq&\; 
    (T+1)^{-\frac{bN}{2}}
    \exp\Big( 2 a N \, \mfrac{1-(T+1)^{-(\kappa-1)}}{\kappa-1} \Big)
    \\
    \;\leq&\;
    (T+1)^{-\frac{bN}{2}}
    e^{\frac{2aN}{\kappa-1}} 
    \;.
\end{align*}

\vspace{.5em}

We now control $S_2$. By \eqref{eq:prod:to:exp} again, we have  
\begin{align*}
    S_2
    \;\leq&\;
    \msum_{t=1}^{t_0}
    t^{-\zeta} 
    \exp\Big( 
        a N \msum_{s=t+1}^{t_0} s^{-\kappa} 
    \Big)
    \\
    \;\leq&\;
    \msum_{t=1}^T 
    t^{-\zeta} \exp\Big( 
        a N \msum_{s=t+1}^T s^{-\kappa} 
    \Big)
    \\
    \;\overset{(c)}{\leq}&\;
    \exp\big( 2 a N \varphi_{1-\kappa}(T+1) \big)
    \times 
    \Big( 
        \msum_{t=1}^T t^{-\zeta} \exp\big( 
            - 2a N \varphi_{1-\kappa}(t+1)
        \big)
    \Big)
    \\
    \;\overset{(d)}{\leq}&\;
    3^\zeta
    \exp\big( 2 a N \varphi_{1-\kappa}(T+1) \big)
    \,
        \msum_{t=1}^T(t+2)^{-\zeta} \exp\big( 
        - 2a N \varphi_{1-\kappa}(t+1)
        \big) 
    \;.
\end{align*}
In $(c)$ above, we have applied \cref{lem:varphi}(ii); in $(d)$, we have noted that $\sup_{t \in \N} (t+2)^\beta / t^\beta = 3^\beta$. If $\kappa \neq 1$ and $a > 0$, we can apply \cref{lem:varphi}(iv) to get that 
\begin{align*}
    S_2
    \;\leq&\; 
    \mfrac{3^{\zeta}}{2a N}
    (T+2)^{\max\{\kappa - \zeta, 0 \}}
    \,
    \exp\big( 2 a N \varphi_{1-\kappa}(T+1) \big)
    \\
    \;=&\;
    \mfrac{Q_2}{N (1+a+c)^{N-1}}\,
    \exp\big( 2 a N \varphi_{1-\kappa}(T+1) \big)
    \;.
\end{align*}
If $\kappa=1$ or $a=0$, the bound from $(c)$ above reads
\begin{align*}
    S_2
    \;\leq&\;
    \exp\big( 2 a N \varphi_{1-\kappa}(T+1) \big)
    \, \msum_{t=1}^T t^{-\zeta} (t+1)^{-2aN}
    \\
    \;\leq&\;
    \exp\big( 2 a N \varphi_{1-\kappa}(T+1) \big)
    \, \msum_{t=1}^T t^{-\zeta-2aN}
    \\
    \;\leq&\;
    2
    \,\varphi_{1-\zeta-2aN}(T+1)
    \, \exp\big( 2 a N \varphi_{1-\kappa}(T+1) \big)
    \;=\; 
    \mfrac{Q_2}{N (1+a)^{N-1}} 
    \, \exp\big( 2 a N \varphi_{1-\kappa}(T+1) \big)
    \;,
\end{align*}
where we have used \cref{lem:varphi}(ii) in the last line. Now consider the special case with $\zeta = 1 < \kappa $, the bound from $(d)$ becomes 
\begin{align*}
    S_2 
    \;\leq&\;
    3
    \exp\big( 2 a N \varphi_{1-\kappa}(T+1) \big)
    \Big( 
        \msum_{t=1}^T(t+2)^{-1} \exp\big( 
        - 2a N \varphi_{1-\kappa}(t+1)
        \big) 
    \Big)
    \\
    \;\leq&\;
    3
    \exp\Big( 2 a N \, \mfrac{1-(T+1)^{-(\kappa-1)}}{\kappa-1} \Big)
        \msum_{t=1}^T(t+2)^{-1} 
    \\
    \;\leq&\;
    3
    e^{\frac{2aN}{\kappa-1} } 
    \log(T+1)
    \;.
\end{align*}

\vspace{.5em}

We are left with controlling $S_1$, which follows from a similar strategy as controlling $S_2$:
\begin{align*}
    S_1 
    \;\overset{\eqref{eq:prod:to:exp}}{\leq}&\;
    \msum_{t=t_0+1}^{T} 
    t^{-\zeta}  
    \exp \Big( - \mfrac{b N }{2} \msum_{s=t+1}^T s^{-\beta} \Big)
    \\
    \;\leq&\;
    \msum_{t=1}^{T} t^{-\zeta}
    \exp\Big( - \mfrac{b N}{2} \msum_{s=t+1}^T s^{-\beta} \Big)
    \\
    \;\overset{(a)}{\leq}&\;
    \exp\Big( - \mfrac{b N }{2} \, \varphi_{1-\beta}(T+1) \Big)
    \msum_{t=1}^{T} t^{-\zeta}
    \exp\Big( \mfrac{b N}{2} \, \varphi_{1-\beta}(t+1) \Big)
    \tagaligneq \label{eq:S1:intermediate}
    \\
    \;\leq&\;
    4^\zeta
    \exp\Big( - \mfrac{b N}{2} \, \varphi_{1-\beta}(T+1) \Big)
    \msum_{t=1}^{T} (t+3)^{-\zeta}
    \exp\Big( \mfrac{b N }{2} \, \varphi_{1-\beta}(t+1) \Big)
    \;.
\end{align*}
In $(a)$ above, we used \cref{lem:varphi}(ii). For $\beta \neq 1$ and $b \neq 0$, we can apply \cref{lem:varphi}(iv) with $\frac{b}{2} > 0$ to obtain 
\begin{align*}
    S_1 
    \;\leq&\;
    4^\zeta
    \exp\Big( - \mfrac{b N }{2} \, \varphi_{1-\beta}(T+1) \Big)
    \mfrac{(T+3)^{\max\{\beta-\zeta, 0\}}}{b N /2} 
    \exp\Big( \, \mfrac{b N}{2} \varphi_{1-\beta}(T+3)  \Big)
    \\
    \;=&\;
    \mfrac{2^{2\zeta+1} (T+3)^{\max\{\beta-\zeta, 0\}}}{b N}
    \exp\Big( \mfrac{b N}{2(1-\beta)} \Big( (T+3)^{1-\beta} - (T+1)^{1-\beta} \Big) \Big) 
    \\
    \;\overset{(b)}{\leq}&\;
    \mfrac{2^{2\zeta+1} (T+3)^{\max\{\beta-\zeta, 0\}}}{b N}
    \exp\Big( \mfrac{b N}{2(1-\beta) (T+1)^\beta}
    \Big) 
    \;=\; \mfrac{Q_1}{N}
    \;.
\end{align*}
In $(b)$, we have used \cref{lem:varphi}(iii) with $1-\beta \leq 1$. Meanwhile, if $\beta=1$ or $b=0$, we have 
\begin{align*}
    S_1
    \;\leq&\;
    \exp\Big( - \mfrac{b N }{2} \, \varphi_{1-\beta}(T+1) \Big)
    \msum_{t=1}^{T} t^{-\zeta} (t+1)^{b N/2}
    \\
    \;\leq&\;
    \exp\Big( - \mfrac{b}{2} \, \varphi_{1-\beta}(T+1) \Big)
    \msum_{t=1}^{T} t^{-\zeta+bN/2}
    \\
    \;\leq&\;
    2 \varphi_{1-\zeta+ bN/2}(T+1)
    \exp\Big( - \mfrac{bN}{2} \, \varphi_{1-\beta}(T+1) \Big)
    \;=\; \mfrac{Q_1}{N}
    \;.
\end{align*}
For the special case with $\zeta =  \beta = 1$, the bound in \eqref{eq:S1:intermediate} becomes
\begin{align*}
    S_1 
    \;\leq&\; 
    \exp\Big( - \mfrac{b N }{2} \, \varphi_0(T+1) \Big)
    \msum_{t=1}^{T} t^{-1}
    \exp\Big( \mfrac{b N}{2} \, \varphi_0(t+1) \Big)
    \\
    \;=&\;
    (T+1)^{-\frac{bN}{2}}
    \msum_{t=1}^{T} t^{-1} (t+1)^{\frac{bN}{2}}
    \\
    \;\leq&\;
    (T+1)^{-\frac{bN}{2}} \msum_{t=1}^{T} t^{-(1-\frac{bN}{2})}
    \\
    \;\overset{(c)}{\leq}&\;
    2 
    (T+1)^{-\frac{bN}{2}} 
    \,
    \varphi_{\frac{bN}{2}}(T+1) 
    \;=\;
    2 (T+1)^{-\frac{bN}{2}} 
    \,
    \mfrac{(T+1)^{bN/2} - 1}{bN/2}
    \;\leq\; 
    \mfrac{4}{bN}\;.
\end{align*}
In $(c)$, we have used \cref{lem:varphi}(ii) for both the case $1-\frac{bN}{2} \leq 0$ and $1-\frac{bN}{2} \geq 0$.

\vspace{1em}

Combining the bounds for the general cases, we obtain the first desired inequality that 
\begin{align*}
    & \msum_{t=1}^T t^{-\zeta} \, \Big( \msum_{j=1}^N ( 1 - b t^{-\beta} + a t^{-\kappa}   )^{j-1}  \Big) \,  \mprod_{s=t+1}^T ( 1 - b s^{-\beta} + a s^{-\kappa}   )^N
    \\
    &\;\leq\;
    Q_1 
    + 
    \exp\Big( - \mfrac{b}{2} \, \varphi_{1-\beta}(T+1) + u \varphi_{1-\xi}(T+3)  + 4a \,\varphi_{1-\kappa}(T+1)  \Big) \, Q_2\;.
\end{align*}
For the special case $\zeta=\beta=1< \kappa, \gamma$, combining the earlier bounds gives 
\begin{align*}
    \msum_{t=1}^T t^{-\zeta} \, \Big( \msum_{j=1}^N  ( 1 - b t^{-\beta} + a t^{-\kappa} )^{j-1}  \Big) \,  \mprod_{s=t+1}^T ( 1 & - b s^{-\beta} + a s^{-\kappa} )^N
    \\
    &\;\leq\;
    \mfrac{4}{b}
    +
    \mfrac{3 N (1+a)^{N-1} \, e^{\frac{4aN}{\kappa-1}}  \, \log(T+1)}{(T+1)^{\frac{bN}{2}}}
    \;.
\end{align*}
\end{proof}

\subsection{Contraction and integrability results} 

The next result is a standard result in convex analysis,
needed to handle projections performed in Algorithms \ref{alg:online_cd} and \ref{alg:offline_cd}.
\begin{lemma}\label{lemma:contraction}
Let $ \Psi  $ a be convex subset of  $ \mathbb{R}^p $. 
Let $ {\psi}^{\star} \in \Psi $. Then, for all $ \psi \in \mathbb{R}^p $, we have:
\begin{equation*} 
\begin{aligned}
  \left \| \mathrm{Proj}_{\Psi}(\psi) - {\psi}^{\star} \right \| &\leq \left \| \psi - {\psi}^{\star} \right \|
\end{aligned}
\end{equation*}
\end{lemma}
\begin{proof}
  We have:
\begin{equation*} \label{lemma:projections-are-contractive}
\begin{aligned}
  \| \psi -  {\psi}^{\star} \|^2 &= \| \psi - \mathrm{Proj}_{\Psi}(\psi) + \mathrm{Proj}_{\Psi}(\psi)-  {\psi}^{\star} \|^2 \\
			  &= \| \psi - \mathrm{Proj}_{\Psi}(\psi) \|^2 + 
2 \left \langle  \psi - \mathrm{Proj}_{\Psi}(\psi), \mathrm{Proj}_{\Psi}(\psi) -  {\psi}^{\star} \right \rangle + \| \mathrm{Proj}_{\Psi}(\psi) -  {\psi}^{\star} \|^2 \\
\end{aligned}
\end{equation*}
Since by \cite[Proposition 1.1.9]{bertsekas2009convex}, we have:
\begin{equation*} 
\begin{aligned}
\left \langle \psi - \mathrm{Proj}_{\Psi}(\psi), \psi' - \mathrm{Proj}_{\Psi}(\psi) \right \rangle \leq  0
\end{aligned}
\end{equation*}
for all $ \psi' \in \Psi $, we can use this inequality at $ \psi' = {\psi}^{\star} \in \Psi $ to obtain:
\begin{equation*} 
\begin{aligned}
  \left \| \psi - {\psi}^{\star} \right \|^2 \geq  \left \| \psi - \mathrm{Proj}_{\Psi}(\psi) \right \|^2
\end{aligned}
\end{equation*}
and the result follows by taking the square root.
\end{proof}

We now state two lemmas that guarantee an amount of integrability sufficient to our analysis.
\begin{lemma}\label{lemma:mean-existence}
	Let $ p, q \in \mathcal  P(\mathcal  X) $ such that $ \frac{ \mathrm{d}p }{ \mathrm{d}q } $ exists, and such that
	$ \chi^2(p; q)< +\infty $. Assume that $ f \in L^2(q)$ Then $ |\int_{   }^{  }f \mathrm{d}p| < +\infty$.
\end{lemma}
\begin{proof}
	By assumption, $ f \in L^2(q)$. Moreover, $ \chi^2(p, q)< +\infty $, and thus we have
	$ \frac{\mathrm{d}p}{\mathrm{d}q} - 1 \in L^2(q) $. Thus, the inner product is finite, and 
	\begin{equation*} 
	\begin{aligned}
		\Big|
            \mint_{  }^{  } f \Big(\mfrac{\mathrm{d}p}{\mathrm{d}q} - 1 \Big) \mathrm{d}q
        \Big| 
        &= 
        \Big| \mint_{  }^{  } f \mathrm{d} p - \int_{  }^{  } f \mathrm{d} q \Big|
        \coloneqq  M < +\infty  \\
		\implies M - \left \lvert \mint_{   }^{  }f \mathrm{d}q  \right \rvert  &< \int_{  }^{  } f \mathrm{d} p < M + \left \lvert  \mint_{   }^{  }f \mathrm{d} q \right \rvert
	\end{aligned}
	\end{equation*}
\end{proof}

\begin{lemma}
	For all $ \psi \in \Psi $, for all $ m \geq  1 $, and for all $ k \geq  1 $, we have:
	\begin{equation*} 
	\begin{aligned}
	\mathbb{ E } \left \lbrack \left \| \mathrm{P}^{m}_{\psi} \phi(X_1) \right \|^k \right \rbrack < +\infty\,.
	\end{aligned}
	\end{equation*}
\end{lemma}

\begin{proof}
	By analycity of the log partition function $ \psi  \longmapsto \log Z(\psi) $, we have $ \int_{  }^{  } \left \| \phi \right \|^k \mathrm d p_{ \psi}<  +\infty $ for all $ \psi \in \Psi $,
	and thus, the function $ x  \longmapsto \left \| \phi \right \|^{k}(x) \in L^2(p_{\psi})$ for all $ \psi $.
	Consequently, $ P_\psi^{m} \|\phi\|^k \in L^2( p_{ \psi}) $.
	We can apply Lemma \ref{lemma:mean-existence} to $ \mathrm{P}_\psi^{m} \left \| \phi \right \|^{k}$ to obtain $ \mathbb{ E }\left \lbrack  \mathrm{P}^{m} \|\phi(X_1)\|^k \right \rbrack < +\infty  $
	for all $ k \geq  1 $ and for all $ m \geq  0 $.
	As a by-product, we obtain  $ \mathrm{P}_{\psi}^{m} \left \| \phi \right \|^{k} \in L^2( p_{ {\psi}^{\star}}) $, and thus so  $ \big \|  \mathrm{P}^{m}_{\psi} \phi \big \|^k $.
\end{proof}
The following lemma is used multiple time in our analysis.
\begin{lemma}\label{lemma:approx-contraction}
	Assume \ref{asst:SGD-A5}. Let $q$ be a positive integer. Let $f \coloneqq (f_1, \dots, f_q)$ such that $ f_k \in \{ \phi_i \}_{i=1}^{p} \cup \{ \phi_i \phi_j \}_{i,j=1}^{p}$ for $k \in [q]$. Then, for all $\psi \in \Psi$, we have
	\begin{equation*} 
	\begin{aligned}
		\left \| \mathbb{ E } \left \lbrack \mathrm{P}^{m}_{ \psi}(f - \mathbb{ E }\big \lbrack f(X^{\psi}) \big \rbrack)(X_1) \right \rbrack \right \|  \leq \alpha^{m} C_{ \chi} \left (\mathbb{ E } \left \lbrack \big \| \big ( f - \mathbb{ E } \big \lbrack f(X^{\psi}) \big \rbrack \big )(X^{\psi}) \big \|^2 \right \rbrack  \right )^{1 / 2} \left \| \psi - {\psi}^{\star} \right \|
	\end{aligned}
	\end{equation*}
\end{lemma}
\begin{proof}
	Let us note first that 
	\begin{equation*} 
	\begin{aligned}
		\MoveEqLeft \| \mathbb{ E } \mathrm{P}^{m}_{ \psi} \left \lbrack \big (f - \mathbb{ E }\big \lbrack f(X^{\psi}) \big \rbrack \big )(X_1) \right \rbrack \|^2  \\
	&\overset{(a)}{=}  
    \msum\limits_{ i=1 }^{ q } \left (  \mint \! \mathrm{P}^{m}_{ \psi}\big (f_i - \mathbb{ E }\big \lbrack f_i(X^{\psi}) \big \rbrack \big )(x) \left (p_{ {\psi}^{\star}}(x) - p_{ \psi}(x) \right ) \text{d}x \right )^2\\ 
	&=  
    \msum\limits_{ i=1 }^{ q } \left ( \mint \! \mathrm{P}^{m}_{ \psi}\big (f_i - \mathbb{ E }\big \lbrack f_i(X^{\psi}) \big \rbrack \big )(x) \left (\mfrac{  \mathrm{d} p_{ {\psi}^{\star}} }{ \mathrm{d} p_{ \psi} }(x) - 1\right) p_{ \psi}(x) \text{d}x \right )^2\\
	&\overset{(b)}{\leq}
    \left (\mint \left (\mfrac{\mathrm{d}p_\psi^\star}{\mathrm{d}p_\psi}(x) - 1\right )^2 p_\psi(x) \mathrm{d} x \right )
    \msum\limits_{ i=1 }^{ q } \mint \! \mathrm{P}^{m}_{ \psi}\big (f_i - \mathbb{ E }\big \lbrack f_i(X^{\psi}) \big \rbrack \big )(x)^2 p_{ \psi}(x) \text{d}x   \\
	&\leq  
    \chi^2(p_{ \psi}, p_{ {\psi}^{\star}})
    \msum\limits_{ i=1 }^{ q } \big \| \mathrm{P}^{m}_{ \psi}(f_i - \mathbb{ E }\big \lbrack f_i(X^{\psi}) \big \rbrack) \big   \|_{L^2(p_{\psi})} \\
	&\overset{(c)}{\leq}
    \alpha^{2m} \chi^2(p_{ \psi}, p_{ {\psi}^{\star}}) 
    \msum\limits_{ i=1 }^{ q } \big \| f_i - \mathbb{ E }\big \lbrack f_i(X^{\psi}) \big \rbrack \big   \|_{L^2(p_{ \psi})(\mathbb R^d)} \\
	&\overset{(d)}{\leq}  \alpha^{2m} C_{ \chi}^2 \left \| \psi - {\psi}^{\star} \right \|^2  \mathbb{ E } \big \lbrack \big \| \big (f - \mathbb{ E } \big \lbrack f(X^{\psi}) \big \rbrack \big )(X^{\psi}) \big \|^2 \big \rbrack
    \;.
	\end{aligned}
	\end{equation*}
Here, we used the fact that $P^m_\psi$ admits $p_\psi$ as an invariant measure in $(a)$ \citep[Eq. (1.2.2)]{bakry2014analysis}, the Cauchy-Schwarz inequality in $(b)$
\end{proof}

\subsection{Miscellaneous}

\begin{lemma} \label{lem:replace:gradient:inequality} Let $f: \Psi \rightarrow \R^p$ be a differentiable function in the interior of $\Psi \subseteq \R^p$. For $\psi \in \Psi$, define $\sigma_{\rm min}(\psi) \coloneqq \minf_{ \theta \in \Psi, \|\theta\|=1} \theta^\top \nabla f(\psi) \theta$ and  $\sigma_{\rm max}(\psi) \coloneqq \msup_{ \theta \in \Psi, \|\theta\|=1} \theta^\top \nabla f(\psi) \theta$ with respect to the Jacobian matrix $\nabla f(\psi)$. Then for any $\psi_1, \psi_2 \in \Psi$, we have that 
\begin{align*}
    \minf_{\psi \in \Psi} \sigma_{\rm min}(\psi)
    \;\leq\; 
    (\psi_1 - \psi_2)^\top \big( f(\psi_1) - f(\psi_2) \big) 
    \;\leq\; 
    \msup_{\psi \in \Psi} \sigma_{\rm max}(\psi)
\end{align*}
\end{lemma}

\begin{proof}[Proof of \cref{lem:replace:gradient:inequality}] By the mean value theorem, there exists some $a \in (0,1)$ such that 
\begin{align*}
    (\psi_1 - \psi_2)^\top \big( f(\psi_1) - f(\psi_2) \big) 
    \;=&\;
    (\psi_1 - \psi_2)^\top \big( f( 1 \times \psi_1 + 0 \times \psi_2) - f(0 \times \psi_1 + 0 \times \psi_2) \big) 
    \\
    \;=&\;
    (\psi_1 - \psi_2)^\top \nabla f( a\psi_1 + (1-a)\psi_2  ) (\psi_1 - \psi_2)
    \\
    \;=&\;
    \| \psi_1 - \psi_2 \|^2 \, 
    \mfrac{(\psi_1 - \psi_2)^\top}{\| \psi_1 - \psi_2 \|} \,\nabla f( a\psi_1 + (1-a)\psi_2  ) \,
    \mfrac{\psi_1 - \psi_2}{\| \psi_1 - \psi_2 \|}
    \;.
\end{align*}
Plugging in the definition of $\sigma_{\rm max}$ gives the desired upper bound and similarly $\sigma_{\rm min}$ implies the lower bound. 
\end{proof}

\allowdisplaybreaks

\newcommand{\cdPpsi}{\mathrm{P}_{\psi}^{m}}
\newcommand{\cdP}[1]{\mathrm{P}_{#1}^{m}}
\newcommand{\Xpsi}{\cramped{X^{\psi}}}
\newcommand{\smallpow}[1]{\vphantom{a}^{#1}}
\newcommand{\smallsub}[1]{\vphantom{a}_{#1}}

\providecommand\given{}
\DeclarePairedDelimiter\cdnorm{\lVert}{\rVert}
\DeclarePairedDelimiter\cdabs{\lvert}{\rvert}
\DeclarePairedDelimiter\cdinner{\langle}{\rangle}
\DeclarePairedDelimiter\cdpar{(}{)}
\DeclarePairedDelimiterXPP\cdPr[1]{\mathbb{P}}{\lbrack}{\rbrack}{}{#1}
\DeclarePairedDelimiterXPP\cdE[1]{\mathbb{E}}{\lbrack}{\rbrack}{}{#1}
\DeclarePairedDelimiterXPP\cdCov[1]{\mathrm{Cov}}{\lbrack}{\rbrack}{}{#1}

\DeclarePairedDelimiterXPP\cdCondCov[1]{\mathrm{Cov}}{\lbrack}{\rbrack}{}{
\renewcommand\given{\nonscript\:\delimsize\vert\nonscript\:\mathopen{}}
#1}
\DeclarePairedDelimiterXPP\cdCondE[1]{\mathbb{E}}{\lbrack}{\rbrack}{}{
\renewcommand\given{\nonscript\:\delimsize\vert\nonscript\:\mathopen{}}
#1}

\section{Proofs for Online CD}\label{sec:proofs-online-CD}
\subsection{Auxiliary Lemmas for Online CD}

We recall the following notations:
$ \sigma_{\psi} \coloneqq \cdE[\big]{\cdnorm{\big (\phi - \cdE[\big]{\phi(\Xpsi)} \big )(\Xpsi)}^2} $,  
as well as
$ \sigma_{\star} \coloneqq \sigma_{{\psi}^{\star}}$ and $ \sigma \coloneqq \sup_{\psi \in \Psi} \sigma_{\psi} $.

We now provide two intermediary lemmas necessary to analyze the impact of variance in the CD gradient.
The strategy is similar in both of them: we change the integration from $ p_{{\psi}^{\star}} $ to $ p_{\psi} $
to obtain contraction, at the cost of an additional term scaling with $ C_{\chi}\left \| \psi - {\psi}^{\star} \right \| $.
We obtain exact constants that we choose to describe in terms of the smoothness parameters of the problem, e.g. 
the $ k^{th} $ derivatives of the log partition function $ \log Z $, which, for $ k \geq  2 $, equals
the $ k^{th} $ derivative of the negative cross-entropy model w.r.t $ p_{{\psi}^{\star}} $.

\paragraph{Second Moment convergence} 
The following lemmas guarantee the second moment of a sample from $ (K^{m}_{\psi} )_{\#} p_{{\psi}^{\star}} $
approaches the second moment of a sample from the target distribution $ p_{\psi} $.

\begin{lemma} \label{lemma:mean-convergence-second-moment} Under \ref{asst:SGD-A1}, \ref{asst:SGD-A4} and \ref{asst:SGD-A5},
  for all $ \psi \in \Psi $, we have:
\begin{equation*}
\begin{aligned}
  \cdabs*{\cdE{\cdPpsi \|\phi(X_1)\|^2} - \cdE{\|\phi(\Xpsi)\|^2}}
  \; & \; \leq  \alpha^{m} C_{\chi}  \cdnorm*{ \psi - {\psi}^{\star}} \cdnorm{\log Z}_{1, \infty} \,,
\end{aligned}
\end{equation*}
where 
\begin{multline}
 \cdnorm*{\log Z}_{1, \infty}
 \coloneqq \sup_{\psi \in \Psi}\sum_{i=1}^{p} 
  \left (
    4 \partial^{1}_{i} \log Z(\psi)^2 \partial^{2}_{i} \log Z(\psi) 
    + 2 \partial^{2}_{i} \log Z(\psi)^2 \right . \\ 
  \left . 
    \,+\, 4 \partial^{1}_{i} \log Z(\psi) \partial^{3}_{i} \log Z(\psi) 
    + \partial^{4}_{i} \log Z(\psi) 
  \right )^{1 / 2}
  < +\infty \,.
\end{multline}
\end{lemma}
\begin{proof}
Applying Lemma \ref{lemma:approx-contraction} to each $ f_i  \coloneqq  \phi_i^2 $, we have
\begin{align*}
  \cdabs*{\cdE[\big]{\cdPpsi \phi_i(X_1)^2} - \cdE[\big]{\phi_i(\Xpsi)^2}}
  &\;=\;
    \alpha^{m} C_{\chi}  \cdnorm*{\psi - {\psi}^{\star}} 
    \cdpar[\Big]{\cdE[\big]{\cdpar*{\phi_i(\Xpsi)^2 - \cdE[\big]{\phi_i(\Xpsi)^2}}^2}}^{1/2} \\
  & \;=\; 
    \alpha^{m} C_{\chi} \cdnorm*{\psi - {\psi}^{\star}} 
    \cdpar[\Big]{\cdE[\big]{\phi_i(\Xpsi)^4} - \cdpar[\big]{\cdE[\big]{\phi_i^2(\Xpsi)}}^2}^{1/2} \\
\end{align*}
We map the two moments to derivatives of $ \log Z(\psi) $, since the $ k^{th} $ derivative of $ \log Z(\psi) $ is the $ k^{th} $ cumulant.
It can be shown, using the multivariate moment to cumulant mapping, that
\begin{align*}
  \MoveEqLeft \cdE*{\phi_i(\Xpsi)^4}\\ 
  & \;=\; \mfrac{\partial \log Z}{\partial \psi_i}^4 
    + 6 \mfrac{\partial \log Z}{\partial \psi_i}^2 \mfrac{\partial^2 \log Z}{\partial \psi_i^2} 
    + 3 \cdpar*{\mfrac{\partial^2 \log Z}{\partial \psi_i^2}}^2 
    + 4 \mfrac{\partial \log Z}{\partial \psi_i} \mfrac{\partial^3 \log Z}{\partial \psi_i^3} 
    + \mfrac{\partial^4 \log Z}{\partial \psi_i^4} \\
  & \;=\; \partial^{1}_{i} \log Z(\psi)^4 
    +6 \partial^{1}_{i} \log Z(\psi)^2 \partial^{2}_{i} \log Z(\psi) 
    + 3 \partial^{2}_{i} \log Z(\psi)^2 
    + 4 \partial^{1}_{i} \log Z(\psi) \partial^{3}_{i} \log Z(\psi) \\
  & \quad + \partial^{4}_{i} \log Z(\psi)
\end{align*}
where $ \partial^{k}_{i} \log Z(\psi) $ denotes the $ k^{th} $ derivative of $ \log Z $ with respect to $ \psi_i $.
On the other hand, 
\begin{align*} 
  \cdE*{\phi_i(\Xpsi)^2}&= \partial^{1}_{i} \log Z(\psi)^2 + \partial^{2}_{i} \log Z(\psi) \\
  \implies \cdpar*{\cdE*{\phi_i(\Xpsi)^{2}}}^2 &= 
    \partial^{1}_{i} \log Z(\psi)^4 
    + 2 \partial^{1}_{i} \log Z(\psi)^2 \partial^{2}_{i} \log Z(\psi) + \partial^{2}_{i} \log Z(\psi)^2
\end{align*}

\begin{align*} 
  \MoveEqLeft \cdE{\phi_i(\Xpsi)^{4}} -  \cdpar*{\mathbb{E} \phi_i(\Xpsi)^2}^2 \\ 
  &=  
  4 \partial^{1}_{i} \log Z(\psi)^2 \partial^{2}_{i} \log Z(\psi) 
    + 2 \partial^{2}_{i} \log Z(\psi)^2 
    + 4 \partial^{1}_{i} \log Z(\psi) \partial^{3}_{i} \log Z(\psi) \\
  &\qquad
    + \partial^{4}_{i} \log Z(\psi) 
\end{align*}
The result follows by summing over $ i $, since:

\begin{align*} 
  \cdabs*{\cdE*{\cdPpsi  \cdnorm*{\phi(X_1)}^2} - \cdE*{\|\phi(\Xpsi)\|^2}}  
  \;\leq\; \msum\limits_{i=1}^{d} \,
  \cdabs*{\cdE*{\cdPpsi \phi_i(X_1)^2} - \cdE*{\phi_i(\Xpsi)^2}}
\end{align*}

Note that $ \cdnorm{\log Z}_{1, \infty} $ is finite since $ \Psi $ is compact and $ \log Z $ is analytic.
\end{proof}

\paragraph{Squared First Moment convergence} 
The next lemma provides convergence (in squared absolute value) of the first moment of the 
$ m $-iterated Markov kernel $ k^{m}_{\psi} $. 
\begin{lemma}\label{lemma:quadratic-convergence-first-moment}
Under \ref{asst:SGD-A1}, \ref{asst:SGD-A4} and \ref{asst:SGD-A5},
for all $ \psi \in \Psi $, we have

\begin{equation*} 
  \cdabs*{
    \cdE*{\cdnorm*{\cdPpsi \phi(X_1)}^2} 
    - {\cdE[\big]{\phi(\Xpsi)}}^2  
  } \leq
  \alpha^{2m} \sigma_{\psi}^2  
    + C_{\chi}\alpha^{m / 2}  \cdnorm*{\log Z}_{2, \infty}  \cdnorm*{\psi - {\psi}^{\star}}  
\end{equation*}
where
\begin{equation*} 
   \cdnorm*{\log Z}_{2, \infty} \;\coloneqq\; \sup_{\psi \in \Psi} \msum\limits_{i=1}^{p}  
    \cdpar*{
	F(\psi) \partial^2_i \log Z(\psi)}^{1 / 4}  
	+ 2  \cdabs*{\partial^1_i \log Z(\psi)\partial^2_{i} \log Z(\psi)}^{1  / 2} 
\end{equation*}
and 
\begin{equation*} 
F(\psi) \; \coloneqq \;  15 \partial^{2}_{i} \log Z(\psi)^3 
	    + 10 \partial^{3}_{i} \log Z(\psi)^2 
	    + 15 \partial^{2}_{i} \log Z(\psi) \partial^{4}_{i} \log Z(\psi) 
	    + \partial^{6}_{i} \log Z(\psi)\,.
\end{equation*}
\end{lemma}
\begin{proof}
  We have:
\begin{align*}
  \cdE*{(\cdPpsi \phi_i(X_1))^2} & = \cdE*{
    \cdpar*{
      \cdPpsi \phi_i(X_1) 
      - \cdE*{\phi_i(\Xpsi)} 
      + \cdE*{\phi_i(\Xpsi)}  
    }^2  
  } \\
				    & = \cdE*{
    \cdpar*{
      \cdPpsi \phi_i(X_1)
      - \cdE*{\phi_i(\Xpsi)}
    }^2
  } \\ 
				    &  \quad 
    + 2 \times \cdE*{\phi_i(\Xpsi)} \, \cdE*{\cdPpsi  \cdpar*{\phi_i - \cdE*{\phi_i(\Xpsi)}}(X_1)}
    +  \cdpar*{\cdE*{\phi_i(\Xpsi)}}^2 \\
\end{align*}
Implying 
\begin{align*}
  \cdabs[\big]{\cdE[\big]{\cdpar[\big]{\cdPpsi \phi_i(X_1)} ^2} -  \cdpar[\big]{\cdE[\big]{\phi_i(\Xpsi)}}^2}
    & \leq  \underbrace{
      \cdabs*{
	\cdE*{\cdpar*{\cdPpsi \cdpar*{\phi_i - \cdE*{\phi_i(\Xpsi)}}(X_1)}^ 2}
      }
    }\limits_{\Delta_1}  \\ 
    & \quad \quad \quad + 2  \underbrace{
       \cdabs*{
	  \cdE*{\phi_i(\Xpsi)} \, 
	  \cdE*{\cdPpsi  \cdpar*{\phi_i - \cdE*{\phi_i(\Xpsi)}}(X_1)}  
       }
    }\limits_{\Delta_2}  \\
\end{align*}
where
\begin{align*}
  \Delta_1 &= \underbrace{
    \cdE*{\cdpar*{\cdPpsi  (\phi_i- \cdE*{\phi_i(\Xpsi)}  )(\Xpsi)}^2}
  }\limits_{\Delta_{1, 1}}  \\ 
	   & \quad +  \underbrace{
    \cdE*{
      \cdpar*{\cdpar*{\cdPpsi \phi_i - \cdE*{\phi_i(\Xpsi)}}^2(\Xpsi) 
      \cdpar*{\mfrac{\mathrm d p_{\psi^\star}}{\mathrm d p_\psi}(\Xpsi) - 1}}} 
    }\limits_{\Delta_{1, 2}}  \\
	   & \hspace{-1em}\overset{(a)}{\leq} \alpha^{2m} 
    \cdE*{\cdpar*{\phi_i - \cdE*{\phi_i(\Xpsi)}}(\Xpsi)^2} 
    +  C_{\chi}  \cdnorm*{\psi - {\psi}^{\star}}  \cdpar*{\cdE*{(\cdPpsi(\phi_i - \cdE*{\phi_i(\Xpsi)}))(\Xpsi)^4}}^{1 / 2} \\
	   & \hspace{-1em}\overset{(b)}\leq \alpha^{2m}
    \cdE*{\cdpar*{\phi_i - \cdE*{\phi_i(\Xpsi)}}(\Xpsi)^2}  
    +C_{\chi}  \cdnorm*{\psi - {\psi}^{\star}} \\ 
	   & \,\, 
    \times    
      \cdpar*{\cdE*{(\cdPpsi(\phi_i - \cdE*{\phi_i(\Xpsi)}))(\Xpsi)^6}}^{1 / 4} 
      \cdpar*{\cdE*{\cdpar*{\cdPpsi(\phi_i - \cdE*{\phi_i(\Xpsi)})}(\Xpsi)^2}}^{1 / 4} \\
	   & \hspace{-1em}\overset{(c)}\leq
      \alpha^{2m} 
      \cdE*{\cdpar*{\phi_i - \cdE*{\phi_i(\Xpsi)}}(\Xpsi)^2} \\ 
	   & \, +  
      \alpha^{\frac{m}{2}} C_{\chi}  \cdnorm*{\psi - {\psi}^{\star}}  
      \cdpar*{\cdE*{(\phi_i - \cdE*{\phi_i(\Xpsi)})(\Xpsi)^6}}^{1 / 4} 
      \cdpar*{\cdE*{(\phi_i - \cdE*{\phi_i(\Xpsi)})(\Xpsi)^2}}^{1 / 4} \\
\end{align*}
In $(a)$, we used the restricted spectral gap Assumption \ref{asst:SGD-A5} for $\Delta_{1, 1}$, and the Cauchy-Schwarz inequality combined with Assumption \ref{asst:SGD-A4} for $\Delta_{1, 2}$. In $(b)$, we used Cauchy-Schwarz once again, and in $(c)$ we used the fact that $ \cdPpsi $ is a contraction in $ L^6(p_{\psi}) $ and another invocation of the spectral gap assumption \ref{asst:SGD-A5}.
For $ \Delta_2 $, applying Lemma \ref{lemma:approx-contraction} to $ f \coloneqq \phi $, we have
\begin{align*}
  \Delta_2 &\leq 
    \cdE*{\phi_i(\Xpsi)} \alpha^{m}  \, C_{\chi}  \cdnorm*{\psi - {\psi}^{\star}}  
    \cdpar*{\cdE*{\cdpar*{\phi_i - \cdE*{\phi_i(\Xpsi)}} (\Xpsi)^2}}^{1 / 2} \\
	   &\leq  
    \alpha^{m} C_{\chi}  \cdnorm*{\psi - {\psi}^{\star}}
    \partial^1_i \log Z(\psi)\partial^2_{i} \log Z(\psi)^{1  / 2}  \\
\end{align*}
Putting everything together, we have:
\begin{align*}
  \cdabs*{\cdE*{\cdPpsi \phi_i^2(X_1)}  -  \cdE*{\phi_i^2(\Xpsi)}} \; \leq & \; 
  \alpha^{2m} \cdE*{\cdpar*{\phi_i - \cdE*{\phi_i(\Xpsi)}}(\Xpsi)^2} 
  + C_{\chi} \cdnorm*{\psi - {\psi}^{\star}} \times \\
      & \hspace{-10em}\Big (
	  \alpha^{\frac{m}{2}}  
	  \cdpar*{\cdE*{(\phi_i - \cdE*{\phi_i(\Xpsi)})(\Xpsi)^6}}^{1 / 4} 
	  \cdpar*{\cdE*{(\phi_i - \cdE*{\phi_i(\Xpsi)})(\Xpsi)^2}}^{1 / 4}
\\
      & \hspace{-3em} + 2 \alpha^{m} \partial^1_i \log Z(\psi)\partial^2_{i} \log Z(\psi)^{1  / 2} \Big ) 
\\
\;\leq&\; 
  \alpha^{2m} \cdE*{\cdpar*{\phi_i - \cdE*{\phi_i(\Xpsi)}}(\Xpsi)^2} 
  + C_{\chi}\alpha^{m / 2} \cdnorm*{\psi - {\psi}^{\star}}  \\
      & \hspace{-10em} \Big (
	 \cdpar*{\cdE*{(\phi_i - \cdE*{\phi_i(\Xpsi)})(\Xpsi)^6}}^{1 / 4} 
	 \cdpar*{\cdE*{(\phi_i - \cdE*{\phi_i(\Xpsi)})(\Xpsi)^2}}^{1 / 4} \\
      & + 2  \cdabs*{\partial^1_i \log Z(\psi)\partial^2_{i} \log Z(\psi)^{1  / 2}} \Big )  \\
\end{align*}
Summing over $ i $, we obtain
\begin{align*} 
  \MoveEqLeft \cdabs*{\cdE*{\| \cdPpsi \phi(X_1)\|^2}  -   \cdpar*{\cdE*{\phi(\Xpsi)}}^2} \\
  &\leq \sum\limits_{i=1}^{p} \,
    \cdabs*{\cdE*{(\cdPpsi \phi_i(X_1))^2}  -   \cdpar*{\cdE*{\phi_i(\Xpsi)}}^2} \\
  &\leq 
    \alpha^{2m} \msum\limits_{i=1}^{p}\cdE*{\cdpar*{\phi_i - \cdE*{\phi_i(\Xpsi)}}(\Xpsi)^2}  
    + C_{\chi}\alpha^{m / 2}  \cdnorm*{\log Z}_{2, \infty}  \cdnorm*{\psi - {\psi}^{\star}}  \\
  &\leq 
    \alpha^{2m} \sigma_{\psi}^2  
    + C_{\chi}\alpha^{m / 2}  \cdnorm*{\log Z}_{2, \infty}  \cdnorm*{\psi - {\psi}^{\star}}  \\
\end{align*}
where 
\begin{align*}
  \MoveEqLeft  \cdnorm*{\log Z}_{2, \infty} \\ 
  &= \msup_{\psi \in \Psi} 
    \msum\limits_{i=1}^{p} 
       \cdpar*{\cdE*{(\phi_i - \cdE*{\phi_i(\Xpsi)} )(\Xpsi)^6}}^{1 / 4} 
       \cdpar*{\cdE*{(\phi_i - \cdE*{\phi_i(\Xpsi)} )(\Xpsi)^2}}^{1 / 4}  \\
  & \quad + 2  \cdabs*{\partial^1_i \log Z(\psi)\partial^2_{i} \log Z(\psi)^{1  / 2}}	
\end{align*}
Similarly to the previous lemma, one can upper bound $ \cdE*{\big (\phi_i - \cdE*{\phi_i(\Xpsi)} \big )(\Xpsi)^6} $ using the \emph{centered} moment to cumulant formula:
\begin{align*} 
  \MoveEqLeft \cdE*{\big (\phi_i - \cdE*{\phi_i(\Xpsi)} \big )(\Xpsi)^6}  \\
  &= \! 15 \partial^{2}_{i} \log Z(\psi)^3 
    + 10 \partial^{3}_{i} \log Z(\psi)^2 
    + 15 \partial^{2}_{i} \log Z(\psi) \partial^{4}_{i} \log Z(\psi) 
    + \partial^{6}_{i} \log Z(\psi) \\
  &\eqqcolon F(\psi)
\end{align*}
To get a full description of $  \cdnorm*{\log Z}_{2, \infty} $:
\begin{align*} 
  \cdnorm*{\log Z}_{2, \infty} &= 
  \msup_{\psi \in \Psi}{
    \msum\limits_{i=1}^{p} 
      \cdpar*{F(\psi) \partial^2_i \log Z(\psi)}^{1/4}  
      + 2 \cdabs*{\partial^1_i \log Z(\psi)\partial^2_{i} \log Z(\psi)^{1/2}} 
  }\;.
\end{align*}
\end{proof}

We can now use the previous lemmas to obtain an expression on the second moment of the contrastive divergence gradient estimator, relating it
to the one of the stochastic log-likelihood gradient estimator.
\begin{lemma}\label{lemma:variance-terms}
  Under \ref{asst:SGD-A1}, \ref{asst:SGD-A4} and \ref{asst:SGD-A5},
  we have:
  \begin{equation*} 
\quad \cdE*{\cdnorm*{h_t(\psi, X_t)}^2}
					      \!\leq 2 \sigma_{\star}^2 + 2\sigma^2_{\psi} + 2 L^2 \left \| \psi - {\psi}^{\star} \right \|^2 
					      +4\cdpar[\big]{\sigma^2_{\psi} \alpha^{2m} \!+\! \alpha^{m / 2} \left \| \log Z \right \|_{3, \infty} C_{\chi} \left \| \psi - {\psi}^{\star} \right \|}
  \end{equation*}
  where 
  $ 
  \cdnorm*{\log Z}_{3, \infty} 
  \coloneqq 
  2 \max_{}\cdpar*{ \cdnorm*{\log Z}\smallsub{1, \infty},  \cdnorm*{\log Z}\smallsub{2, \infty} }
  $.
\end{lemma}

\begin{proof}
We rely on the following decomposition:
\begin{align*}
	h_{t}(\psi, X_t) 
	\;= &\; 
	  \underbrace{(\phi(X_t) - \cdE*{\phi(X_1)})}\limits_{\Delta_{1, 1}} +
	  \underbrace{(\cdE*{\phi(X_1)}  - \cdE[\big]{\phi(\Xpsi)})}\limits_{\Delta_{1, 2}}  \\
	    &\;
	  +\underbrace{(\cdPpsi\phi(X_t)  -\phi(K^m_{\psi}(X_t)))}\limits_{\Delta_2}  
	  + \underbrace{\cdpar*{\cdE*{\phi(\Xpsi)} - \cdPpsi \phi(X_t)}}\limits_{\Delta_3}
\end{align*}
$ \Delta_{1, 1} + \Delta_{1, 2} $ form the differentiable stochastic gradient $ g_t $ of Equation \ref{eq:online-mle-gradient}.
% , and (3) is a random variable that will be controlled thanks to the previous lemmas.
Note that $\Delta_{1, 1}$ is mean-zero, and $\Delta_2$ is mean-zero conditionally on $ X_t $.
Consequently, 
$ 
\cdE*{\cdinner*{\Delta_2, \Delta_3}}  
= \cdE*{\cdinner*{\Delta_2, \Delta_{1, 1}}}
= \cdE*{\cdinner*{\Delta_{1, 1}, \Delta_{1, 2}}} 
= \cdE*{\cdinner*{\Delta_{1, 1}, \Delta_2}} 
= 0
$,
and the only mixed-terms that remain to be controlled are 
$ \cdE*{\cdinner*{\Delta_{1, 1}, \Delta_3}} $ and $ \cdE*{\cdinner*{\Delta_{1, 2}, \Delta_3}} $.
We first control the unmixed terms, and the simple ones first: we have
$ \cdE*{\cdnorm*{\Delta_{1, 1}}^2} = \sigma_{\star}^2 $, as well as $ \cdE*{\cdnorm*{\Delta_{1, 2}}^2} \leq L^2  \cdnorm*{\psi - {\psi}^{\star}}^2$.
For $ \Delta_2 $, we have:
\begin{equation*} 
  \cdCondE{\cdnorm*{\Delta_2}^2 \given X_t} = \cdPpsi  \cdnorm*{\phi(X_t)}^{2} -  \cdnorm*{\cdPpsi \phi(X_t)}^2
\end{equation*}
We can invoke Lemmas \ref{lemma:quadratic-convergence-first-moment} and \ref{lemma:mean-convergence-second-moment}, which guarantee
\begin{align*}
  \cdabs*{\cdE*{\cdnorm*{\cdPpsi \phi(X_1)}^2} -  \cdnorm*{\cdE[\big]{\phi(\Xpsi)}}^2} 
    &\,\leq\, \alpha^{2m} \sigma_{\psi}^2  + C_{\chi}\alpha^{m / 2}  \cdnorm*{\log Z}_{2, \infty}  \cdnorm*{\psi - {\psi}^{\star}}  \\
  \cdabs*{\cdE*{\cdPpsi  \cdnorm*{\phi(X_1)}^2} - \cdE*{\cdnorm*{\phi(\Xpsi)}^2}}  
    &\,\leq\,  \alpha^{m} C_{\chi}   \cdnorm*{\log Z}_{1, \infty}  \cdnorm*{\psi - {\psi}^{\star}}
\end{align*}
to obtain
\begin{align*}
  \cdE*{\cdnorm*{\Delta_2}\smallpow{2}} 
  &= \cdE*{\cdnorm*{\phi(\Xpsi)}\smallpow{2}} -  \cdnorm*{\cdE*{\phi(\Xpsi)}}^2 \\ 
  & \quad 
    + \alpha^{2m} \sigma^2_{\psi} 
    + C_{\chi} \alpha^{m / 2}  \cdpar*{
	\cdnorm*{\log Z}\smallsub{1, \infty} 
	+ \cdnorm*{\log Z}\smallsub{2, \infty}
      }  \cdnorm*{\psi - {\psi}^{\star}} \\
  &= \cdE*{\cdnorm*{\phi(\Xpsi)}\smallpow{2}} -  \cdnorm*{\cdE*{\phi(\Xpsi)}}^2 
    + \alpha^{2m} \sigma^2_{\psi} \\ 
  & \quad 
    +  \alpha^{m / 2}   \cdnorm*{\log Z}_{3, \infty} C_{\chi}  \cdnorm*{\psi - {\psi}^{\star}}
\end{align*}
where $  \cdnorm*{\log Z}_{3, \infty} \coloneqq 2 \max_{}( \cdnorm*{\log Z}\smallsub{1, \infty},  \cdnorm*{\log Z}\smallsub{2, \infty} ) $.

For $ \Delta_3$, notice that $ \Delta_3 $ is precisely the term  $ \Delta_1 $ in Lemma \ref{lemma:quadratic-convergence-first-moment}, and we can thus bound it by
\begin{equation*} 
  \cdE*{\cdnorm*{\Delta_3}\smallpow{2}} \leq \sigma\smallpow{2}_{\psi} \alpha^{2m} + \alpha^{m / 2}  \cdnorm*{\log Z}_{2, \infty} C_{\chi}  \cdnorm*{\psi - {\psi}^{\star}} 
\end{equation*}
Finally, we simply bound $ 2\cdE*{\cdinner*{\Delta_{1, 1}, \Delta_3}}$ by $ \cdE*{\cdnorm*{\Delta_{1, 1}}\smallpow{2}} + \cdE*{\cdnorm*{\Delta_3}\smallpow{2}} $,
and $ 2\cdE*{\cdinner*{\Delta_{1, 2}, \Delta_3}}$ by $ \cdE*{\cdnorm*{\Delta_{1, 2}}\smallpow{2}} + \cdE*{\cdnorm*{\Delta_3}\smallpow{2}} $.
Putting everything together, we have:
\begin{align*}
  \cdE*{\cdnorm*{h_t(\psi)}\smallpow{2}}
    &= 
      \cdE*{\cdnorm*{\Delta_{1,1}}\smallpow{2}} 
      + \cdE*{\cdnorm*{\Delta_{1, 2}}\smallpow{2}} 
      + \cdE*{\cdnorm*{\Delta_2}\smallpow{2}} 
      + \cdE*{\cdnorm*{\Delta_3}\smallpow{2}} \\ 
    & \quad 
      + 2 \cdE*{\cdinner*{\Delta_{1, 1}, \Delta_3}} 
      + 2 \cdE*{\cdinner*{\Delta_{1, 2}, \Delta_3}} \\
    &\leq 
      2\cdE*{\cdnorm*{\Delta_{1, 1}}\smallpow{2}} 
      + 2\cdE*{\cdnorm*{\Delta_{1, 2}}\smallpow{2}} 
      + \cdE*{\cdnorm*{\Delta_2}\smallpow{2}} 
      + 3 \cdE*{\cdnorm*{\Delta_3}\smallpow{2}} \\
    &\leq 
      2 \sigma_{\star}^2 + 2 L^2  \cdnorm*{\psi - {\psi}^{\star}}^2 + \sigma^2_{\psi}
      + 4(\sigma^2_{\psi} \alpha^{2m} 
      + \alpha^{m / 2}  \cdnorm*{\log Z}_{3, \infty} C_{\chi}  \cdnorm*{\psi - {\psi}^{\star}})  \\
    &\leq 
      2 \sigma_{\star}^2 + 2\sigma^2_{\psi} 
      + 2 L^2  \cdnorm*{\psi - {\psi}^{\star}}^2 
      + 4(\sigma^2_{\psi} \alpha^{2m} 
      + \alpha^{m / 2}  \cdnorm*{\log Z}_{3, \infty} C_{\chi}  \cdnorm*{\psi - {\psi}^{\star}} )\,.
\end{align*}
\end{proof}

\subsection{Proof of the SGD recursion (Lemma \ref{lem:online-cd-recursion})} \label{app:sgd-style-recursion}
We are now ready to provide an SGD-style recursion for the expected squared distance to the optimum
$ \delta_t \coloneqq \cdE*{\cdnorm*{\psi_{t} - {\psi}^{\star}}^2} $
for the iterates $ \psi_t $ produced by Algorithm \ref{alg:online_cd}.
We define the following quantities:
\begin{align*}
  \sigma_{\star} = (\cdE*{\| \phi(X_1) - \cdE*{\phi(X_1)} \|^2})^{1 / 2}
  \,, \quad 
  \sigma(\psi_t) = (\cdCondE*{\| \phi(X^{\psi_t}) - \cdE*{\phi(X^{\psi_{t}})}\|^2 \given \mathcal  F_t})^{1 / 2}\,,
\end{align*}
and $\sigma_t \coloneqq \cdE*{\sigma(\psi_t)}$.

\begin{lemma-non}[Restatement of Lemma \ref{lem:online-cd-recursion}]
  Under Assumptions \ref{asst:SGD-A1}, \ref{asst:SGD-A4} and \ref{asst:SGD-A5}, for all $  t \geq  1 $, we have:
  \begin{equation*} 
	  \delta_{t} \leq 
	     \cdpar*{1 - 2 \eta_t \tilde{\mu}_{m, t-1} + 2 \eta_t^2 L^2} \delta_{t-1} 
	     + \eta_t^2 \tilde{\sigma}_{m, t-1}^2  
	     +  4 \alpha^{m / 2} \eta_t^2   \cdnorm*{\log Z}_{2, \infty} C_{\chi} \delta_{t-1}^{1 / 2}\,,
  \end{equation*}
  where $  \cdnorm*{\log Z}_{3, \infty} $  is a
  constant,
  $ \tilde{\mu}_{m, t} \coloneqq \mu - \alpha^{m} \sigma_{t} C_{\chi} $, and 
  $ \tilde{\sigma}_{m, t} \coloneqq (\sigma_{\star}^2 + \sigma_t^2 + 2\sigma^2_{t} \alpha^{2m})^{1 / 2}$.
\end{lemma-non}
\begin{proof}
  In this proof, we note $ (\mathcal F_{t})_{t \geq  0} $, the increasing family of $ \sigma $-algebras generated by the random variables
  $ (X_{t})_{t \geq  0} \sim p_{{\psi}^{\star}} $ and the Markov chain samples  $ K^m_{\psi_{t-1}}(X_t)$.
  We decompose the integrand of $ \delta_t $ as follows:
  \begin{align*}
	   \cdnorm*{\psi_{t} - {\psi}^{\star}}^2 
	  &  =  
	    \cdnorm*{\mathrm{Proj}_{\Psi}(\psi_{t-1} - \eta_t h_t(\psi_{t-1}, X_t)) - {\psi}^{\star}}^2\\
	  &  \leq  
	    \cdnorm*{\psi_{t-1} - \eta_t h_t(\psi_{t-1}, X_t) - {\psi}^{\star}}^2 \quad \text{(By Lemma \ref{lemma:contraction})}\\
	  &  =   
	    \cdnorm*{\psi_{t-1} - {\psi}^{\star}}^2  
	    - 2 \eta_t  \cdinner*{h_t(\psi_{t-1}, X_t), \psi_{t-1} - {\psi}^{\star}}  
	    + \eta_t^2  \cdnorm*{ h_t(\psi_{t-1}, X_t)  }^2 \\
  \end{align*}
  The first term is (up to an averaging operation) the previous iterate.  The middle term will ensure (provided $ m $ is large enough) contraction of the expected distance to the optimum. 
  Finally, the third term can be described by Lemma \ref{lemma:variance-terms}, and essentially behaves like the second moment of a log-likelihood stochastic gradient.
  Indeed, noting 
  \begin{equation*} 
  \overline{g}(\psi_{t-1}) \coloneqq  - \cdE*{\phi(X_1)} + \cdE*{\phi(X^{\psi_{t-1}})}\,,
  \end{equation*}
  the expectation of $ g_t $ w.r.t $ x_t $ (which is the gradient
  of the negative cross-entropy between $ p_{\psi} $ and $ p_{{\psi}^{\star}}$), we have:
  \begin{align*}
     \cdinner*{h_t(\psi_{t-1}, X_t) , \psi_{t-1} - {\psi}^{\star}}   
     &=  
       \cdinner*{\overline{g}(\psi_{t-1}) , \psi_{t-1} - {\psi}^{\star}}\\ 
     & \quad 
      + \underbrace{
	  \cdinner*{
	    (h_t(\psi_{t-1}, X_t) - \overline{g}(\psi_{t-1})) ,
	    \psi_{t-1} - {\psi}^{\star}
	  }  
	}\limits_{\Delta}\,.
  \end{align*}
  Applying Lemma \ref{lemma:approx-contraction}, we get that
  \begin{align*}
    h_t(\psi_{t-1}, X_t) - \overline{g}(\psi_{t-1}) 
      &= \phi(K_{\psi_{t-1}}^{m}(X_t) - \cdCondE[\big]{\phi(X^{\psi_{t-1}}) \given \mathcal  F_{t-1}} \\
    \implies \cdCondE*{h_t(\psi_{t-1}, X_t) - \overline{g}(\psi_{t-1}) \given \mathcal  F_{t-1}}  
      &=\cdP{\psi_{t-1}} \phi - \cdCondE[\big]{\phi(X^{\psi_{t-1}}_{1}) \given \mathcal  F_{t-1}}\;, \\
  \end{align*}
  meaning
  \begin{align*}
     \cdabs*{\cdCondE*{\Delta \given \mathcal  F_{t-1}}}  
      &\leq  \cdnorm*{\cdCondE*{\cdP{\psi_{t-1}}\cdpar*{\phi - \cdE*{\phi(\cramped{X^{\psi_{t-1}}})}}(X_{t}) \given \mathcal  F_{t-1}}} 
       \cdnorm*{\psi_{t-1} - {\psi}^{\star}} \\
    & \leq \alpha^{m} \sigma(\psi_{t-1}) C_{\chi}   \cdnorm*{\psi_{t-1} - {\psi}^{\star}}^2 \\
  \end{align*}
  On the other hand, by applying \cref{lem:replace:gradient:inequality} to $ \overline{g} $, we have:
  \begin{equation*} 
   \cdinner*{\overline{g}(\psi_{t-1}), \psi_{t-1} - {\psi}^{\star}} 
   =  \cdinner*{\overline{g}(\psi_{t-1}) - \overline{g}({\psi}^{\star}), \psi_{t-1} - {\psi}^{\star}}  
   \geq  \mu  \cdnorm*{\psi_{t-1} - {\psi}^{\star}}^2
  \end{equation*}
  Combining the above results, we obtain:
  \begin{multline}
    \cdCondE*{\cdnorm*{\psi_{t} - {\psi}^{\star}}^2 \given \mathcal  F_{t-1}} 
    \leq (1 - 2\eta_t (\mu - \alpha^{m} \sigma(\psi_{t-1}) C_{\chi})) \| \psi_{t-1} - {\psi}^{\star} \|^2 \\
    + \eta_t^2 (2 \sigma_{\star}^2 + 2\sigma^2(\psi_{t-1}) + 2 L^2  \cdnorm*{\psi_{t-1} - {\psi}^{\star}}^2  \\
    + 4(\sigma^2(\psi_{t-1}) \alpha^{2m} + \alpha^{m / 2}  \cdnorm*{\log Z}_{3, \infty} C_{\chi} 
       \cdnorm*{\psi_{t-1} - {\psi}^{\star}} ))
  \end{multline}
  And the result follows by integrating over $ \mathcal F_{t-1} $.
\end{proof}

\subsection{Proof of Online CD convergence} \label{app:online-cd-convergence}
We now prove \Cref{thm:cd-sgd}.
The recursion of Lemma \ref{lem:online-cd-recursion} is almost identifiable, up to a cross-term of second order, with the one of an SGD algorithm as presented in the setting of \cite[Theorem 1]{moulines2011non}.
To make the identification exact, we use the bound $ 4 \alpha^{m / 2} \eta_t^2  \cdnorm*{\log Z}_{3, \infty} C_{\chi} \delta_{t-1}^{1 / 2} \leq 2 \alpha^{m / 2} \eta_{t}^2 \delta_{t} + 2 \alpha^{m / 2} \eta_t^2 \left \| \log Z \right \|_{3, \infty}^2 C_{\chi}^2
$, yielding the following recursion:
  \begin{equation*} 
	  \delta_{t} \leq 
	    (1 - 2 \eta_t (\mu - \alpha^{m} \sigma C_\chi) + 2\eta_t^2({L}^2 \!+\! \alpha^{m / 2})) \delta_{t-1} 
	    + (\sigma^2(2 + 2 \alpha^{2m})  
	    \!+\! \alpha^{m / 2}  \cdnorm*{\log Z}_{3, \infty}^{2} C_{\chi}^{2})\eta_t^2
  \end{equation*}
  where we used the fact that $ \tilde{\sigma}_{m, t} \leq \sigma$. This recursion is of the same form as the one studied in \cite[Equation 6, Theorem 1]{moulines2011non} given by:
  \begin{equation*} 
  \delta_t \leq \cdpar*{1-2 \textcolor{blue}{\mu} \gamma_t+2 \textcolor{blue}{L}^2 \gamma_t^2} \delta_{t-1}
    +2 \textcolor{blue}{\sigma}^2 \gamma_t^2
  \end{equation*}
  by identifying:
  \begin{align*}
  \textcolor{blue}{\sigma}^2 
    &\leftarrow \sigma^2(2 + 2 \alpha^{2m})  + \alpha^{m / 2}  \cdnorm*{\log Z}_{3, \infty}^{2} \eqqcolon \tilde{\sigma}_m^2   \\
  \textcolor{blue}{L}^2 
    &\leftarrow (L^2 + \alpha^{m / 2}) \eqqcolon \tilde{L}_m^2\\
  \textcolor{blue}{\mu}  
    &\leftarrow \mu - \alpha^{m} \sigma C_{\chi} \eqqcolon \tilde{\mu}_m\\
  \gamma_t 
    &\leftarrow \eta_t
  \end{align*}
  We can use the same unrolling strategy as theirs 
  (the only condition required to proceed is that $ \tilde{\mu}_m < \tilde{L}_m $, which automatically holds since $ \mu < L $), and we obtain
  \begin{equation*} 
  \delta_n \leq 
  \begin{cases}
      2 \exp \cdpar*{4 \tilde{L}_m C^2 \varphi_{1-2 \beta}(n)} \exp \cdpar*{-\frac{\tilde{\mu}_m C}{4} n^{1-\beta}}
      \cdpar*{\delta_0+\frac{\tilde{\sigma}_m^2}{\tilde{L}_m^2}}
    +
      \frac{4 C \tilde{\sigma}_m^2}{\tilde{\mu}_m n^\beta}, & \text {if } 0 \leq \beta<1 \\ 
      \frac{\exp \cdpar*{2 \tilde{L}_m^2 C^2}}{n^{\tilde{\mu}_m C}}
      \cdpar*{\delta_0+\frac{\tilde{\sigma}_m^2}{\tilde{L}_m^2}}
    +
      2 \tilde{\sigma}_m^2 C^2 \frac{\varphi_{\tilde{\mu}_m C / 2-1}(n)}{n^{\tilde{\mu}_m C / 2}}, & \text {if } \beta=1.
  \end{cases}
  \end{equation*}

\qed

\subsection{Proof of online CD with averaging (Theorem \ref{thm:online-cd-averaging})}\label{app:proof-online-cd-averaging}

We first restate the theorem in its complete form.
\begin{theorem-non}[Contrastive Divergence with Polyak-Ruppert averaging]
	Let $ (\psi_t)_{t \geq  0} $ the sequence of iterates obtained by running the CD algorithm
	with a learning rate $ \eta_t = C t^{-\beta} $ for $ \beta \in (\frac{1}{2}, 1) $.
	Define $ \bar{\psi}_n \coloneqq \frac{1}{n}\sum_{i=1}^{n} \psi_i $. Then, under the same assumptions as 
	Theorem \ref{thm:cd-sgd} we have, for all $ n \geq  1$,
	\begin{equation*} 
	  \sqrt{\cdE*{\cdnorm*{\overline{\psi}_n - {\psi}^{\star}}^{2}}}
	  \;\leq\;
	  2\sqrt{\mfrac{\mathrm{Tr}(\mathcal  I(\psi)^{-1})}{n}} 
	  + \mathcal O  \cdpar*{n^{\max_{} \cdpar*{
	    - \cdpar*{\frac{1}{2} + \frac{\beta}{4}}, 
	    -\beta, 
	    \frac{\beta}{2} - 1, 
	    -\left (\frac{\beta}{2} + m\frac{\lvert   \log \alpha \rvert}{\log n} \right ) 
	  }}}
	\end{equation*}
	Where 
	$ \mathcal  I({\psi}^{\star}) \coloneqq \mathrm{Cov} \left \lbrack \phi(X_1) \right \rbrack $ 
	is the Fisher information matrix of the data distribution.  
	Additionally, if $ m > \frac{(1 - \beta) \log n}{2|\log \alpha|}$, we have
	$
	  \sqrt{\cdE*{\cdnorm*{\overline{\psi}_n - {\psi}^{\star}}^{2}}} 
	    \;\leq\;
	  2\sqrt{\mfrac{\mathrm{tr}(\mathcal  I(\psi)^{-1})}{n}} + o\left (n^{-1/2} \right)
	$.
\end{theorem-non}

Recall that we denote by $h_n$ the standard online CD gradient defined in Equation \ref{eq:generic_cd_gradient}:
\begin{equation*} 
	h_n(\psi_{n-1}, X_n) = -\phi(X_n) + \phi(K^{m}_{\psi_{n-1}}(X_n)), \quad X_n \sim p_{{\psi}^{\star}}, \quad \forall n \in \mathbb N \setminus \{0\},
\end{equation*}
as well as
\begin{equation*} 
  \overline{h}(\psi_{n-1}) \coloneqq \cdCondE*{h_n(\psi_{n-1}, X_n) \given \mathcal  F_{n-1}}  
  = \cdE*{\phi(X_n)} + \cdCondE[\big]{\phi(K^{m}_{\psi_{n-1}}(X_{n})) \given \mathcal  F_{n-1}}.
\end{equation*}
In the following, we will hide the dependence on $ X_n $ of $ h_n(\psi, X_n) $ by writing $ h_n(\psi) $.
We start by establishing some intermediate lemmas.
\begin{lemma}\label{lem:delta-sum-convergence}
  Under Assumptions \ref{asst:SGD-A1}, \ref{asst:SGD-A4}, \ref{asst:SGD-A5}, 
  the online CD iterates produced by Algorithm \ref{alg:online_cd} using 
  $ \eta_t = C t^{-\beta} $ for $ \beta \in (\frac{1}{2}, 1) $ verify
  \begin{equation*} 
    \mfrac{1}{n} \sqrt{\msum\limits_{i=1}^{n}  \cdpar*{\cdE*{\cdnorm*{\psi_{i} - {\psi}^{\star}}^{2}}}} 
    =  \mathcal O  \cdpar*{n^{-\frac{1}{2} - \frac{\beta}{2}}  }.
  \end{equation*}
\end{lemma}
\begin{proof}
  Let us note $ \delta_n = \cdE*{\cdnorm*{\psi_n - {\psi}^{\star}}^{2}} $. 
  Summing the r.h.s of Theorem \ref{thm:cd-sgd}, we have
\begin{align*}
  \msum\limits_{i=1}^{n} \delta_i & \leq 
  \msum_{i=1}^{n}
    \mfrac{4 C \tilde{\sigma}_m^2}{\tilde{\mu}_m i^\beta} 
    + 2 \cdpar*{\delta_0+\mfrac{\tilde{\sigma}_m^2}{\tilde{L}_m^2}} 
      \underbrace{\msum_{i=1}^{n}e^{4 \tilde{L}_m C^2 \varphi_{1-2 \beta}(i)} e^{-\mfrac{\tilde{\mu}_m C}{4} n^{1-\beta}}}_{A_3} \\
	\\
  \implies \mfrac{1}{n} \sqrt{\msum\limits_{i=1}^{n} \delta_i} &\leq 
  \mfrac{1}{n} \sqrt{\mfrac{4 C \tilde{\sigma}^2_m}{\tilde{\mu}_m}\varphi_{1 - \beta}(n)} 
  + \mfrac{1}{n}\sqrt{\cdpar[\Big]{2 (\delta_0 + \mfrac{\tilde{\sigma}_m^2}{\tilde{L}_m^2})A_3}}
  = \mathcal O \cdpar*{n^{-\mfrac{1}{2} - \mfrac{\beta}{2}}}.
\end{align*}
where $ A_3 $ is finite if $ \beta < 1 $, and $ A_3 = O(n) $ otherwise \cite{moulines2011non}.
\end{proof}
\begin{lemma}\label{lem:square-root-delta-sum-convergence}
  Under Assumptions \ref{asst:SGD-A1}, \ref{asst:SGD-A4}, \ref{asst:SGD-A5}, 
  the online CD iterates produced by Algorithm \ref{alg:online_cd} using 
  $ \eta_t = C t^{-\beta} $ for $ \beta \in (\frac{1}{2}, 1) $ verify
\begin{equation*} 
\mfrac{1}{n} \sqrt{\msum\limits_{i=1}^{n}  \cdpar*{\cdE*{\cdnorm*{\psi_{i} - {\psi}^{\star}}^{2}}}^{1/2}} 
=  \mathcal O(n^{-\frac{1}{2} - \frac{\beta}{4}}).
\end{equation*}
\end{lemma}
\begin{proof}
  Let us note $ \delta_n = \cdE*{\cdnorm*{\psi_n - {\psi}^{\star}}^{2}} $. 
  Applying $\sqrt{x+y} \leq \sqrt{x} + \sqrt{y}$ to the r.h.s of Theorem \ref{thm:cd-sgd}, we have
\begin{align*}
	\msum\limits_{i=1}^{n} \delta_i^{1 / 2} 
	&\leq 
	  \mfrac{2C^{1 / 2} \tilde{\sigma}_m}{2 \tilde{\mu}_m^{1 / 2}}\msum\limits_{i=1}^{n} i^{-\beta / 2} 
	  + 
	    \sqrt {2 \Big(\delta_0 + \mfrac{\tilde{\sigma}_m^2}{\tilde{L}_m^2}\Big)}
	    \underbrace{
	      \msum\limits_{i=1}^{n}
		e^{2 \tilde{L}_m^2C^2 \varphi_{1 - 2 \beta}(i)}
		e^{- \frac{\tilde{\mu}_m C}{8} i^{1 - \beta}}
	    }\limits_{A_4}  \\
	\mfrac{1}{n} \sqrt{\msum\limits_{i=1}^{n} \delta_i^{1 / 2}} 
	&\leq 
	  \mfrac{1}{n} \sqrt{\mfrac{2 C^{1 / 2} \tilde{\sigma}_m}{\tilde{\mu}_m^{1 / 2}}\varphi_{1 - \beta / 2}(n)} 
	  + \mfrac{1}{n} \cdpar*{\sqrt {2  \cdpar*{\delta_0 + \mfrac{\tilde{\sigma}_m^2}{\tilde{L}_m^2}}}A_4}^{1 / 2}
	= \mathcal O \cdpar*{n^{-\frac{1}{2} - \frac{\beta}{4}}}.
\end{align*}
where $ A_4 $ is finite if $ \beta < 1 $, and $ A_4 = O(n) $ otherwise \cite{moulines2011non}.
\end{proof}

\begin{lemma}\label{lem:h_n_psi-convergence}
  Under Assumptions \ref{asst:SGD-A1}, \ref{asst:SGD-A4}, \ref{asst:SGD-A5}, the online CD iterates produced by Algorithm \ref{alg:online_cd} using 
  $ \eta_t = C t^{-\beta} $ for $ \beta \in (\frac{1}{2}, 1) $ verify
\begin{equation*} 
\begin{aligned}
\sqrt{\cdE*{\cdnorm*{\mfrac{1}{n} \msum\limits_{i=1}^{n} h_i(\psi_{i-1})}^2}} 
= \mathcal O \cdpar*{n^{\max \cdpar*{\frac{\beta}{2} - 1, - \frac{2 \beta}{2 \beta + 1}}}}.
\end{aligned}
\end{equation*}
\end{lemma}
\begin{proof}
The main subtlety consists in taking into account the projection step. Indeed, write $ h_t(\psi_t) $ as:
\begin{align*}
  \psi_{t} &=  \mathrm{Proj}_{\Psi}(\psi_{t-1} - \eta_t h_t(\psi_{t-1})) \\
	   & = \psi_{t-1} - \eta_t h_t(\psi_{t-1}) 
		+ \underbrace{
		     \cdpar*{\mathrm{Proj}_{\Psi}(\psi_{t-1} - \eta_t h_t(\psi_{t-1})) - (\psi_{t-1} - \eta_t h_t(\psi_{t-1}))} 
		   }_{\Delta_t}  \\
  \implies h_t 
	   & = \mfrac{1}{\eta_t}  \cdpar*{\psi_{t-1} - \psi_{t}} + \mfrac{\Delta_t}{\eta_t}  \\
  \MoveEqLeft[3.7] \implies \sqrt {\cdE*{\cdnorm*{\mfrac{1}{n} \msum_{i=1}^{n} h_i(\psi_{i-1})}^2}}  \\ 
	   &\leq
	    \sqrt {\cdE*{\cdnorm*{\mfrac{1}{n} \msum_{i=1}^{n} \mfrac{1}{\eta_i}  \cdpar*{\psi_{i-1} - \psi_i}}^2}}  
	    + \sqrt {\cdE*{\cdnorm*{\mfrac{1}{n}\msum_{i=1}^{n} \mfrac{\Delta_i}{\eta_i}}^2}} 
\end{align*}
where the last line used Minkowski's inequality. We now control the two terms separately.
The first term is well-known in the SGD literature, and can be controlled in a straighforward manner.
The second term characterizes the deviation caused by the projection step, and we control it below.
\paragraph{Controlling the second term} 
We use a similar strategy as in \cite{godichon2023non}, but without requiring the assumption on the boundedness of the gradients.
Let us use the notation $ h_t \coloneqq h_t(\psi_{t-1}) $ for brevity.
First, we have, using again Minkowski's inequality, that:
\begin{equation*} 
\sqrt {\cdE*{\cdnorm*{\mfrac{1}{n} \msum_{i=1}^{n} \mfrac{\Delta_i}{\eta_i}}^2}}  
\leq \mfrac{1}{n} \msum_{i=1}^{n} \mfrac{1}{\eta_i} \sqrt {\cdE*{\cdnorm*{{\Delta_i}}^2}}\;.
\end{equation*}
Note that $ \Delta_t $ is $ 0 $ if $ \psi_{t-1} - \eta_t h_t\in \Psi $, e.g.
\begin{equation*} 
  \Delta_t =  \cdpar*{\mathrm{Proj}_{\Psi}(\psi_{t-1} - \eta_t h_t) - (\psi_{t-1} - \eta_t h_t)} 
    \times \mathbb{I}_{\{\psi_{t-1} - \eta_t h_t \not \in \Psi\}}
\end{equation*}
Which implies, using H\"older's inequality, that
\begin{align*}
  \MoveEqLeft \cdE*{\cdnorm*{\Delta_t}^2} \\ 
  = {} & \cdE*{
    \cdnorm*{\mathrm{Proj}_{\Psi}(\psi_{t-1} - \eta_t h_t) - (\psi_{t-1} - \eta_t h_t)}^2 \times 
    \mathbb{I}_{\{\psi_{t-1} - \eta_t h_t \not \in \Psi\}}
  }  \\ 
\leq {} & \cdE*{\cdnorm*{\mathrm{Proj}_{\Psi}(\psi_{t-1} - \eta_t h_t) - (\psi_{t-1} - \eta_t h_t)}^{2p}}^{\frac{1}{p}} \times \\ 
	& {} \hspace{20em}\cdPr*{\left \{\psi_{t-1} - \eta_t h_t \not \in \Psi \right \}}^{\frac{1}{q}}
\end{align*}
For any pair $ (p, q) $ of H\"older conjugates, whose exact value will be determined later.
We investigate the two terms separately; for the first term, note that
\begin{align*}
  \MoveEqLeft  \cdnorm*{\mathrm{Proj}_{\Psi}(\psi_{t-1} - \eta_t h_t) - (\psi_{t-1} - \eta_t h_t)}^2 \\ 
  & \leq 2  \cdnorm*{\mathrm{Proj}_{\Psi}(\psi_{t-1} - \eta_t h_t) - \psi_{t-1}}^2 + 2  \cdnorm*{\eta_t h_t}^2 \\
  & \leq 2  \cdnorm*{\mathrm{Proj}_{\Psi}(\psi_{t-1} - \eta_t h_t)- \mathrm{Proj}_{\Psi}(\psi_{t-1})}^2 + 2 \eta_t^2  \cdnorm*{h_t}^2 \\ 
  & \overset{(a)}{\leq} 
    2  \cdnorm*{\psi_{t-1} - \eta_t h_t - \psi_{t-1}}^2 
    + 2 \eta_t^2  \cdnorm*{h_t}^2 \\ 
  & \leq 4 \eta_t^2  \cdnorm*{h_t}^2
\end{align*}
where in $(a)$, we used the fact that projections are $1$-Lipschitz.
Thus, we have:
\begin{align*}
   \cdpar*{\cdE*{\cdnorm*{\mathrm{Proj}_{\Psi}(\psi_{t-1} - \eta_t h_t) - (\psi_{t-1} - \eta_t h_t)}^{2p}}  }^{1 / p}
  & \leq   \cdpar*{4^{p} \eta_t^{2p} \times \cdE*{\cdnorm*{h_t}^{2p}}}^{1 / p} \\ 
  & \leq  16 \eta_t^{2}  \times  \cdpar*{\tau_{2p}^{2p} + 2 C_{\chi} r_{\Psi} \tau_{4p}^{2p}}^{1 /p}
\end{align*}
where we noted, as in Lemma \ref{lem:quartic-convergence},  
$ \tau_p \coloneqq \left (\sup_{\psi \in \Psi} \mathbb{E} \left \| \phi \right \|^{p} \right )^{1 / p} < +\infty $,
$ r_{\Psi} = \sup_{\psi, \psi' \in \Psi} \left \| \psi - \psi' \right \|$,
and we relied on the fact that
\begin{equation*} 
  \cdCondE[\Big]{\cdnorm*{h_t}^{2p} \given \mathcal  F_{t-1}} 
    \leq 
      2^{2p - 1}  
      \cdpar*{
	\tau_{2p}^{2p}  + 2 C_{\chi} r_{\Psi}\tau_{4p}^{2p}  \cdnorm*{\psi_{t-1} - {\psi}^{\star}} 
      } \\ 
    \leq 
      2^{2p - 1}  \cdpar*{\tau_{2p}^{2p}  + 2 C_{\chi} r_{\Psi}\tau_{4p}^{2p}}
\end{equation*}
For the probability term, we have:
\begin{align*}
  \cdE*{\cdnorm*{\psi_{t} - {\psi}^{\star}}\smallpow{4}} 
  & \geq  
  \cdE*{
    \cdnorm*{
      (\psi_{t} - {\psi}^{\star}) 
      \times \mathbb{I}_{\left \{\psi_{t-1} - \eta_t h_{t} \not \in \Psi \right \}}}  
  }^{4} \\  
  & \geq  
    \mathrm{d}_{{\psi}^{\star}}^{4} 
    \times \cdPr*{\left \{\psi_{t-1} - \eta_t h_{t} \not \in \Psi \right \}} \\
  \implies \cdPr*{\left \{\psi_{t-1} - \eta_t h_t(\psi_{t-1}) \not \in \Psi \right \}} 
  & \leq \mfrac{1}{\mathrm{d}_{{\psi}^{\star}}^{4}} \cdE*{\cdnorm*{\psi_{t} - {\psi}^{\star}}\smallpow{4}}
\end{align*}
where we noted 
$ \mathrm{d}_{{\psi}^{\star}} = \inf_{\psi \in \partial \Psi}  \cdnorm*{\psi - {\psi}^{\star}}$, 
which is positive by \ref{asst:SGD-A1} as $ {\psi}^{\star} \in \mathrm{int}(\Psi) $.
Combining the two bounds, we thus obtain:
\begin{align*}
  \MoveEqLeft \mfrac{1}{n} \msum_{i=1}^{n} \mfrac{1}{\eta_i} \sqrt {\cdE*{\left \| \Delta_i \right \|^2}}   \\
  & \leq \mfrac{1}{n} \msum_{i=1}^{n} 
    \mfrac{1}{\eta_i} 
    \times  \cdpar*{
      4 \eta_i  \cdpar*{\tau_{2p}^{2p}  + 2 C_{\chi} r_{\Psi}\tau_{4p}^{2p}}^{1 / 2p}} 
    \times \mfrac{1}{d_{{\psi}^{\star}}^{2 / q}} \cdpar*{
      \cdE*{\cdnorm*{\psi_{i} - {\psi}^{\star}}\smallpow{4}}  
    }^{\frac{1}{2q}} \\ 
  & = 
    \mfrac{
      4  \cdpar*{\tau_{2p}^{2p} + 2 C_{\chi} r_{\Psi}\tau_{4p}^{2p}}^{1/ 2 p}
    }{
      d_{{\psi}^{\star}}^{2 / q}
    } 
    \times \mfrac{1}{n} \msum_{i=1}^{n}  \cdpar*{\cdE*{\cdnorm*{\psi_{i} - {\psi}^{\star}}\smallpow{4}}}^{1 / 2q} \\
\end{align*} 
Recalling that
\begin{align*}
  \MoveEqLeft \mfrac{1}{n} \msum_{k=1}^n \cdpar*{\cdE*{\cdnorm*{\psi_k-{\psi}^{\star}}\smallpow{4}}}^{1 / 2q}  \\ 
& \leq 
  \mfrac{1}{n} \msum_{k=1}^n \left (
    \left (\mfrac{12 C^{3 / 2} \tilde{\tau}_1^{2}}{\tilde{\mu}_m^{1 / 2}} \right )^{1 / q} \mfrac{1}{k^{3 \beta / 2q}}
    + \cdpar*{\mfrac{26 \tilde{\tau}_1^2 C}{\tilde{\mu}_m}}^{1 / q} \mfrac{1}{k^{\mfrac{\beta}{q}}}
  \right ) \\
& \quad
  +\mfrac{1}{n} \msum_{k=1}^n \exp \cdpar*{
    -\mfrac{1}{q} \times  \cdpar*{
      \mfrac{\tilde{\mu}_m C}{16} k^{1-\beta}
      +16 \tilde{L}_1^2 C^2 \varphi_{1-2 \beta}(k)
      +24 \tilde{L}_1^4 C^4  
    }
  } \\
& \qquad
  \times
  \cdpar*{
    32 \tilde{\tau}_1^4 C^4+\mfrac{160 \tilde{\tau}_1^4 C^3}{\tilde{\mu}_m}
    +\cdE*{\cdnorm*{\psi_0-{\psi}^{\star}}\smallpow{4}}
    +\mfrac{20 C \tilde{\tau}_1^2}{\tilde{\mu}_m} \delta_0
  }^{1 / 2q} \\ 
& \leq 
  \mfrac{1}{n}   \cdpar*{
    \cdpar*{\mfrac{12 C^{3 / 2} \tilde{\tau}_1^{2}}{\tilde{\mu}_m^{1 / 2}}}^{1 / q} 
    \times \varphi_{1 - \frac{3 \beta}{2 q}}(n) + \cdpar*{\mfrac{26 \tilde{\tau}_1^2 C}{\tilde{\mu}_m}}^{1 / q} 
    \times \varphi_{1 - \frac{\beta}{q}}(n) 
  } \\ 
&  \quad
  + \mfrac{A_5}{n} \times \cdpar*{
    32 \tilde{\tau}_1^4 C^4+\mfrac{160 \tilde{\tau}_1^4 C^3}{\tilde{\mu}_m}
    +\cdE*{\cdnorm*{\psi_0-{\psi}^{\star}}\smallpow{4}}
    +\mfrac{20 C \tilde{\tau}_1^2}{\tilde{\mu}_m} \delta_0
  }^{1 / 2q}
\end{align*}
where $ \tilde{\tau}_1 $ and $ \tilde{L}_1 $ are defined in Lemma \ref{lem:quartic-convergence}, and 
\begin{equation*} 
A_5 = \msum_{k=1}^n \exp \cdpar*{-\mfrac{1}{q} 
  \times  \cdpar*{
    \mfrac{\tilde{\mu}_m C}{16} k^{1-\beta}+16 \tilde{L}_1^2 C^2 \varphi_{1-2 \beta}(k)+24 \tilde{L}_1^4 C^4 
}}\;,
\end{equation*}
which is finite if $ \beta < 1 $, and $ O(n) $ otherwise.
Consequently, we have  that
$$ 
\mfrac{1}{n} \msum_{k=1}^{n}  \cdpar*{\cdE*{\cdnorm*{\psi_{k} - {\psi}^{\star}}^{4}}}^{1/2q} 
= \mathcal O(n^{-\frac{\beta}{q}})\,.
$$
In particular, taking $ (p, q) = (\frac{2 \beta + 1}{2 \beta - 1}, \frac{1}{2} + \beta) $,
we have that
\begin{equation*} 
\cdE*{\cdnorm*{\mfrac{1}{n} \msum_{i=1}^{n} \mfrac{\Delta_i}{\eta_i}}^2} 
= \mathcal O \cdpar*{n^{- \frac{2 \beta}{2 \beta + 1}}} = o\cdpar*{n^{-1/2}}.
\end{equation*}
for all $ \beta \in  \cdpar*{\frac{1}{2}, 1  } $.
\paragraph{Controlling the first term} 
The first term was handled in the case of standard SGD \cite[Theorem 3]{moulines2011non},
and the only condition needed to reuse their steps is that $ (\psi_{t})_{t \leq n} $
satisfies an upper bound of the same form as the one \cite[Theorem 1]{moulines2011non}
derived. This is precisely the nature of our bound of $ \psi_{t} $ estblished in Theorem  \ref{thm:cd-sgd},
with $ \tilde{\mu}_m, \tilde{\sigma}_m, \tilde{L}_m$. Borrowing on their result, we have:
\begin{align*}
  \sqrt{\cdE*{\cdnorm*{\mfrac{1}{n} \msum\limits_{i=1}^{n} \mfrac{1}{\eta_i} \left (\psi_{i-1} - \psi_i \right )}^2}} 
  &\leq 
    \mfrac{4 \tilde{\sigma}_m \beta}{C^{1 / 2} n \tilde{\mu}_m} \varphi_{\beta / 2}(n)
    +\mfrac{4 \beta}{C n \tilde{\mu}_m^{1 / 2}} \Big(\delta_0+\mfrac{\tilde{\sigma}_m^2}{\tilde{L}_m^2}\Big)^{1 / 2} A_2
 \\
  &\qquad 
     + \mfrac{1}{n \tilde{\mu}_m^{1 / 2}}\Big(\mfrac{1}{C}+2 \tilde{L}_m\Big) \delta_0^{1 / 2} 
     + \mfrac{2 \tilde{L}_m}{n \tilde{\mu}_m^{1 / 2}} \mfrac{2 C^{1 / 2} \tilde{\sigma}_m}{\tilde{\mu}_m^{1 / 2}} 
      \varphi_{1-\beta}(n)^{1 / 2}
\\
  &\qquad 
    +\mfrac{4 \tilde{L}_m}{n \tilde{\mu}_m^{1 / 2}}\Big(\delta_0+\mfrac{\tilde{\sigma}_m^2}{\tilde{L}_m^2}\Big)^{1 / 2} A_2^{1 / 2}
\end{align*}
where $ \tilde{\mu}_m, \tilde{\sigma}_m $ and $ \tilde{L}_m $ are defined in Theorem \ref{thm:cd-sgd}, and
\begin{equation*} 
A_2 \;=\; \msum\limits_{k=1}^{n} e^{\frac{-\tilde{\mu}_mC}{16} k^{1 - \beta} + 16 \tilde{L}_m^4 C^4  \varphi_{1 - 2\beta}(k)}
 \qedhere
\end{equation*}
\end{proof}

\vspace{.5em}

\begin{lemma}\label{lem:quartic-convergence}
  Under Assumptions \ref{asst:SGD-A1}, \ref{asst:SGD-A4}, \ref{asst:SGD-A5},
  the online CD iterates produced by Algorithm \ref{alg:online_cd} using 
  $ \eta_t = C t^{-\beta} $ for $ \beta \in (\mfrac{1}{2}, 1) $ verify
  \begin{equation*} 
  \frac{1}{n}\sum\limits_{i=1}^{n} \sqrt{\cdE*{\cdnorm*{\psi_{i} - {\psi}^{\star}}^4}} =  \mathcal O(n^{-\beta}).
  \end{equation*}
\end{lemma}
\begin{proof}
  We proceed as in the proof of \cite[Theorem 3]{moulines2011non}, first establishing a recurrence for  
  $\cdE[\big]{\cdnorm*{\psi_{n} - {\psi}^{\star}}^4}$, and then unrolling it. We have
  \begin{align*}
  \cdCondE*{\cdnorm*{\psi_n-{\psi}^{\star}}^4 \given \mathcal{F}_{n-1}}
  \leq & 
    \left\|\psi_{n-1}-{\psi}^{\star}\right\|^4
    +6 \eta_n^2\cdnorm*{\psi_{n-1}-{\psi}^{\star}}^2 \cdCondE*{\cdnorm*{h_n\cdpar*{\psi_{n-1}, X_n}}^2  \given \mathcal{F}_{n-1}} \\
  & +\eta_n^4 \cdCondE*{\cdnorm*{h_n \cdpar*{\psi_{n-1}, X_n}}^4 \given \mathcal{F}_{n-1}} \\
  & -4 \eta_n\cdnorm*{\psi_{n-1}-{\psi}^{\star}}^2\cdinner*{
    \psi_{n-1}-{\psi}^{\star}, 
    \cdCondE*{h_n(\psi_{n-1}, X_n) \given \mathcal  F_{n-1}}  
  } \\
  & + 4 \eta_n^3\cdnorm*{\psi_{n-1}-{\psi}^{\star}} \cdCondE*{\cdnorm*{h_n\cdpar*{\psi_{n-1}, X_n}}^3 \given \mathcal{F}_{n-1}} .
  \end{align*}
  The second and fourth terms will be controlled using results from our previous sections.
  For simplicity, we don't attempt to relate the moments of $  \cdnorm*{h_n}^4 $
  as precisely as before. Instead we use the fact that, for all $ \psi \in \Psi $,
  \begin{equation*} 
    \cdE*{\cdnorm*{h_n(\psi, X_n)}^k} \leq 
    2^{k-1}  \cdpar*{
    \cdE*{\|\phi(K^m_{\psi}(X_n))\|^{k}} + 
    \cdE*{\|\phi(X_n) \|^{k}}}\,.
  \end{equation*}
  Let us note $ \tau =  \cdpar*{\sup_{\psi \in \Psi} \cdE*{\cdnorm*{\cramped{\phi(\Xpsi)}}\vphantom{a}^{8}}}^{1 / 8} $. 
  We have, for $ k \leq  4 $, and $ \psi \in \Psi $
  \begin{align*}
  \cdE*{\|\phi(K^m_{\psi}(X_n)) \|^{k}}
  &
  \leq \mathbb{E}\|\phi(\Xpsi) \|^{k} \\ 
  & \quad 
    + C_{\chi}  \cdpar[\Big]{\cdE[\Big]{\cdpar[\Big]{
      \cdnorm[\big]{\phi(\Xpsi)}^{k} - \cdE[\Big]{\cdnorm[\big]{\phi(\Xpsi)}^{k}}
    }^{2}}}^{1 / 2}  \cdnorm[\big]{\psi - {\psi}^{\star}} \\
  &\leq
  \tau^{k} + 2 C_{\chi} \tau^{k} \cdnorm*{\psi - {\psi}^{\star}}\,.
  \end{align*}
  Here, we used the fact that $ \cdP{\psi_{n-1}} $ is a contraction, and 
  $ (\cdE*{\cdnorm*{\phi(\Xpsi)}\vphantom{a}^{k}})^{1 / k}$ is an increasing function of $ k $ for all $ \psi $.
  On the other hand, we simply have
  $\cdE*{\cdnorm*{\phi(X_n)}^{k}}  \leq \tau^{k}$. 
  Plugging this into the previous equation, applying it to $ \psi=\psi_{n-1} $, we obtain
  \begin{align*}
  \cdCondE*{\cdnorm*{\psi_n-{\psi}^{\star}}^4 \ \given \mathcal{F}_{n-1}} \leq & 
  \cdnorm*{\psi_{n-1}-{\psi}^{\star}}^4 \\
    & + 
	6 \eta_n^2\cdnorm*{\psi_{n-1}-{\psi}^{\star}}^2 
	\cdpar*{4 \tau^{2} + 4 C_{\chi} \tau^{2}  \cdnorm*{\psi_{n-1} - {\psi}^{\star}}} \\
    & + \eta_n^4 (16\tau^{4} + 16 C_{\chi} \tau^{4}  \cdnorm*{\psi_{n-1} - {\psi}^{\star}} ) \\
    & - 
      4 \eta_n\cdnorm*{\psi_{n-1}-{\psi}^{\star}}^2
      \cdinner*{\psi_{n-1}-{\psi}^{\star}, \cdCondE*{h_n  \given \mathcal  F_{n-1}}} \\
    & +
      4 \eta_n^3\cdnorm*{\psi_{n-1}-{\psi}^{\star}} 
      \cdpar*{8 \tau^{3} + 8 C_{\chi} \tau^{3}  \cdnorm*{\psi_{n-1} - {\psi}^{\star}}}
  \end{align*}
  To simplify the recursion, we use the four following inequalities:
  \begin{align*}
    \tau^{2}\eta_n^2  \cdnorm*{\psi_{n-1} - {\psi}^{\star}}^3 
      & \leq 
    \mfrac{1}{2} \cdpar*{
	\tau^{2}\eta_{n}^2 (\cdnorm*{\psi_{n-1} - {\psi}^{\star}}^2 
	+ \tau^{2}\eta_{n}^2  \cdnorm*{\psi_{n-1} - {\psi}^{\star}}^4
    } \\
    \tau^{4} \eta_n^4  \cdnorm*{\psi_{n-1} - {\psi}^{\star}} 
      & \leq
    \cdpar*{
      \tau^{4}\eta_{n}^4 
      + \mfrac{1}{4}\tau^{4}\eta_{n}^4 \cdnorm*{\psi_{n-1} - {\psi}^{\star}}^4  
    } \\
    \eta_n^{3} \tau^{3}  \cdnorm*{\psi_{n-1} - {\psi}^{\star}} 
      & \leq 
    \mfrac{1}{2} \cdpar*{
      \eta_n^2 \tau^{2}  \cdnorm*{\psi_{n-1} - {\psi}^{\star}}^2 
      + 16 \eta_n^4 \tau^{4} 
    }  \\
    \tau^{3}\eta_n^3  \cdnorm*{\psi_{n-1} - {\psi}^{\star}}^2 
      & \leq 
    \mfrac{1}{2} \cdpar*{
      \eta_{n}^4 \tau^{4} 
      + \cdnorm*{\psi_{n-1} - {\psi}^{\star}}^2 \eta_{n}^2 \tau^{2}
    } \\
  \end{align*}
  Injecting them in our recursion, we obtain:
  \begin{align*}
  \cdCondE*{\cdnorm*{\psi_n-{\psi}^{\star}}^4 \given \mathcal{F}_{n-1}}
  \leq & \cdnorm*{\psi_{n-1}-{\psi}^{\star}}^4 \\
    & + 12 \eta_n^2\cdnorm*{\psi_{n-1}-{\psi}^{\star}}^2  \tau^2  \\
    & + 12 C_{\chi} \tau^{2} \eta_{n}^2  \cdnorm*{\psi_{n-1} - {\psi}^{\star}}^2  \\
    & + 12 C_{\chi} \tau^{2} \eta_{n}^2  \cdnorm*{\psi_{n-1} - {\psi}^{\star}}^4 \\
    & + 16 \eta_n^4  \tau^{4} \\
    & + 16 C_{\chi} \eta_n^4  \tau^{4} \\
    & + 4 C_{\chi} \eta_{n}^{4} \tau^{4}  \cdnorm*{\psi_{n-1} - {\psi}^{\star}}^4 \\
    & - 4 \eta_n \tilde{\mu}_{m}\cdnorm*{\psi_{n-1}-{\psi}^{\star}}^4 \\
    & + 16 \eta_n^2 \tau^{2} \cdnorm*{\psi_{n-1}-{\psi}^{\star}}^2  \\
    & + 16 \eta_n^4 \tau^{4}   \\
    & + 16 \eta_n^{4} C_{\chi} \tau^{4}  \\
    & + 16 \eta_n^{2} C_{\chi} \tau^{2}  \cdnorm*{\psi_{n-1} - {\psi}^{\star}}^2  \;, \\
  \end{align*}
  which, after further simplifications, yields
  \begin{align*}
  \MoveEqLeft \cdCondE*{\cdnorm*{\psi_n-{\psi}^{\star}}^4 \given \mathcal{F}_{n-1}}
  \\
  & \leq 
     \cdnorm*{\psi_{n-1} - {\psi}^{\star}}^4
      (1 - 4 \eta_n \tilde{\mu}_{m} + 12 C_{\chi} \eta_n^2 \tau^{2}   + 4 C_{\chi} \tau^{4} \eta_{n}^{4}) \\ 
  & 
      + \eta_n^2  \cdnorm*{\psi_{n-1} - {\psi}^{\star}}^2 \cdpar*{28(1 + C_{\chi}) \tau^2} 
      + 32\eta_{n}^4  \cdpar*{\tau^{4}(1 + C_{\chi})} \\
  & \leq
    \cdnorm*{\psi_{n-1} - {\psi}^{\star}}^4 (1 - 4 \eta_n \tilde{\mu}_{m} 
      + 12 (1+C_{\chi}) \eta_n^2 \tau^{2}   
      + 4 (1 + C_{\chi}) \tau^{4} \eta_{n}^{4}
  ) \\ 
  & 
      + 28 \eta_n^2  \cdnorm*{\psi_{n-1} - {\psi}^{\star}}^2  \cdpar*{(1 + C_{\chi})\tau}^{2} 
      + 32\eta_{n}^4 (1 + C_{\chi})\tau)^{4} \\
  & \leq 
     \cdnorm*{\psi_{n-1} - {\psi}^{\star}}^4  \cdpar*{
      1 - 4 \eta_n \tilde{\mu}_{m} + 12 \eta_n^2 ((1+C_{\chi})\tau)^{2} + 4  \cdpar*{(1 + C_{\chi})\tau}^{4} \eta_{n}^{4}  
    } \\ 
  & 
    + 28 \eta_n^2  \cdnorm*{\psi_{n-1} - {\psi}^{\star}}^2  \cdpar*{(1 + C_{\chi})\tau}^{2} 
    + 32\eta_{n}^4 (1 + C_{\chi})\tau)^{4} \\
  & \leq  
    \cdnorm*{\psi_{n-1} - {\psi}^{\star}}^4(1 - 4 \eta_n \tilde{\mu}_{m} 
      + 12 \eta_n^2 (2(1+C_{\chi})\tau)^{2} +  16 \eta_n^2 (2(1+C_{\chi})\tau)^{3} \\
  & 
      + 4 (2(1 + C_{\chi})\tau)^{4} \eta_{n}^{4}) 
      + 20 \eta_n^2  \cdnorm*{\psi_{n-1} - {\psi}^{\star}}^2 \left (2(1 + C_{\chi})\tau \right)^{2} 
      + 16 \eta_{n}^4 (2(1 + C_{\chi})\tau)^{4} \\
  & \leq  
    \cdnorm*{\psi_{n-1} - {\psi}^{\star}}^4 (
      1 - 4 \eta_n \tilde{\mu}_{m} 
      + 12 \eta_n^2 (2(1+C_{\chi})\tau + L)^{2} 
      + 16 \eta_n^2 (2(1+C_{\chi})\tau + L)^{3} \\
  & 
      + 4 (2(1 + C_{\chi})\tau + L)^{4} \eta_{n}^{4}
      ) 
      + 20 \eta_n^2  \cdnorm*{\psi_{n-1} - {\psi}^{\star}}^2  \cdpar*{2(1 + C_{\chi})\tau}^{2} 
      + 16 \eta_{n}^4 (2(1 + C_{\chi})\tau)^{4} \\
  & \leq 
     \cdnorm*{\psi_{n-1} - {\psi}^{\star}}^4(1 - 4 \eta_n \tilde{\mu}_{m} 
       + 12 \eta_n^2 \tilde{L}_1^{2} 
       +  16 \eta_n^2 \tilde{L}_1^{3}  
       + 4 \tilde{L}_1^{4} \eta_{n}^{4}) \\ 
  &
      + 20 \eta_n^2  \cdnorm*{\psi_{n-1} - {\psi}^{\star}}^2\tilde{\tau}_1^{2} 
      + 16 \eta_{n}^4 \tilde{\tau}_1^{4} \\
  \end{align*}
  where we defined 
  $ \tilde{\tau}_1 \coloneqq 2 (1 + C_{\chi}) \tau $ 
  and 
  $  \tilde{L}_1 \coloneqq  2(1 + C_{\chi})\tau + L$.
  This recursion is of the form of the one studied in \cite[Equation 32]{moulines2011non} (note that by design, $ \tilde{L}_1 \geq  \tilde{\mu}_{m} $.)
  The steps performed to bound $ \cdE*{\cdnorm*{\psi_n - {\psi}^{\star}}^4}  $ thus follow from their derivations, and we obtain:

  \begin{align*}
  &\mfrac{1}{n} \sqrt {\msum\limits_{i=1}^{n} \cdE*{\cdnorm*{\psi_{i-1} - {\psi}^{\star}}^4}}
  \\
  \;\leq\; 
  & 
    \mfrac{C \tilde{\tau}_1^2}{2 n}
    \Big(C^{1 / 2} \varphi_{1 - 3 \beta/2}(n)+\tilde{\mu}_m^{-1/2} \varphi_{1-\beta}(n)\Big) \\
  &
    +
      \mfrac{\sqrt{20} C^{1 / 2} \tilde{\tau}_1}{2 n} A_1 \exp \cdpar*{24 \tilde{L}_1^4 C^4}
      \cdpar[\Big]{
	\delta_0
	+\mfrac{\tilde{\mu}_m \cdE*{\cdnorm*{\psi_0 - {\psi}^{\star}}\smallpow{4}}}{20 C \tilde{\tau}_1^2}
      +2 \tilde{\tau}_1^2 C^3 \tilde{\mu}_m+8 \tilde{\tau}_1^2 C^2
     }^{1 / 2} \\
  \;=\; 
  & \mathcal O \left (n^{-\beta}\right),
  \end{align*}
  where 
  \begin{equation*} 
  A_1 = \msum\limits_{k=1}^{n} e^{\frac{-\tilde{\mu}_mC}{16} k^{1 - \beta} + 16 \tilde{L}_1^4 C^4  \varphi_{1 - 2\beta}(k)}
  \end{equation*}
  and we have $ A_1<  +\infty $ if $ \beta < 1 $, and $ A_1 = O(n) $ otherwise.
\end{proof}

\begin{lemma}
For all $\psi \in \Psi$, we have:
\begin{multline}
  \cdnorm*{\cdCov*{\phi(K^m_{\psi}(X_n))}  - \cdCov*{\phi(\Xpsi)}}_{\mathrm F} \\ 
  \leq 
    \alpha^m C_\chi (\bar{\tau}^{1/2} + 2 \| \log Z\|_{4, \infty} \sigma)  \cdnorm*{\psi - \psi^\star} 
    + \alpha^{2m} C_\chi^2 \sigma^2  \cdnorm*{\psi - {\psi}^{\star}}^2.
\end{multline}
where 
$ 
\bar{\tau} \coloneqq 
\sup_{\psi \in \Psi} \cdE*{
  \cdnorm*{\phi(\Xpsi) \phi^{\top}(\Xpsi) - \mathbb{E}\lbrack \phi(\Xpsi) \phi^{\top}(\Xpsi) \rbrack}_{F}^{2}  
} < +\infty$ 
and  
$  \cdnorm*{\log Z}_{4, \infty} \coloneqq \sup_{\psi \in \Psi} 
\| \cdE*{\phi(\Xpsi)} \| \leq \sup_{\psi \in \Psi} \msum_{i=1}^{d} \partial_i^2 \log Z(\psi)^2
  $.
\end{lemma}
\begin{proof}
We have
\begin{equation*} 
  \cdCov*{\phi(K^m_{\psi}(X_n))}  = 
  \cdE*{\cdP{\psi}   \left (\phi(X_n) \phi^{\top}(X_n) \right )} - 
  \cdpar*{\cdE*{\cdP{\psi}\phi(X_n)}} \cdpar*{\cdE*{\cdP{\psi} \phi(X_n)}^{\top}} \\
\end{equation*}
Looking at the second moment first, we have
\begin{equation*} 
  \cdE*{\cdPpsi \phi(X_n) \phi(X_n)^{\top}} - \cdE*{\phi(\Xpsi) \phi(\Xpsi)^{\top}} = 
  \underbrace{
    \cdE*{\cdPpsi(\phi \phi^{\top} - \cdE*{\phi(\Xpsi) \phi(\Xpsi)^{\top}})(X_n)}
  }_{\Delta_1}.
\end{equation*}
Applying Lemma \ref{lemma:approx-contraction} to the $\mathbb R^{d^2}$-valued function $f$ given by  
$ f_{ij} \coloneqq \phi_i \phi_j  - \cdE*{\phi_i(\Xpsi) \phi_j(\Xpsi)}$, 
\begin{equation*} 
  \cdnorm*{\Delta_1}_{\mathrm F} 
  \leq  
    \cdnorm*{\psi - {\psi}^{\star}}\alpha^{m} C_{\chi} 
    \sqrt{
      \cdE*{
        \cdnorm*{
          \cdpar[\big]{\phi \phi^{\top} - \cdE*{(\phi \phi^\top)(\Xpsi)}}(\Xpsi)  
        }^2_{F}
      }
    }
  \leq
  \bar{\tau}^{1/2} \alpha^{m} C_{\chi}  \cdnorm*{\psi - {\psi}^{\star}}.
\end{equation*}
We now investigate the first moment. We have
\begin{align*}
  &\cdE*{\cdPpsi \phi(X_n)}  
    = 
      \underbrace{\cdE*{\cdPpsi \phi(X_n)} - \cdE*{\phi(\Xpsi)}}\limits_{\Delta_{2,1}} 
      \;+\; \mathbb{E}  \big \lbrack \phi(\Xpsi) \big \rbrack \\
\implies 
  & \underbrace{
    \cdpar*{\cdE*{\cdPpsi \phi(X_n)}}  \cdpar*{\cdE*{\cdPpsi \phi(X_n)}}^{\top} 
    - \cdE*{\phi(\Xpsi)} \cdE*{\phi(\Xpsi)^{\top}}  
  }\limits_{\Delta_{2}}
    = 
    \Delta_{2, 1} \Delta_{2, 1}^T \\ 
  & \quad 
    + \Delta_{2,1} \cdE*{\phi(\Xpsi)^{\top}} 
    + \cdE*{\phi(\Xpsi)} \Delta_{2,1}^{\top}\,.
\end{align*}
Thus, applying Lemma \ref{lemma:approx-contraction} on $ \Delta_{2, 1} $, we have:
\begin{align*}
  \cdnorm*{\Delta_{2}}_{\mathrm F}  
  & \leq 
    \cdnorm*{
      \Delta_{2, 1} \Delta_{2, 1}^{\top} 
      + \Delta_{2,1} \cdE*{\phi(\Xpsi)^{\top}} 
      + \cdE*{\phi(\Xpsi_n)} \Delta_{2, 1}^{\top} 
    }_{\mathrm F} \\
  & \leq  
    \cdnorm*{\Delta_{2, 1} \Delta_{2, 1}^{\top}}_{\mathrm F} 
    + 2  \cdnorm*{\Delta_{2, 1}}  \cdnorm*{\cdE*{\phi(\Xpsi)}} \\
  & \leq
    \alpha^{2m} \sigma^2 C_\chi^2  \cdnorm*{\psi - {\psi}^{\star}}^2 
    + 2  \cdnorm*{\log Z}_{4, \infty} \alpha^{m} C_{\chi} \sigma  \cdnorm*{\psi - {\psi}^{\star}},
\end{align*}
where we used that $ \| \Delta_{2, 1} \Delta_{2, 1}^{\top} \|_{\mathrm F} = \| \Delta_{2, 1} \|^2$.
We can now combine our two matrix moment bounds to obtain
\begin{align*}
 \cdnorm*{\mathrm{Cov}  \left\lbrack \phi(K^m_{\psi}(X_n))  \right\rbrack  - \mathrm{Cov} \big \lbrack \phi(\Xpsi) \big \rbrack}_{\mathrm F} = \left \| \Delta_1 + \Delta_2 \right \|_{\mathrm F} & \\ 
				    & \hspace{-25em} \leq  \alpha^m C_\chi (\bar{\tau}^{1/2} + 2 \| \log Z\|_{4, \infty} \sigma) \| \psi - \psi^\star \| + \alpha^{2m}C_\chi^2 \sigma^2 \left \| \psi - {\psi}^{\star} \right \|^2.
\end{align*}

\end{proof}

\begin{lemma}\label{lem:markov-kernel-variance-averaged-cd}
  Under Assumptions \ref{asst:SGD-A1}, \ref{asst:SGD-A4}, \ref{asst:SGD-A5} it holds that
  \begin{multline}
    \sqrt{
      \cdE*{
	\cdnorm*{\mfrac{1}{n} \msum\limits_{i=1}^{n} 
	  f''({\psi}^{\star})^{-1} (\overline{h}(\psi_{i-1}) - h_i(\psi_{i-1}))}^2
      }
    }   
    \,\leq \,
    2 \sqrt{\mfrac{\mathrm{tr}(\mathcal  I({\psi}^{\star})^{-1})}{n}}  \\
      + \mfrac{\cdnorm*{\mathcal I(\psi^\star)^{-2}}^{1/2}_{\mathrm F}}{n} \left (
	  (M + \alpha^m C_\chi (\bar{\tau}^{1/2} + 2 \| \log Z \|_{4, \infty} \sigma))^{1/2}
	  \cdpar*{\msum\limits_{i=1}^{n} \delta_{i-1}^{1 / 2}}^{1 / 2}  \right  .\\ 
      + \left . 
	\alpha^{2m} C_\chi^2 \sigma^2  \cdpar*{\msum\limits_{i=1}^{n} \delta_{i-1}}^{1 / 2} \right ) \\
  \end{multline}
  where 
  $
  M \coloneqq \sup_{\psi \in \Psi}
    \cdnorm*{\nabla_{\psi}^3 \log Z(\psi)}_{\mathrm{op}(\cdnorm*{\cdot}_{F},  \cdnorm*{\cdot}_F)}
  < +\infty
  $.
\end{lemma}
\begin{proof}
Let us first note that
\begin{equation*} 
\begin{aligned}
  \MoveEqLeft[0.2] f''({\psi}^{\star})^{-1}(\overline{h}(\psi_{n-1}) - h_n(\psi_{n-1})) \\ 
  & = 
    f''({\psi}^{\star})^{-1} \underbrace{(\phi(X_n) - \cdE*{\phi(X_1)})}\limits_{\Delta_{1, n}}  
    + {f}''({\psi}^{\star})^{-1} \underbrace{
      \cdpar*{\phi(K^m_{\psi_{n-1}}(X_n)) - \cdCondE*{\phi(K^m_{\psi_{n-1}}(X_n)) \given \mathcal  F_{n-1}}}
    }_{\Delta_{2, n}} \\
  & =
    f''({\psi}^{\star})^{-1}\Delta_{1, n} + f''({\psi}^{\star})^{-1}\Delta_{2, n}\,.
\end{aligned}
\end{equation*}
Noting $\Delta$ the l.h.s of Lemma \ref{lem:markov-kernel-variance-averaged-cd}, we have, summing over $[n]$, and using Minkowski's inequality,
\begin{equation*}
  \Delta 
  \;\leq\; 
    \mfrac{1}{n} \sqrt{\cdE*{\cdnorm*{\msum\limits_{i=1}^{n}   f''({\psi}^{\star})^{-1}\Delta_{1, i}}^2}}
  + 
    \mfrac{1}{n} \sqrt{\cdE*{\cdnorm*{\msum\limits_{i=1}^{n}   f''({\psi}^{\star})^{-1}\Delta_{2, i}}^2}}
\end{equation*}
Note that this step was made possible because we are looking at the square-root of the variance, 
which is unlike the recursion in Lemma \ref{lem:online-cd-recursion}. This allows to separate the terms and 
use fewer intermediaries than in the proof of Lemma \ref{lem:online-cd-recursion}.

Since both $ \Delta_{1, n} $ and $ \Delta_{2, n} $ are martingale differences with respect to the 
filtration $ \mathcal F_{n-1} $, the covariance terms vanish, and we have
\begin{align*}
  \Delta 
  & \leq 
  \mfrac{1}{n} \sqrt{\msum\limits_{i=1}^{n}  \cdE*{\cdnorm*{f''({\psi}^{\star})^{-1}\Delta_{1, i}}^2}}
  + 
  \mfrac{1}{n} \sqrt{\msum\limits_{i=1}^{n} \cdE*{\cdnorm*{f''({\psi}^{\star})^{-1} \Delta_{2, i}}^2}}
  \\
  & \leq 
  \mfrac{1}{n} \sqrt{
    \msum\limits_{i=1}^{n}
      \mathrm{tr}(f''({\psi}^{\star})^{-1} \cdE*{\Delta_{1, i} \Delta_{1, i}^{\top} f''({\psi}^{\star})^{-1} )} 
  } \\ 
  &  \quad 
  +\, \mfrac{1}{n} \sqrt{
    \msum\limits_{i=1}^{n}  \mathrm{tr}(f''({\psi}^{\star})^{-1}\cdE*{\Delta_{2, i} \Delta_{2, i}^{\top} f''({\psi}^{\star})^{-1})} 
  }  \\
  & \leq
  \sqrt{\mfrac{\mathrm{tr}(\mathcal  I({\psi}^{\star})^{-1})}{n}}
  + \mfrac{1}{n} \sqrt{
    \msum\limits_{i=1}^{n}
      \mathrm{tr}(f''({\psi}^{\star})^{-1} \cdE*{
	\cdCondCov*{\phi(K^m_{\psi_{i-1}}(X_i)) \given \mathcal  F_{i-1}}   
      } f''({\psi}^{\star})^{-1}) 
  } \\
  & \leq
  \sqrt{\mfrac{\mathrm{tr}(\mathcal  I({\psi}^{\star})^{-1})}{{n}}}
  + \mfrac{1}{n} \Big (
      \msum\limits_{i=1}^{n}  \mathrm{tr}(
	f''({\psi}^{\star})^{-1} 
	\left (
	  \mathbb{E} \left \lbrack
	    \cdCov*{\phi(X_i)} 
	    + 
	      \left (\cdCov*{\phi(X^{\psi_{i-1}}_i) \, \lvert \, \psi_{i-1}} \right . \right . \right .\\ 
  & \quad  \left .  \left . \left .
	      - \cdCov*{\phi(X_i)} \right )  
	    +  \cdpar*{
	      \cdCondCov*{\phi(K^m_{\psi_{i-1}}(X_i)) \given \mathcal  F_{i-1}} 
	      - \cdCondCov*{\phi(X^{\psi_{i-1}}) \given \mathcal  F_{i-1}} 
	   } \right \rbrack  \right )
	f''({\psi}^{\star})^{-1}) \Big )^{1 / 2} \\
  & \leq  
    2\sqrt{\mfrac{\mathrm{tr}(\mathcal  I({\psi}^{\star})^{-1})}{n}}
    + \mfrac{\sqrt{\mathrm{tr}(\mathcal I(\psi^\star)^{-2})}}{n} \Big (\msum\limits_{i=1}^{n}
      \cdE*{\cdnorm*{\cdCondCov*{\phi(X^{\psi_{i-1}}) \given \mathcal F_{i-1}} - \cdCov*{\phi(X_i)}}_{\mathrm F}} \\ 
& \quad 
+ \cdE*{
    \| \cdCondCov*{\phi(K^m_\psi(X_i)) \given \psi_{i-1}} 
    - \cdCondCov*{\phi(X^{\psi_{i-1}}) \given \mathcal F_{i-1}})\|_{\mathrm F}  
  } \Big )^{1/2} \\
& \overset{(a)}{\leq}  
  2\sqrt{\mfrac{\mathrm{tr}(\mathcal  I({\psi}^{\star})^{-1})}{n}}
  + \mfrac{\sqrt{\mathrm{tr}(\mathcal I(\psi^\star)^{-2})}}{n} \left (
    (M + \alpha^m C_\chi (\bar{\tau}^{1/2} + 2 \| \log Z \|_{4, \infty} \sigma)^{1/2}
  ) \times \right .\\ 
& \quad \quad \left .
  \sqrt{
    \msum\limits_{i=1}^{n} 
      \cdE*{\cdnorm*{\psi_{i-1} - \psi^\star}  
    }
  } 
+ \alpha^{m} C_\chi \sigma \sqrt{\msum\limits_{i=1}^{n} \cdE*{\cdnorm*{\psi_{i-1} - \psi^\star}^2}} \right ) \\
& \overset{(b)}{\leq}  
  2\sqrt{\mfrac{\mathrm{tr}(\mathcal  I({\psi}^{\star})^{-1})}{n}}
  + \mfrac{\sqrt{\mathrm{tr}(\mathcal I(\psi^\star)^{-2})}}{n} \left (
    (
      M 
      + 
	\alpha^m C_\chi (\bar{\tau}^{1/2} + 2 \| \log Z \|_{4, \infty} \sigma)
    )^{1/2}
    \sqrt{\msum\limits_{i=1}^{n} \delta_i^{1/2}} \right . \\
& \quad \quad \left .
    + \alpha^{m} C_\chi \sigma \sqrt{\msum\limits_{i=1}^{n} \delta_i} \right ) \\
\end{align*}
In $(a)$ we used Lemma \ref{lem:markov-kernel-variance-averaged-cd},
the cyclicity of the trace, $  \mathrm{tr}(A^{\top}B) \leq  \cdnorm*{AB}_{\mathrm F} \leq  \cdnorm*{A}_{\mathrm{F}}  \cdnorm*{B}_{\mathrm F} $,
and the fact that since $\cdCov*{\phi(X^{\psi_{i-1}})} = \nabla_{\psi}^{2} \mathcal  L(\psi_{i-1}) $, by analycity of $ \mathcal  L $,
there exists a constant 
$ M 
\coloneqq 
\sup_{\psi \in \Psi}
  \cdnorm*{\nabla_{}^3 \log Z(\psi)}_{\mathrm{op}(\cdnorm*{\cdot},  \cdnorm*{\cdot}_F)}
$ such that
\begin{equation*} 
  \cdnorm*{\nabla_{\psi}^2 \mathcal  L(\psi_{i-1}) - \nabla_{\psi}^2 \mathcal  L({\psi}^{\star})}_{\mathrm F} 
  \leq M  \cdnorm*{\psi_{i-1} - {\psi}^{\star}}.
\end{equation*}
In $(b)$ we used Jensen's inequality to get 
$ 
\cdE*{\cdnorm*{\psi_{i-1} - {\psi}^{\star}}}  
\leq \sqrt{\cdE*{\cdnorm*{\psi_{i-1} - {\psi}^{\star}}^2}} 
= \delta_{i-1}^{1 / 2}
$.
\end{proof}

\vspace{1em}

We are now ready to prove Theorem \ref{thm:online-cd-averaging}.
\begin{proof}[Proof of Theorem \ref{thm:online-cd-averaging}]
It holds that:
\begin{align*}
  f''({\psi}^{\star})(\psi_{n-1} - {\psi}^{\star}) 
  &= 
    f'(\psi_{n-1})  - f'({\psi}^{\star})  
    +  \cdpar*{f''({\psi}^{\star})(\psi_{n-1} - {\psi}^{\star}) -  f'(\psi_{n-1}) + f'({\psi}^{\star})} \\
  &= 
    h_n(\psi_{n-1})  - f'({\psi}^{\star})  
    +  \cdpar*{f''({\psi}^{\star})(\psi_{n-1} - {\psi}^{\star}) -  f'(\psi_{n-1}) + f'({\psi}^{\star})} \\ 
  & \quad
    + (f'(\psi_{n-1})  - \bar{h}(\psi_{n-1})) 
    + (\overline{h}(\psi_{n-1}) - h_n(\psi_{n-1})).\\
\end{align*}
Applying on both sides: (a) a summation over $i \in [n]$, (b) a multiplication by $f''(\psi^\star)^{-1}$, (c) $\sqrt{\mathbb E \lbrack \| \cdot \|^2 \rbrack}$, and using Minkowski's inequality on the r.h.s, we obtain
\begin{align*}
  \sqrt {\cdE*{\cdnorm*{\mfrac{1}{n}\msum\limits_{i=1}^{n} \psi_{i} - {\psi}^{\star}}^2}}
  & \leq 
    \underbrace{\sqrt{\cdE*{\cdnorm*{\mfrac{1}{n}\msum\limits_{i=1}^{n} f''({\psi}^{\star})^{-1} h_i(\psi_{i-1})}^2}}}\limits_{(i)}  \\
  & \hspace{-10em}
  + \underbrace{
      \sqrt{\cdE*{\cdnorm*{
	\mfrac{1}{n} \msum\limits_{i=1}^{n}
	  f''({\psi}^{\star})^{-1}(f''({\psi}^{\star})(\psi_{i-1} - {\psi}^{\star}) -  f'(\psi_{i-1})
      }^2}}
    }\limits_{(i i)}  \\
 &  \hspace{-10em} 
  + \underbrace{
      \sqrt{\cdE*{\cdnorm*{
	\mfrac{1}{n} \msum\limits_{i=1}^{n} 
	  f''({\psi}^{\star})^{-1} (f'(\psi_{i-1})  - \bar{h}(\psi_{i-1}))
      }^2}}
    }\limits_{(i i i)} \\
 & \hspace{-10em} 
 + \underbrace{
      \sqrt{\cdE*{\cdnorm*{
	\mfrac{1}{n} \msum\limits_{i=1}^{n}
	  f''({\psi}^{\star})^{-1} (\overline{h}(\psi_{i-1}) - h_i(\psi_{i-1}))  
      }^2}}
    }\limits_{(i v)}.\\
\end{align*}
$ (i) $ and  $ (ii) $ have direct analogues in the proofs of prior work \cite{moulines2011non} on the convergence of (unbiased) SGD with Polyak-Ruppert averaging, and will be bounded similarly. $ (i i i) $ captures the bias of the CD algorithm, while $ (i v) $ captures the variance.
\paragraph{Bounding $(i)$} 
Using Lemma \ref{lem:h_n_psi-convergence}, we have
$(i) = \mathcal O  \cdpar*{n^{\max  \cdpar*{\frac{\beta}{2} - 1, - \frac{2 \beta}{2 \beta + 1}}}}$.

\paragraph{Bounding $ (ii) $} 
%This term also appears in the bound of \cite[Theorem 3]{moulines2011non}, and we bound it similarly.
Since $\log Z(\psi)$ is analytic, there exists some constant $M'$ such that
$$  \cdnorm*{f''(\psi^\star)(\psi_{i-1} - \psi^\star) - f'(\psi_{i-1})} \leq M'  \cdnorm*{\psi_{i-1} - \psi^\star}^2.$$
Thus, we have:
\begin{equation*} 
  (ii) 
  \leq \frac{M'}{n} \sqrt{\cdE*{\cdpar*{\msum_{i=1}^{n}  \|\psi_{i} - \psi^\star\|^2}^2}}
  \leq \frac{M'}{n} \msum_{i=1}^{n} \sqrt{\cdE*{\cdnorm*{\psi_{i} - \psi^\star}^4}} 
  = \mathcal O(n^{-\beta})
%= \mathcal O(n^{-\beta})
\end{equation*} 
where the second-to-last inequality used Minkowski's inequality, and the last applied Lemma \ref{lem:quartic-convergence}.

\paragraph{Bounding $(iii)$} 
By Minkowski's inequality, we have:
\begin{equation*} 
  (iii) 
  \;\leq\; 
  \mfrac{1}{n} \msum\limits_{i=1}^{n}  \sqrt{\cdE*{\cdnorm*{f'(\psi_{n-1})  - \bar{h}(\psi_{n-1})}^{2}}}
\end{equation*}
Moreover, using lemma \ref{lemma:approx-contraction}, we have:
\begin{align*}
	\MoveEqLeft  \cdnorm*{f'(\psi_{n-1})  - \bar{h}(\psi_{n-1})} \\ 
	&=  \cdnorm*{\cdCondE*{
	  \phi(X^{\psi_{n-1}}) \given \mathcal F_{n-1}} 
	  - \cdCondE*{\cdP{\psi_{n-1}}\phi(X_n) \given \mathcal F_{n-1}}
	}\\
	& \leq 
	  \alpha^{m}\sqrt{\cdCondE*{
	    \cdnorm*{\phi(X^{\psi_{n-1}}) - \cdCondE*{\phi(X^{\psi_{n-1}}) \given \mathcal F_{n-1}}}^2 
	    \given \mathcal  F_{n-1}  
	  }}
	&\leq \alpha^{m} \sigma C_{\chi}  \cdnorm*{\psi_{n-1} - {\psi}^{\star}}.
\end{align*}
We thus obtain
\begin{equation*} 
  (i i i)
  \;\leq\; 
  \mfrac{\alpha^{m} C_{\chi}}{n} \msum\limits_{i=1}^{n}
    \cdpar*{\cdE*{\cdnorm*{\psi_{i-1} - {\psi}^{\star}}\vphantom{a}^{2}}}^{1 / 2} 
  \;=\; 
  \mfrac{\alpha^{m} C_{\chi}}{n} \msum\limits_{i=1}^{n} \delta_{i-1}^{1 / 2}.
\end{equation*}
Recalling that $ \delta_i$ satisfies Theorem \ref{thm:cd-sgd}, we have that $ \delta_n^{\frac{1}{2}} = \mathcal O(n^{-\beta/2}) $, and we thus have $ \sum_{i=1}^{n} \delta_{i}^{1 / 2} = \mathcal O(n^{1 - \frac{\beta}{2}})$.
By squaring the result of Lemma \ref{lem:square-root-delta-sum-convergence}, we have $ \sum_{i=1}^{n} \delta_{i}^{1 / 2} = \mathcal O(n^{-1 - \frac{\beta}{2}})$, and thus,
we obtain that 
$
(i i i)
= \mathcal O(\alpha^{m} n^{-\beta / 2})
= \mathcal O(
    n^{-(\frac{\beta}{2} + m\frac{\lvert   \log \alpha \rvert}{\log n})}
  )
$.

\paragraph{Bounding $(iv)$} 
We have
\begin{align*}
&(iv) \\
&\,\overset{(a)}\leq \,
  2 \sqrt{\mfrac{\mathrm{tr} \cdpar*{\mathcal  I({\psi}^{\star})^{-1}}}{n}}
  + \mfrac{\cdnorm*{\mathcal I(\psi^\star)^{-2}}^{1/2}_{\mathrm F}}{n} 
    \left (
      (M + \alpha^m C_\chi (\bar{\tau} + 2 \| \log Z \|_{4, \infty} \sigma))^{1/2}
      \sqrt{\msum\limits_{i=1}^{n} \delta_{i-1}^{1 / 2}}
    \right .\\
&\quad \quad  \left . 
      + \alpha^{m} C_\chi \sigma \sqrt{\msum\limits_{i=1}^{n} \delta_{i-1}}
    \right ) \\
& \overset{(b)}\leq 
  2 \sqrt{\mfrac{\mathrm{tr} \cdpar*{\mathcal  I({\psi}^{\star})^{-1}}}{n}} 
  + \mathcal O(n^{-\frac{1}{2} - \frac{\beta}{4}}).
\end{align*}
Where in $(a)$, we used Lemma \ref{lem:markov-kernel-variance-averaged-cd} and in $(b)$, we used Lemma \ref{lem:square-root-delta-sum-convergence} and \ref{lem:delta-sum-convergence}.
%, noting that $ -\frac{1}{2} - \frac{\beta}{2} < \frac{1}{2} - \frac{\beta}{4}$.

\paragraph{Final bound}
Putting everything together, we have that:
\begin{equation*} 
  \sqrt{\cdE*{\cdnorm*{\overline{\psi}_n - {\psi}^{\star}}^{2}}} 
  \leq 
    2\sqrt{\frac{\mathrm{tr}(\mathcal  I({\psi}^{\star})^{-1})}{n}}
    + \mathcal O \left (
      n^{\max_{} \cdpar*{
	- \cdpar*{\frac{1}{2} + \frac{\beta}{4}}, 
	-\beta,
	\frac{\beta}{2} - 1,
	- \frac{2\beta}{2 \beta + 1},
	-  \cdpar*{\frac{\beta}{2} + m\frac{\cdabs*{\log \alpha}}{\log n}}
       }}
    \right )\,.
\end{equation*}
If, furthermore, $m > \frac{(1 - \beta) \log n}{2|\log \alpha|}$, we have
\begin{equation*} 
 \max_{} \cdpar*{
    - \cdpar*{\frac{1}{2} + \frac{\beta}{4}}, 
    -\beta,
    \frac{\beta}{2} - 1,
    - \frac{2 \beta}{2 \beta + 1},
    - \cdpar*{\frac{\beta}{2} + m\frac{\cdabs*{\log \alpha}}{\log n}} 
  } < -\frac{1}{2},
\end{equation*}
which concludes the proof.
\end{proof}

\section{\texorpdfstring{$L_2$}{L₂} approximation by auxiliary gradient updates}
\label{appendix:L2:approx}

In this section, we consider different gradient update schemes starting from some random initialization $\theta_{\rm init}$, and control the $L_2$ distance between the different updates and the deterministic target $\psi^* \in \Psi$. 

\paragraph{Notation.}
Let $\theta^{\rm init}$ be some $\Psi$-valued random initialization that is possibly correlated with $X_1, \ldots, X_n$. We capture the effect of correlation through the following quantities: For $\epsilon > 0$ and $\nu > 2$, let
\begin{align*}
    \vartheta^{\rm init}_{n,m} (\epsilon) 
    \;\coloneqq&\;
    \P\Bigg(  \mfrac{ \Big\| \sum_{i=1}^n \mean\big[ \phi\big(K^m_{\theta^{\rm init}}(X_i) \big) \,\big|\, X_i, \theta^{\rm init} \big]  
    - 
    \mean\big[ \phi\big(K^m_{\theta_{\rm init}}(X'_i)\big) \,\big|\, \theta^{\rm init}
    \big] \Big\| }{n}
    \,>\, \epsilon
    \Bigg) 
    \;,
    \\
    \text{and }
    \quad 
    \varepsilon^{\rm init}_{n,m;\nu}(\epsilon)
    \;\coloneqq&\;
    \sqrt{\epsilon^2 
    + 
    \kappa_{\nu;m}^2
    \big(\vartheta^{\rm init}_{n,m} (\epsilon) \big)^{\frac{2}{\nu-2}}}
    \;.
\end{align*}
For notational clarity, we shall use $\theta_{m,B}$ to denote parameters arising from an one-step update, where the subscripts $m, B$ represent performing the one-step update with length-$m$ Markov chains and with batch size $B$. This is to be distinguished from $\psi_t$ elsewhere in the text, which denotes the parameter from the actual multi-step CD algorithm and the subscript $t$ denotes the $t$-th CD iterate.

\vspace{.5em}

\paragraph{Gradient update schemes.} We consider five different updates. Let $X'_1$ be an i.i.d.~copy of $X_1$ drawn independently of all other random variables. The SGD-with-replacement update is given by 
\begin{align*}
    \theta^{\rm SGDw}_{m,B}
    \;\coloneqq\;
    F^{\rm SGDw}_{m,B} ( \theta^{\rm init} )\;,
    \quad \text{ where }
    F^{\rm SGDw}_{m,B}(\psi)
    \;\coloneqq\;
    \psi 
    -
    \mfrac{\eta}{B} 
    \msum_{i \in S^w} 
    \Big(   \phi(X_i) - \phi\big( K^m_{i;\psi}(X_i) \big)  \Big)
\end{align*}
and $S^w$ is a uniformly drawn size-$B$ subset of $[n]$. The SGD-without-replacement update, after renormalizing the learning rate, is given by the $N$-fold function composition
\begin{align*}
    &
    \;\theta^{\rm SGDo}_{m,B}
    \;\coloneqq\;
    F^{\rm SGDo}_{m,B;N} 
    \circ   
    \ldots 
    \circ 
    F^{\rm SGDo}_{m,B;1}
    ( \theta^{\rm init} )\;,
    \\
    &\; \hspace{3em}
    \text{ where }
    \quad 
    F^{\rm SGDo}_{m,B;j}(\psi)
    \;\coloneqq\;
    \psi 
    -
    \mfrac{\eta}{NB} 
    \msum_{i \in S^o_j} 
    \Big(   \phi(X_i) - \phi\big( K^m_{i;\psi}(X_i) \big)  \Big)
    \;\;
    \text{ for each }
    j \in [N] \;,
\end{align*}
and $S^o_j$'s are disjoint size-$B$ random subsets of $[n]$, defined by $(S^o_1, \ldots, S^o_N) = \pi([n])$ for a uniformly drawn element $\pi$ of the permutation group on $n$ objects. The full-batch gradient update is given by 
\begin{align*}
    \theta^{\rm GD}_m 
    \;\coloneqq\; 
    F^{\rm GD}_m ( \theta^{\rm init} )\;,
    \quad 
    \text{ where }
    \quad 
    F^{\rm GD}_m(\psi)
    \;\coloneqq\;
    \psi 
    -
    \mfrac{\eta}{n} 
    \msum_{i \leq n} 
    \Big(   \phi(X_i) - \phi\big( K^m_{i;\psi}(X_i) \big)  \Big)
    \;.
\end{align*}
The full-batch gradient update with an infinite-length Markov chain is given by
\begin{align*}
    \theta^{\rm GD}_\infty
    \;\coloneqq\; 
    F^{\rm GD}_\infty ( \theta^{\rm init} )\;,
    \qquad 
    \text{ where }
    \qquad
    F^{\rm GD}_\infty(\psi) \;\coloneqq\; \psi 
    -
    \mfrac{\eta}{n} 
    \msum_{i \leq n} 
    \Big(   \phi(X_i) - \phi( X^\psi_1)  \Big)
    \;.
\end{align*}
The population gradient update with an infinite-length Markov chain is given by
\begin{align*}
    \theta^{\rm pop}
    \;\coloneqq\; 
    f^{\rm pop} ( \theta^{\rm init} )\;,
    \qquad 
    \text{ where }
    \quad
    f^{\rm pop} (\psi)
    \;\coloneqq\;
    \psi 
    -
    \eta \, 
    \mean \big[  \phi(X_1) - \phi(X^\psi_1) \big]
    \;,
\end{align*}
where we use the lowercase $f$ to emphasize that $f^{\rm pop}$ is a deterministic function. 

The forthcoming results are summarized below:

\begin{figure}[h]
    \centering
    \begin{tikzpicture}
        \node[inner sep=0pt] at (0,0) {
            $
                \theta^{\rm SGDw}_{m,B}
                \;\overset{\rm \cref{lem:SGDw:GD}}{\approx}\;
                \theta^{\rm GD}_{m} 
                \;\overset{\rm \cref{lem:GD:sample:trunc:error}}{\approx}\; 
                \theta^{\rm GD}_{\infty}
                \;\overset{\rm \cref{lem:GD:sample:pop:error}}{\approx}\; 
                \theta^{\rm pop}
                \;\overset{\rm \cref{lem:GD:pop:opt:error}}{\approx}\; 
                \psi^*
                \;.
            $
        };
        \node[inner sep=0pt] at (0,-0.5){};
    \end{tikzpicture}
\end{figure}

\vspace{.5em}

\begin{lemma} \label{lem:SGDw:GD} Let $\cF_n$ be the sigma algebra generated by $\{ X_i, K^m_{i; \theta^{\rm init}}(X_i) \,|\, 1 \leq i \leq n \}$. Then 
\begin{align*}
    \mean\big[ \theta^{\rm SGDw}_{m,B} - \theta^{\rm GD}_{m} \,\big|\, \theta^{\rm init}, \cF_n \big] 
    \;=&\; 0 \qquad \text{ almost surely}
    \;.
\end{align*}
Moreover, under \ref{asst:SGD-A1} and \ref{asst:Markov:noise:X1}, we have 
\begin{align*}
    \mean \big\|  \theta^{\rm SGDw}_{m,B} - \theta^{\rm GD}_{m}  \big\|^2
    \;\leq&\;
    \mfrac{4 \eta^2 ( \sigma^2 + \kappa_{\nu;m}^2 )}{B} \, \ind_{\{ B < n\}}
    \;.
\end{align*}
\end{lemma}

\begin{proof}[Proof of \cref{lem:SGDw:GD}] Write $A \coloneqq (A_1, \ldots, A_n)$, where 
\begin{align*}
    A_i 
    \;\coloneqq\; 
    \big( \phi(X_i) - \phi\big( K^m_{i; \theta^{\rm init}}(X_i) \big) \big)
    -
    \mean \big[ \phi(X_1) - \phi\big( K^m_{i; \theta^{\rm init}}(X_1) \big) \big] \;.
\end{align*}
Since $S^w$ is uniformly drawn from all size-$B$ subsets of $[n]$ and independently of all other variables, we have that almost surely
\begin{align*}
    \mean\big[ \theta^{\rm SGDw}_{m,B} - \theta^{\rm GD}_{m} \,\big|\, \theta^{\rm init}, \cF_n \big]
    \;=\;
    \mean\Big[ 
        \mfrac{\eta }{B} \msum_{i \in S^w} A_i
        - 
        \mfrac{\eta }{n} \msum_{i \leq n}  A_i
    \,\Big|\,  A \Big]
    \;=\; 0\;.
\end{align*}
To prove the remaining bound, we note that the above relation implies $\theta^{\rm SGDw}_{m,B} - \theta^{\rm GD}_{m}$ is zero-mean. By the law of total variance, we have 
\begin{align*}
    &\mean \big\|  \theta^{\rm SGDw}_{m,B} - \theta^{\rm GD}_{m}  \big\|^2
    \;=\;
    \Tr \, \Cov \big[ \theta^{\rm SGDw}_{m,B} - \theta^{\rm GD}_{m} \big]
    \;=\;
    \Tr \, \mean \, \Cov \big[ \theta^{\rm SGDw}_{m,B} - \theta^{\rm GD}_{m} \,\big|\, \theta^{\rm init}, \cF_n  \big]
    \\
    &\;=\;
    \eta^2 \, 
    \Tr \, \mean \bigg[ \,
        \mean\Big[  
        \Big(  \mfrac{1 }{B} \msum_{i \in S^w} A_i \Big)\Big(  \mfrac{1 }{B} \msum_{i \in S^w} A_i \Big)^\top \,\Big|\, A \Big]
        -
        \Big( \mfrac{1 }{n} \msum_{i \leq n}  A_i \Big)\Big( \mfrac{1 }{n} \msum_{i \leq n}  A_i \Big)^\top 
        \,
    \bigg]
    \;.
\end{align*}
To compute the covariance, recall that $S^w$ is a uniformly drawn size-$B$ subset of $[n]$ with $B=n/N$. Let $\cP_N([n])$ be the collection of all partitions of $[n]$ into $N$ size-$B$ subsets. We can generate $S^w$ by the following two-step process:
\begin{proplist}
    \item Uniformly draw a partition $P' = (P'_1, \ldots, P'_N)$ from $\cP_N([n])$;
    \item Uniformly sample an index $K$ from $[N]$ and set $S^w = P'_K$.
\end{proplist}
Then we have, almost surely
\begin{align*}
    &\;\mean\Big[  
        \Big(  \mfrac{1 }{B} \msum_{i \in S^w} A_i \Big)\Big(  \mfrac{1 }{B} \msum_{i \in S^w} A_i \Big)^\top \,\Big|\, A \Big]
        -
        \Big( \mfrac{1 }{n} \msum_{i \leq n}  A_i \Big)\Big( \mfrac{1 }{n} \msum_{i \leq n}  A_i \Big)^\top 
    \\
    &\;=\;
    \mfrac{1}{|\cP_N([n])|} \msum_{P' \in \cP_N([n])} 
    \Big( 
        \mfrac{1}{N} \msum_{k \leq N} 
        \mfrac{1}{B^2} \msum_{i,j \in P'_k} A_i A_j^\top
        -
        \mfrac{1}{n^2} \msum_{i, j \leq n}  A_i A_j^\top
        \,
    \Big)
    \\
    &\;=\;
    \mfrac{1}{|\cP_N([n])|} \msum_{P' \in \cP_N([n])} 
    \Big( 
        \mfrac{1}{N B^2} \msum_{k \leq N} \msum_{i,j \in P'_k} A_i A_j^\top
        -
        \mfrac{1}{N^2 B^2} \msum_{k, l \leq N} \msum_{i \in P'_k, j \in P'_l}  A_i A_j^\top
        \,
    \Big)
    \\
    &\;=\;
    \mfrac{1}{|\cP_N([n])|} \msum_{P' \in \cP_N([n])} 
    \Big( 
        \mfrac{N-1}{N^2 B^2} \msum_{k \leq N} \msum_{i,j \in P'_k} A_i A_j^\top
        -
        \mfrac{1}{N^2 B^2} \msum_{k \neq l} \msum_{i \in P'_k, j \in P'_l}  A_i A_j^\top
        \,
    \Big)\,.
\end{align*}
By noting that $A_i$'s are exchangeable, we obtain 
\begin{align*}
    \mean \big\|  \theta^{\rm SGDw}_{m,B} -  & \theta^{\rm GD}_{m}  \big\|^2
    \;=\; 
    \eta^2
    \Tr \, \mean \bigg[ \, 
        \mfrac{N-1}{N B} A_1 A_1^\top + \mfrac{(N-1)(B-1)}{N B} A_1 A_2^\top 
        -
        \mfrac{N-1}{N} A_1 A_2^\top
    \,\bigg]
    \\
    \;=&\;
    \eta^2
    \Tr \, \mean \bigg[ \, 
        \mfrac{N-1}{N B} A_1 A_1^\top
        -
        \mfrac{N-1}{N B} A_1 A_2^\top
    \,\bigg]
    \\
    \;=&\;
    \mfrac{\eta^2 (N-1)}{NB} 
    \big( \mean \| A_1 \|^2 - \mean \langle A_1 , A_2 \rangle \big)
    \\
    \;\overset{(a)}{\leq}&\;
    \mfrac{2 \eta^2}{B} \mean \| A_1 \|^2
    \\
    \;=&\; 
    \mfrac{2 \eta^2}{B} \mean \Big\| 
        \big( \phi(X_1) - \mean[\phi(X_1)] \big) 
        - 
        \Big( \phi\big( K^m_{i; \theta^{\rm init}}(X_1) \big)
                -
              \mean\Big[ \phi\big( K^m_{i; \theta^{\rm init}}(X_1) \big) \Big] 
        \Big)
    \Big\|^2
    \\
    \;\leq&\; 
    \mfrac{4 \eta^2}{B} \Big( \Tr \, \Cov[ \phi(X_1)] + \Tr \, \Cov\Big[ \phi\big( K^m_{i; \theta^{\rm init}}(X_1) \big) \Big]  \Big) 
    \\
    \;\overset{(b)}{\leq}&\; 
    \mfrac{4 \eta^2 ( \sigma^2 + \kappa_{\nu;m}^2 )}{B}
    \;.
\end{align*}
In $(a)$, we have used a Cauchy-Schwarz inequality; in $(b)$, we have used \ref{asst:SGD-A1} and \ref{asst:Markov:noise:X1}. Finally we note that if $B=n$, $\theta^{\rm SGDw}_{m,B} = \theta^{\rm GD}_{m}$ almost surely, which implies the desired bound.
\end{proof}

\vspace{.5em}

\begin{lemma} \label{lem:GD:sample:trunc:error} Denote $\cA_n$ as the sigma algebra generated by $\theta^{\rm init}, X_1, \ldots, X_n$. Under \ref{asst:SGD-A1}, \ref{asst:SGD-A4}, \ref{asst:SGD-A5} and \ref{asst:Markov:noise:X1}, we have that for any $\epsilon > 0$ and $\nu > 2$,
    \begin{align*}
        &\;
        \mean \, \big\| \mean\big[ \theta^{\rm GD}_{m}  -  \theta^{\rm GD}_{\infty} \,\big|\, \cA_n \big]  \big\|^2 
        \;\leq\;
        \eta^2 
            \Big( 
                \alpha^{m} \sigma C_\chi  \, \sqrt{\mean \| \theta^{\rm init} -      \psi^* \|^2 } \, + \, \varepsilon^{\rm init}_{n,m;\nu}(\epsilon)
            \Big)^2
        \;,
        \\
        &\;
        \mean \, \| \theta^{\rm GD}_{m} -  \theta^{\rm GD}_{\infty}   \|^2 
        \;\leq\;
        \eta^2 
        \Big(   
            \Big( 
                \alpha^{m} \sigma C_\chi  \, \sqrt{\mean \| \theta^{\rm init} -      \psi^* \|^2 } \, + \, \varepsilon^{\rm init}_{n,m;\nu}(\epsilon)
            \Big)^2
            +
            \mfrac{\kappa_{\nu;m}^2 + \sigma^2}{n}
        \Big)
        \;.
    \end{align*}
\end{lemma}

\begin{proof}[Proof of \cref{lem:GD:sample:trunc:error}] The main challenge arises from the possible correlation between $\theta^{\rm init}$ and $X_1, \ldots, X_n$. First note that for any $\epsilon > 0$, $\nu > 2$ and a real-valued random variable $Y$, by H\"older's inequality, we have 
\begin{align*}
    \mean[Y^2] 
    \;=&\; 
    \mean\big[ Y^2  \, \ind_{\{ Y \leq \epsilon \}} + Y^2 \, \ind_{\{Y > \epsilon \}} \big]
    \\
    \;\leq&\; \epsilon^2 + \mean[ Y^2  \ind_{\{Y > \epsilon \}}  ]
    \;\leq\;
    \epsilon^2 + (\mean[Y^\nu])^{2/\nu} \P(Y > \epsilon)^{(\nu-2)/\nu} \;.
    \tagaligneq \label{eq:bound:moment:by:prob}
\end{align*}
Also note the useful inequality that for two real-valued random vectors (possibly correlated) $V_1, V_2$, we have 
\begin{align*}
    \mean[ \| V_1+V_2\|^2 ] 
    \;\leq\; 
    \mean[ (\| V_1 \| +\| V_2\|)^2 ] 
    \;\leq&\; 
    \mean \|V_1\|^2 + 2 \sqrt{ (\mean \|V_1\|^2)(\mean \| V_2\|^2)} + \mean \|V_2\|^2
    \\
    \;=&\; 
    \big( \sqrt{\mean \|V_1\|^2} + \sqrt{\mean \|V_2\|^2} \big)^2\;.
    \tagaligneq \label{eq:simple:triangle:ineq}
\end{align*}
Now to control the first quantity of interest, by using a triangle inequality, we have
\begin{align*}
        \mean \, \big\| \mean\big[  \theta^{\rm GD}_{m}  - \theta^{\rm GD}_{\infty} \,\big|\, \cA_n \big]  \big\|^2 
        \;=&\;
        \mean\Big\| \, \mean\Big[ \mfrac{\eta}{n} \msum_{i \leq n} 
        \big( 
            \phi\big( K^m_{i;\theta^{\rm init}}(X_i) \big)
            - 
            \phi\big( X^{\theta^{\rm init}}_i \big)
        \big)  \,\Big|\, \cA_n \Big] \Big\|^2 
        \\
        \;\overset{\eqref{eq:simple:triangle:ineq}}{\leq}&\;
        \eta^2 \, \Big( \sqrt{\mean[\Delta_1^2]} + \sqrt{\mean[\Delta_2^2]} \Big)^2
        \;,
\end{align*}
where
\begin{align*}
        \Delta_1 
        \;\coloneqq&\;
            \Big\| \, \mean\Big[ \mfrac{1}{n} \msum_{i \leq n} 
            \big( 
                \phi\big( K^m_{i;\theta^{\rm init}}(X_i) \big)
                - 
                \mean\big[ \phi\big( K^m_{i;\theta^{\rm init}}(X'_1) \big) \,\big|\, \theta^{\rm init} \big]
            \big)  \,\Big|\, \cA_n \Big] \Big\|
        \\
        \;=&\; 
            \Big\| \,\mfrac{1}{n} \msum_{i \leq n} 
            \big( 
                \mean\big[ 
                    \phi\big( K^m_{i;\theta^{\rm init}}(X_i) \big)
                \,\big|\, \theta^{\rm init}, X_i \big]
                - 
                \mean\big[ \phi\big( K^m_{i;\theta^{\rm init}}(X'_1) \big) \,\big|\, \theta^{\rm init} \big] 
            \big) \Big\|
        \;,
        \\
        \Delta_2 
        \;\coloneqq&\;
            \Big\| \, \mean\Big[ \mfrac{1}{n} \msum_{i \leq n} 
            \big( 
                \mean\big[ \phi\big( K^m_{i;\theta^{\rm init}}(X'_1) \big) \,\big|\, \theta^{\rm init} \big]
                -
                \phi\big( X^{\theta^{\rm init}}_i \big)
            \big)  \,\Big|\, \cA_n \Big] \Big\| 
        \\
        \;=&\;
        \Big\| \, 
                \mean\big[ \phi\big( K^m_{1;\theta^{\rm init}}(X'_1) \big) 
                -
                \phi\big( X^{\theta^{\rm init}}_1 \big)
                \,\big|\, \theta^{\rm init} \big]
        \Big\|
        \;,
\end{align*}    
and $X'_1$ is an i.i.d.~copy of $X_1$ and in particular independent of $\theta^{\rm init}$. $\Delta_1$ is controlled via \eqref{eq:bound:moment:by:prob}:
\begin{align*}
    \mean[\Delta_1^2] 
    \;\leq&\;
    \epsilon^2 
    +
    (\mean[\Delta_1^\nu])^{2/\nu} \, \P(\Delta_1 > \epsilon)^{\nu/(\nu-2)} 
    \\
    \;\overset{(a)}{\leq}&\;
    \epsilon^2 
    +
    \Big( 
        \mean\Big\| \phi(K^m_{1; \theta^{\rm init}}(X_1)) - \mean\Big[  \phi(K^m_{1; \theta^{\rm init}}(X'_1)) \,\Big|\, \theta^{\rm init} \Big] \Big\|^\nu
    \Big)^{2/\nu}
    \\
    &\hspace{1em}
    \times 
    \P\bigg( 
        \,\mfrac{\big\| \sum_{i \leq n} 
        \big( 
            \mean\big[ 
                \phi\big( K^m_{i;\theta^{\rm init}}(X_i) \big)
            \,\big|\, \theta^{\rm init}, X_i \big]
            - 
            \mean\big[ \phi\big( K^m_{i;\theta^{\rm init}}(X'_1) \big) \,\big|\, \theta^{\rm init} \big] 
        \big) \big\|}{n}
        > \epsilon
    \bigg)^{\frac{\nu-2}{\nu}}
    \\
    \;\overset{(b)}{\leq}&\;
    \epsilon^2 
    + 
    \kappa_{\nu;m}^2
    \big(\vartheta^{\rm init}_{n,m} (\epsilon) \big)^{\frac{\nu-2}{\nu}}
    \\
    \;=&\;
    \big( \varepsilon^{\rm init}_{n,m;\nu}(\epsilon) \big)^2
    \;.
    \tagaligneq \label{eq:GD:sample:trunc:correlate}
\end{align*}
In $(a)$, we have plugged in the definition of $\Delta_1$ and applied a Jensen's inequality with respect to the empirical average; in $(b)$, we have  used \ref{asst:Markov:noise:X1} to bound the $\nu$-th moment term as 
\begin{align*}
    &\;
    \Big( 
        \mean\Big\| \phi(K^m_{1; \theta^{\rm init}}(X_1)) - \mean\Big[  \phi(K^m_{1; \theta^{\rm init}}(X'_1)) \,\Big|\, \theta^{\rm init} \Big] \Big\|^\nu
    \Big)^{1/\nu}
    \\
    \;=&\;
    \Big( 
        \mean \, \Big[ \mean\Big[ \Big\| \phi(K^m_{1; \theta^{\rm init}}(X_1)) - \mean\Big[  \phi(K^m_{1; \theta^{\rm init}}(X'_1)) \,\Big|\, \theta^{\rm init} \Big] \Big\|^\nu \,\Big|\, \theta^{\rm init} \Big] 
        \, \Big]
    \Big)^{1/\nu}
    \\
    \;\leq&\;
    \msup_{\psi \in \Psi}
    \Big( 
         \mean\Big[ \Big\| \phi(K^m_{1; \psi}(X_1)) - \mean\big[  \phi(K^m_{1; \psi}(X'_1)) \, \big] \Big\|^\nu 
        \, \Big]
    \Big)^{1/\nu}
    \\
    \;=&\;
    \msup_{\psi \in \Psi}
    \Big( 
        \mean\Big[ \Big\| \phi(K^m_{1; \psi}(X_1)) - \mean\big[  \phi(K^m_{1; \psi}(X_1)) \, \big] \Big\|^\nu 
        \, \Big]
    \Big)^{1/\nu}
    \;\leq\; \kappa_{\nu;m}
\end{align*}
and recalled the definitions of $\vartheta^{\rm init}_{n,m}$ and $\varepsilon^{\rm init}_{n,m;\nu}$. On the other hand,
\begin{align*}
    \mean[ &\Delta^2_2]  
    \;=\;
    \mean\Big\| \,
    \mint_{\R^d}
        \phi(x) 
        (K^m_{1;\theta^{\rm init}} p_{\psi^*})(x)
    dx 
    -
    \mean\big[ \phi(X^{\theta^{\rm init}}_1)  \big| \theta^{\rm init} \big] 
    \Big\|^2
    \\
    \;\overset{(a)}{=}&\;
    \mean\Big\| \,
    \mint_{\R^d}
        (K^m_{1;\theta^{\rm init}} \phi)(x) \, p_{\psi^*}(x)
    dx 
    -
    \mean_{p_{\theta^{\rm init}}}[ \, \phi \, ]
    \Big\|^2
    \\
    \;\overset{(b)}{=}&\; 
    \mean\Big\| 
    \mint_{\R^d}
        \big(K^m_{1;\theta^{\rm init}} \big( \phi -  \mean_{p_{\theta^{\rm init}}}[ \, \phi \, ] \big) \big)(x) \times p_{\psi^*}(x)
    dx 
    \Big\|^2
    \\
    \;\overset{(c)}{=}&\;
    \mean\Big\| 
    \mint_{\R^d}
        \big(K^m_{1;\theta^{\rm init}} \big( \phi -  \mean_{p_{\theta^{\rm init}}}[ \, \phi \, ] \big) \big)(x) \times ( p_{\psi^*}(x) - p_{\theta^{\rm init}}(x) )
    dx 
    \Big\|^2
    \\
    \;\overset{(d)}{=}&\;
    \msum_{l=1}^d
    \mean
        \Big(
        \mint_{\R^d}
            \big(K^m_{t1;\theta^{\rm init}} \big( \phi -  \mean_{p_{\theta^{\rm init}}}[ \, \phi \, ] \big) \big)(x)^\top e_l 
            \times p_{\theta^{\rm init}}(x) \times  \mfrac{ p_{\psi^*}(x) - p_{\theta^{\rm init}}(x) }{p_{\theta^{\rm init}(x)}}
            \,
        dx 
        \Big)^2
    \\
    \;\overset{(e)}{\leq}&\;
    \msum_{l=1}^d 
    \mean \Big[
        \Big(
        \mint_{\R^d}
             \big( \big(K^m_{t1;\theta^{\rm init}} \big( \phi -  \mean_{p_{\theta^{\rm init}}}[ \, \phi \, ] \big) \big)(x)^\top e_l \big)^2 p_{\theta^{\rm init}}(x) 
        dx 
    \\
    &\hspace{6em} \times
        \mint_{\R^d}
        \Big(\mfrac{ p_{\psi^*}(x) - p_{\theta^{\rm init}}(x) }{p_{\theta^{\rm init}(x)}}\Big)^2
        p_{\theta^{\rm init}}(x) 
        dx 
    \Big)^2 \Big]
    \\
    \;\overset{(f)}{\leq}&\;
    \msum_{l=1}^d
    \mean
    \Big( 
        \alpha^{2m} 
        \, 
        \Tr \, \Cov\big[ \phi(X^{\theta^{\rm init}}_1) \,\big|\, \theta^{\rm init} \big]
        \,
        \chi^2( p_{\psi^*}, p_{\theta^{\rm init}})
    \Big)^2
    \\
    \;\overset{(g)}{\leq}&\;
    \alpha^{2m} \sigma^2 C_\chi^2  \, \mean \| \theta^{\rm init} - \psi^* \|^2
    \;.
\end{align*}
In $(a)$, we have used that $(K f)(x) = \int K(x,y) f(y) dy$; in $(b)$, we have used that the Markov operator leaves the constant function invariant; in $(c)$, we used that $K^m_{1;\theta^{\rm init}}$ leaves $p_{\theta^{\rm init}}$ invariant; in $(d)$, we denoted $(e_l)_{l \leq d}$ as the standard basis vectors of $\R^d$ and multiplied and divided by $p_{\theta^{\rm init}}(x)$; in $(e)$, we have used a Cauchy-Schwarz inequality; in $(f)$, we have used the definition of the spectral gap $\alpha$ in \ref{asst:SGD-A5}; in $(g)$, we have used \ref{asst:SGD-A1} and \ref{asst:SGD-A4}. Combining the bounds gives the first inequality that 
\begin{align*}
    \mean \, \big\| \mean\big[  \theta^{\rm GD}_{m}  - & \theta^{\rm GD}_{\infty} \,\big|\, \cA_n \big]  \big\|^2 
    \;\leq\;
    \eta^2
    \Big( \alpha^{m} \sigma C_\chi  \, \sqrt{\mean \| \theta^{\rm init} - \psi^* \|^2 } \, + \, \varepsilon^{\rm init}_{n,m;\nu}(\epsilon)
    \Big)^2\;.
\end{align*}

\vspace{.5em}

Now we handle the second quantity by conditioning on $\theta^{\rm init}$ and perform a bias-variance decomposition:
\begin{align*}
    \mean \, \|  \theta^{\rm GD}_{m}  -  \theta^{\rm GD}_{\infty}   \|^2 
    \;=&\;
    \eta^2 
    \,
    \mean \Big[ \, \mean\Big[ 
    \Big\| \mfrac{1}{n} \msum_{i \leq n} 
    \big( 
        \phi\big( K^m_{i;\theta^{\rm init}}(X_i) \big)
        - 
        \phi( X^{\theta^{\rm init}}_i) 
    \big)
    \Big\|^2 \,\Big|\, \cA_n \Big] \, \Big]
    \\
    \;=&\;
    \eta^2 (  Q_B + Q_V )
    \;,
\end{align*}
where 
\begin{align*}
    Q_B
    \coloneqq&\;
    \mean\Big\| \mean\Big[ \mfrac{1}{n} \msum_{i \leq n} 
    \big( 
        \phi\big( K^m_{i;\theta^{\rm init}}(X_i) \big)
        - 
        \phi( X^{\theta^{\rm init}}_i)
    \big)  \,\Big|\, \cA_n \Big] \Big\|^2 
    \;,
    \\
    Q_V
    \coloneqq&\;
    \mean \Big[
        \Tr  
        \Big( 
        \Cov 
        \Big[ \mfrac{1}{n} \msum_{i \leq n}  \phi\big( K^m_{i;\theta^{\rm init}}(X_i) \big)  
        \Big|  
        \cA_n \Big]
        +
        \Cov 
        \Big[ \mfrac{1}{n} \msum_{i \leq n}  \phi( X^{\theta^{\rm init}}_i) 
        \Big| 
        \theta^{\rm init} \Big]
        \Big)
    \Big] 
    \;.
\end{align*}
Note that the covariance terms separate because $X^{\theta^{\rm init}}_i$ is independent of $K^m_{i;\theta^{\rm init}}(X_i)$ conditioning on $\theta^{\rm init}$. $\eta^2 Q_B$ is exactly the quantity controlled above, so it suffices to bound the variance term $Q_V$. By explicitly computing the second covariance term while noting that $X^{\theta^{\rm init}}_1, \ldots, X^{\theta^{\rm init}}_n$ are conditionally i.i.d.~given $\theta^{\rm init}$ and $K^m_{i;\theta^{\rm init}}(X_i)$'s are conditionally independent across $1 \leq i \leq n$ given $\cA_n$, we have
    \begin{align*}
        Q_V
        \;=&\;
        \mfrac{\sum_{i \leq n} \, \mean[\Tr \, \Cov[ \phi(K^m_{i;\theta^{\rm init}}(X_i)) \,|\, \theta^{\rm init}, X_i ]]}{n^2}
        +
        \mfrac{\mean\big[ \Tr\,\Cov[ \phi( X^{\theta^{\rm init}}_1)  | \theta^{\rm init} ] \big]}{n}
        \\
        \;\overset{(a)}{=}&\;
        \mfrac{\mean[\Tr \, \Cov[ \phi(K^m_{1;\theta^{\rm init}}(X_1)) \,|\, \theta^{\rm init}, X_1 ]]}{n}
        +
        \mfrac{\mean\big[ \Tr\,\Cov[ \phi( X^{\theta^{\rm init}}_1)  | \theta^{\rm init} ] \big]}{n}
        \\
        \;\leq&\;
        \mfrac{\kappa_{\nu;m}^2 + \sigma^2}{n}
        \;,
    \end{align*}
    where we have used \ref{asst:Markov:noise}, \ref{asst:Markov:noise:X1} and \ref{asst:SGD-A1} in the last line. Combining the bounds, we obtain that 
    \begin{align*}
        \mean \, \| \theta^{\rm GD}_{m} -  \theta^{\rm GD}_{\infty}   \|^2 
        \;=&\;
        \eta^2 (  Q_B + Q_V )
        \\
        \;\leq&\;
        \eta^2 
        \Big(   
            \Big( 
                \alpha^{m} \sigma C_\chi  \, \sqrt{\mean \| \theta^{\rm init} -      \psi^* \|^2 } \, + \, \varepsilon^{\rm init}_{n,m;\nu}(\epsilon)
            \Big)^2
            +
            \mfrac{\kappa_{\nu;m}^2 + \sigma^2}{n}
        \Big) 
        \;. 
    \end{align*}
\end{proof}

\vspace{.5em}

\begin{lemma} \label{lem:GD:sample:pop:error} Under \ref{asst:SGD-A1}, 
$
    \mean \, \|  \theta^{\rm GD}_{\infty}  - \theta^{\rm pop} \|^2 
    \,\leq\,
    \frac{4 \eta^2 \sigma^2}{n}\;
$.
\end{lemma}

\begin{proof}[Proof of \cref{lem:GD:sample:pop:error}] Since both $F^{\rm GD}_\infty$ and $f^{\rm pop}$ involve infinite-length Markov chains, the initializations do not matter and we can decouple the stochasticity of $X_i$ and $K_{i;\theta^{\rm init}}$. In particular,
    \begin{align*}
        \mean \, \|  \theta^{\rm GD}_{\infty}  - \theta^{\rm pop} \|^2 
        \;=&\;
        \eta^2
        \,
        \mean \, 
        \Big\| 
            \mfrac{1}{n} 
            \msum_{i \leq n} 
            \big(   \phi(X_i) - \phi\big( X^{\theta^{\rm init}}_i \big) \big)
            -
            \mean\big[ \phi(X_1) - \phi\big( X^{\theta^{\rm init}}_1 \big) \big]
        \Big\|^2
        \\
        \;\leq&\;
        \eta^2 \big( Q'_1 + 2 \sqrt{Q'_1 Q'_2} + Q'_2 \big)\;,
    \end{align*}
    where
    \begin{align*}
        Q'_1
        \;\coloneqq&\; 
        \mean \, 
        \Big\| 
            \mfrac{1}{n} 
            \msum_{i \leq n} 
            \big(   \phi(X_i) - \mean[\phi(X_1)] \big) 
        \Big\|^2
        \;=\;
        \mfrac{\Tr \, \Cov[ \phi(X_1) ]}{n}
        \;=\;
        \mfrac{\Tr \, \nabla^2_{\theta} \log Z(\psi^*)}{n}
        \;\leq\;
        \mfrac{\sigma^2}{n}\;,
        \\
        Q'_2
        \;\coloneqq&\; 
        \mean \, 
        \Big\| 
            \mfrac{1}{n} 
            \msum_{i \leq n} 
            \big(   \phi(X^{\theta^{\rm init}}_i) - \mean\big[\phi\big(X^{\theta^{\rm init}}_1 \big) \big] \big) 
        \Big\|^2
        \\
        \;=&\; 
        \mfrac{ \mean\big[ \Tr \,  \Cov[ \phi(X^{\theta^{\rm init}}_1) | \theta^{\rm init} ] \big] }{n}
        \;=\;
        \mfrac{ \mean\big[ \Tr \,  \nabla^2_{\theta} \log Z(\theta^{\rm init})  \big] }{n}
        \;\leq\; \mfrac{\sigma^2}{n}\;.
    \end{align*}
    In the computations above, we have used the relation $\nabla^2_{\theta} \log Z(\theta) = \Cov_{X \sim p_\theta}[\phi(X)]$ and the assumption $\sup_{ \theta \in \Psi } \mathrm{tr}(\nabla_{ \theta}^2 \log Z(\theta)) = \sigma^2 $ from \ref{asst:SGD-A1}. This implies the desired bound. 
\end{proof}

\vspace{.5em}

\begin{lemma} \label{lem:GD:pop:opt:error} Under \ref{asst:SGD-A1},
    $
        \mean \, \| \theta^{\rm pop}  - \psi^* \|^2  \leq \big( 1 - 2 \mu \eta + L^2 \eta^2 \big) \, \mean\, \| \theta^{\rm init} - \psi^* \|^2 \;.
    $
\end{lemma}

\begin{proof}[Proof of \cref{lem:GD:pop:opt:error}] Recall that 
\begin{align*}
    f^{\rm pop}(\theta') 
    \;=\; 
    \theta 
    -
    \eta \, 
    \mean \big[  \phi(X_1) - \phi(X^{\theta'}_1) \big]
    \;=\;
    \theta 
    - 
    \eta \, 
    \big( \nabla_\psi \log Z(\psi^*) - \nabla_\psi \log Z(\theta') \big)
    \;.
\end{align*}    
By construction, $f^{\rm pop}$ is deterministic and $f^{\rm pop}(\psi^*) = \psi^*$. By plugging in the recursions and expanding the square, we get that
\begin{align*}
    \mean \, \big\| &\theta^{\rm pop}  - \psi^* \big\|^2 
    \;=\;
    \mean \, \big\| f^{\rm pop}( \theta^{\rm init} )  - f^{\rm pop}( \psi^* ) \big\|^2 
    \\
    \;=&\;
    \mean \, \big\| ( \theta^{\rm init}  - \psi^*) - \eta (\nabla_\psi \log Z(\theta^{\rm init} ) - \nabla_\psi \log Z(\psi^*) )\big\|^2 
    \\
    \;=&\;
    \mean \, \big\| \theta^{\rm init}  - \psi^* \big\|^2 
    - 
    2 \eta \,
    \mean\big[ \big\langle \theta^{\rm init}  - \psi^* \,,\, 
    \nabla_\psi \log Z( \theta^{\rm init} ) - \nabla_\psi \log Z(\psi^*) 
    \big\rangle \big] 
    \\
    &\; 
    +
    \eta^2 \,
    \mean \, \big\| \nabla_\psi \log Z( \theta^{\rm init}  ) - \nabla_\psi \log Z(\psi^*)  \big\|^2
    \\
    \;\leq&\;
    \mean \, \big\| \theta^{\rm init} - \psi^* \big\|^2 
    -
    2 \mu \eta  \,
    \mean \, \big\| \theta^{\rm init}  - \psi^* \big\|^2 
    +
    L^2 \eta^2 \, \mean \, \big\| \theta^{\rm init}  - \psi^* \big\|^2 
    \;.
\end{align*}
In the last line, we have recalled $\inf_{ \psi \in \Psi } \lambda_{\min_{  }}(\nabla_{ \psi }^2 \log Z(\psi)) = \mu$ and $\sup_{ \theta \in \Psi } \lambda_{\mathrm{\max_{  }}}(\nabla_{ \psi }^2 \log Z(\psi)) = L$ by \ref{asst:SGD-A1} and applied \cref{lem:replace:gradient:inequality}. Combining the coefficients gives the desired statement.
\end{proof}

\section{Proofs for offline SGD} \label{appendix:proofs:offline}

We prove \cref{thm:SGDw:full} (which directly implies \cref{thm:SGDw}) and \cref{thm:SGDo} in this section. The key ingredient of both proofs is \cref{lem:SGDw:recursion} below, which provides an iterative error bound for the SGD-with-replacement scheme by combining different approximation bounds in \cref{appendix:L2:approx}. Throughout this section, we denote $\delta^{\rm SGDw}_{t,j} \coloneqq \mean \, \big\| \psi^{\rm SGDw}_{t,j}  - \psi^* \big\|^2$.
 
\vspace{.5em}

\begin{lemma} \label{lem:SGDw:recursion} Under \ref{asst:SGD-A1}, \ref{asst:SGD-A4}, \ref{asst:SGD-A5}, \ref{asst:Markov:noise} and \ref{asst:Markov:noise:X1}, we have that for $1 \leq j \leq N-1$,
\begin{align*}
    \sqrt{\delta^{\rm SGDw}_{t,j}}
    \;\leq\;
    \Big(
        1
        -
        \eta_t 
        \Big( 
            \mu 
            - 
            \alpha^m \sigma C_\chi 
            -
            \mfrac{L^2}{2} \eta_t
        \Big)
    \Big)
    \,
    \sqrt{ \delta^{\rm SGDw}_{t, j-1} }
    \,+\,
    \eta_t \Big(   \varepsilon^{\rm SGDw}_{n,m,t;\nu}(\epsilon)  + \mfrac{5 \sigma + 5\kappa_{\nu;m}}{\sqrt{B}} \Big)
    \;,
\end{align*}
where $  1
-
\eta_t 
\big( 
    \mu 
    - 
    \alpha^m \sigma C_\chi 
    -
    \frac{L^2}{2} \eta_t
\big) > 0$\,.
\end{lemma} 

\begin{proof}[Proof of \cref{lem:SGDw:recursion}] We first remark that in view of \cref{lemma:contraction}, the projection step in \cref{alg:offline_cd} does not increase $\delta^{\rm SGD_w}_{t,j}$, so it suffices to bound $\delta^{\rm SGDw}_{t,j}$ as if projection is not performed on $\psi^{\rm sGDw}_{t,j}$. To apply the results from \cref{appendix:L2:approx}, we identify $\theta^{\rm init} = \psi^{\rm SGDw}_{t,j-1}$ and $\eta = \eta_t$, which allows us to write $\psi^{\rm SGDw}_{t,j} = \theta^{\rm SGDw}_{m,B}$. This also implies $\mean \| \theta^{\rm init} -  \psi^* \|^2 = \delta^{\rm SGDw}_{t, j-1}$ and $\varepsilon^{\rm init}_{n,m;\nu}(\epsilon) \leq \varepsilon^{\rm SGDw}_{\nu;n,m,t}(\epsilon)$. By adding and subtracting the auxiliary gradient updates followed by expanding the square, we obtain 
\begin{align*}
    \delta^{\rm SGDw}_{t,j}
    \;=&\;
    \mean \, \big\| 
        \theta^{\rm SGDw}_{m,B} - \theta^{\rm GD}_{m}
        +
        \theta^{\rm GD}_{m} - \theta^{\rm GD}_{\infty}
        +
        \theta^{\rm GD}_{\infty} - \theta^{\rm pop}
        +
        \theta^{\rm pop} - \psi^* 
    \big\|^2
    \\
    \;\overset{(a)}{=}&\;
    \mean \big\| 
        \theta^{\rm SGDw}_{m,B} - \theta^{\rm GD}_{m} 
    \big\|^2
    +
    \mean \big\| 
        \theta^{\rm GD}_{m} - \theta^{\rm GD}_{\infty}
    \big\|^2
    +
    \mean \big\| 
        \theta^{\rm GD}_{\infty} - \theta^{\rm pop}
    \big\|^2
    +
    \mean \big\| 
        \theta^{\rm pop} - \psi^* 
    \big\|^2
    \\
    &\;
    + 
    2 \,\mean\, \big\langle 
        \theta^{\rm SGDw}_{m,B} - \theta^{\rm GD}_{m} 
        \,,\, 
        \theta^{\rm GD}_{m} - \theta^{\rm GD}_{\infty}
    \big\rangle
    +
    2  \,\mean\, \big\langle 
        \theta^{\rm SGDw}_{m,B} - \theta^{\rm GD}_{m} 
        \,,\, 
        \theta^{\rm GD}_{\infty} - \theta^{\rm pop}
    \big\rangle
    \\
    &\;
    +
    2  \,\mean\,\big\langle 
        \mean\big[ \theta^{\rm SGDw}_{m,B} - \theta^{\rm GD}_{m} \,\big|\, \theta^{\rm init}, \cF_n \big] 
        \,,\, 
        \theta^{\rm pop} - \psi^* 
    \big\rangle
    +
    2  \,\mean\,\big\langle 
        \theta^{\rm GD}_{m} - \theta^{\rm GD}_{\infty}
        \,,\, 
        \theta^{\rm GD}_{\infty} - \theta^{\rm pop}
    \big\rangle
    \\
    &\;
    +
    2  \,\mean\,\big\langle 
        \mean\big[ \theta^{\rm GD}_{m} - \theta^{\rm GD}_{\infty} \,\big|\, \cA_n \big]
        \,,\, 
        \theta^{\rm pop} - \psi^* 
    \big\rangle
    +
    2  \,\mean\,\big\langle 
        \theta^{\rm GD}_{\infty} - \theta^{\rm pop}
        \,,\, 
        \theta^{\rm pop} - \psi^* 
    \big\rangle
    \\
    \;\overset{(b)}{\leq}&\;
    \mfrac{4 \eta_t^2 (\sigma^2 + \kappa_{\nu;m}^2)}{\sqrt{B}} \ind_{\{B < n \}}
    + 
    \eta_t^2 
        \Big(   
            \Big( 
                \alpha^{m} \sigma C_\chi  \, \sqrt{\delta^{\rm SGDw}_{t, j-1} } \, + \, \varepsilon^{\rm SGDw}_{n,m,t;\nu}(\epsilon)
            \Big)^2
            +
            \mfrac{\kappa_{\nu;m}^2 + \sigma^2}{n}
        \Big)
    \\
    &\;
    +
    \mfrac{4 \eta^2_t \sigma^2}{n}
    +
    (1 - 2 \mu \eta_t + L^2 \eta^2_t) \delta^{\rm SGDw}_{t, j-1}
    \\
    &\;
    +
    2  \mfrac{2 \eta_t \sqrt{\sigma^2 + \kappa_{\nu;m}^2}}{\sqrt{B}} \, \ind_{\{B < n \}} \eta_t 
    \Big(   
            \alpha^{m} \sigma C_\chi  \, \sqrt{\delta^{\rm SGDw}_{t, j-1} } \, + \, \varepsilon^{\rm SGDw}_{n,m,t;\nu}(\epsilon)
        +
        \mfrac{\sqrt{\kappa_{\nu;m}^2 + \sigma^2}}{\sqrt{n}}
    \Big) 
    \\
    &\;
    + 
    2  \mfrac{2 \eta_t \sqrt{\sigma^2 + \kappa_{\nu;m}^2}}{\sqrt{B}} \, \ind_{\{B < n \}}  \mfrac{2 \eta_t \sigma}{\sqrt{n}} 
    +
    0 
    \\
    &\;
    + 
    2  \eta_t
    \Big(   
        \alpha^{m} \sigma C_\chi  \, \sqrt{\delta^{\rm SGDw}_{t, j-1} } \, + \, \varepsilon^{\rm SGDw}_{n,m,t;\nu}(\epsilon)
        +
        \mfrac{\sqrt{\kappa_{\nu;m}^2 + \sigma^2}}{\sqrt{n}}
    \Big)
    \mfrac{2 \eta_t \sigma}{\sqrt{n} }
    \\
    &\;
    + 
    2 \eta_t 
        \Big( 
            \alpha^{m} \sigma C_\chi  \, \sqrt{\delta^{\rm SGDw}_{t, j-1} } \, + \, \varepsilon^{\rm SGDw}_{n,m,t;\nu}(\epsilon)
        \Big)
        \sqrt{ (1 - 2 \mu \eta_t + L^2 \eta^2_t) \delta^{\rm SGDw}_{t, j-1} }
    \\
    &\;
    + 2  \mfrac{2 \eta_t \sigma}{\sqrt{n}}  \sqrt{ (1 - 2 \mu \eta_t + L^2 \eta^2_t) \delta^{\rm SGDw}_{t, j-1} }
    \\
    \;\overset{(c)}{\leq}&\;
    \Big( 1 - 2 \mu \eta_t + L \eta_t^2 + \eta_t^2 \alpha^{2m} \sigma^2 C^2_\chi + 2 \eta_t \alpha^m \sigma C_\chi \sqrt{1 - 2 \mu \eta_t + L^2 \eta_t^2} \, \Big) \times \delta^{\rm SGDw}_{t, j-1} 
    \\
    &\; 
    +
    2 \bigg( \eta_t^2 \alpha^m \sigma C_\chi \varepsilon^{\rm SGDw}_{n,m,t;\nu}(\epsilon) 
            + 
            \mfrac{2 \eta_t^2 (\sqrt{\sigma^2 + \kappa_{\nu;m}^2} + \sigma) \, \alpha^m \sigma C_\chi}{\sqrt{B}} 
    \\
    &\hspace{3em}
        +
        \eta_t \varepsilon^{\rm SGDw}_{n,m,t;\nu}(\epsilon)  \sqrt{1-2\mu \eta_t + L_2 \eta_t^2}
        + 
        \mfrac{2 \eta_t \sigma  \sqrt{1-2\mu \eta_t + L^2 \eta_t^2} }{\sqrt{n}}\, \bigg) 
    \times \sqrt{ \delta^{\rm SGDw}_{t, j-1} }
    \\
    &\;
    +
    \eta_t^2 
    \Big( 
        \mfrac{4 (\sigma^2 + \kappa_{\nu;m}^2)}{B} 
        +
        \big( \varepsilon^{\rm SGDw}_{n,m,t;\nu}(\epsilon)  \big)^2
        +
        \mfrac{5\kappa_{\nu;m}^2 + 9 \sigma^2}{B}
        +
        \mfrac{4 \sqrt{\sigma^2 + \kappa_{\nu;m}^2} \, \varepsilon^{\rm SGDw}_{n,m,t;\nu}(\epsilon)}{\sqrt{B}} 
    \\
    &\hspace{3em} 
        +
        \mfrac{ 8 \sigma \sqrt{\sigma^2 + \kappa_{\nu;m}^2}}{B} 
        +
        \mfrac{4 \sigma \varepsilon^{\rm SGDw}_{n,m,t;\nu}(\epsilon)}{\sqrt{B}}
    \Big)
    \\
    \;\eqqcolon&\; (A_1) \times \delta^{\rm SGDw}_{t, j-1}  + (A_2) \times \sqrt{ \delta^{\rm SGDw}_{t, j-1} } + (A_3)
    \;.
\end{align*}
In $(a)$, we have expanded the square, used $\cF_n$ defined in \cref{lem:SGDw:GD} and $\cA_n$ defined in \cref{lem:GD:sample:trunc:error}, and noted that $\theta^{\rm pop} - \psi^*$ is almost surely constant given $\theta^{\rm init}$; in $(b)$, we have applied Lemmas \ref{lem:SGDw:GD}, \ref{lem:GD:sample:trunc:error}, \ref{lem:GD:sample:pop:error} and \ref{lem:GD:pop:opt:error} under \ref{asst:SGD-A1}, \ref{asst:SGD-A4}, \ref{asst:SGD-A5},  \ref{asst:Markov:noise} and \ref{asst:Markov:noise:X1}, and used $\sqrt{a+b} \leq \sqrt{a} + \sqrt{b}$ for $a,b \geq 0$; in $(c)$, we have grouped the terms by powers of $\delta^{\rm SGDw}_{t, j-1}$, bounded the indicator function from above by $1$ and used $B \leq n$ to replace $n$ in the denominator by $B$. While the computation above is complicated, we remark that the key steps are taking the conditional expectations for the cross-terms in $(b)$, which gives us a tighter bound than directly applying a triangle inequality of the form $\mean \| Y_1 + Y_2\|^2 \leq ( \sqrt{\mean \| Y_1 \|^2} + \sqrt{\mean \| Y_2\|^2} )$. To further simplify the bounds, we seek to bound each coefficient by a square, which yields 
\begin{align*}
    (A_1) 
    \;=&\;
    \Big(
            \eta_t \alpha^{m} \sigma C_\chi 
            +
            \sqrt{1-2\mu\eta_t+L\eta_t^2} \,
    \Big)^2 
    \;,
    \\
    (A_2)
    \;\leq&\;
    2
    \Big(  \eta_t \alpha^{m} \sigma C_\chi 
    +
    \sqrt{1-2\mu\eta_t+L\eta_t^2}  \,\Big)
    \times 
    \,
    \eta_t \Big( \varepsilon^{\rm SGDw}_{n,m,t;\nu}(\epsilon) + \mfrac{4 \sigma + 2\kappa_{\nu;m}}{\sqrt{B}} \Big)
    \;,
    \\
    (A_3)
    \;\leq&\;
    \eta^2_t \Big( 
        \big( \varepsilon^{\rm SGDw}_{n,m,t;\nu}(\epsilon)  \big)^2
        + 
         \mfrac{8 \sigma^2+ 4\kappa_{\nu;m}^2}{\sqrt{B}} \varepsilon^{\rm SGDw}_{n,m,t;\nu}(\epsilon)     
        +
        \mfrac{21 \sigma^2 + 17 \kappa_{\nu;m}^2}{B}
    \Big)
    \\
    \;\leq&\;
    \eta_t^2 \Big(   \varepsilon^{\rm SGDw}_{n,m,t;\nu}(\epsilon)  + \mfrac{5 \sigma + 5\kappa_{\nu;m}}{\sqrt{B}} \Big)^2\;.
\end{align*}
This implies 
\begin{align*}
    \delta^{\rm SGDw}_{t, j}
    \;\leq&\;
    (A_1) \times \delta^{\rm SGDw}_{t, j-1}  + (A_2) \times \sqrt{ \delta^{\rm SGDw}_{t, j-1} } + (A_3)
    \\
    \;\leq&\;
    \Big( 
        \Big(
            \eta_t \alpha^{m} \sigma C_\chi 
            +
            \sqrt{1-2\mu\eta_t+L\eta_t^2}
        \Big) 
            \sqrt{ \delta^{\rm SGDw}_{t, j-1} }  
        +
        \eta_t \Big(   \varepsilon^{\rm SGDw}_{n,m,t;\nu}(\epsilon)  + \mfrac{5 \sigma + 5\kappa_{\nu;m}}{\sqrt{B}} \Big)
    \Big)^2\;.
\end{align*}
Now note that since $\mu \leq L$ by the definitions in \ref{asst:SGD-A1}, $2 \mu \eta_t - L^2 \eta^2_t \leq 2 L \eta_t - L^2 \eta_t^2 \leq 1$. By using $\sqrt{1-x} \leq 1 - \frac{x}{2}$ for all $x \leq 1$, we get that 
\begin{align*}
    \sqrt{ 1 - 2 \mu \eta_t + L^2 \eta^2_t }
    \;\leq\; 
    1 - \mu \eta_t + \mfrac{L^2}{2} \eta^2_t\;.
\end{align*}
Substituting this into the earlier bound and taking a square-root, we obtain the desired bound that 
\begin{align*}
    \sqrt{\delta^{\rm SGDw}_{t, j}}
    \;\leq\;
    \Big(
        1
        -
        \eta_t 
        \Big( 
            \mu 
            - 
            \alpha^m \sigma C_\chi 
            -
            \mfrac{L^2}{2} \eta_t
        \Big)
    \Big)
    \,
    \sqrt{ \delta^{\rm SGDw}_{t, j-1} }
    \,+\,
    \eta_t \Big(   \varepsilon^{\rm SGDw}_{n,m,t;\nu}(\epsilon)  + \mfrac{5 \sigma + 5\kappa_{\nu;m}}{\sqrt{B}} \Big)
    \;.
\end{align*}
Moreover, since $L \geq \mu$ by definition from \ref{asst:SGD-A1}, we have 
\begin{align*}
    1
    -
    \eta_t 
    \Big( 
        \mu 
        - 
        \alpha^m \sigma C_\chi 
        -
        \mfrac{L^2}{2} \eta_t
    \Big)
    \;=&\;
    \Big( \mfrac{L}{\sqrt{2}} \eta_t - 1 \Big)^2
    + (\sqrt{2}\, L - \mu) \eta_t
    +
    \alpha^m \sigma C_\chi \eta_t
    \;>\; 0\;,
\end{align*}
which finishes the proof.
\end{proof}

\subsection{Proof of Theorem \ref{thm:SGDw:full}} \label{appendix:proof:SGDw}

Under \ref{asst:SGD-A1}, \ref{asst:SGD-A4}, \ref{asst:SGD-A5}, \ref{asst:Markov:noise} and \ref{asst:Markov:noise:X1}, \cref{lem:SGDw:recursion} implies
\begin{align*}
    \sqrt{\delta^{\rm SGDw}_{t, j}}
    \;\leq&\;
    \Big(
        1
        -
        \tilde \mu_m C t^{-\beta}
        +
        \mfrac{L^2 C^2}{2} t^{-2\beta}
    \Big)
    \,
    \sqrt{ \delta^{\rm SGDw}_{t, j-1} }
    \,+\,
    C t^{-\beta}
    \,
    \sigma^{\rm SGDw}_{n,T} 
\end{align*}
for $1 \leq t \leq T$ and $1 \leq j \leq N$, where we have used $\tilde \mu_m = \mu - \alpha^m \sigma C_\chi$, $\eta_t = C t^{-\beta}$ and
\begin{align*}
    \varepsilon^{\rm SGDw}_{n,m,t;\nu}(\epsilon)  + \mfrac{5 \sigma + 5\kappa_{\nu;m}}{\sqrt{B}}
    \;\leq\;
    \varepsilon^{\rm SGDw}_{n,m,T;\nu}(\epsilon)  + \mfrac{5 \sigma + 5\kappa_{\nu;m}}{\sqrt{B}}
    \;=\; \sigma^{\rm SGDw}_{n,T} \;.
\end{align*}
By an induction on $j=1,\ldots, N$, we have
\begin{align*}
    \sqrt{\delta^{\rm SGDw}_{t, N}}
    \;\leq&\;
    \Big(
        1
        -
        \tilde \mu_m C t^{-\beta}
        +
        \mfrac{L^2 C^2}{2} t^{-2\beta}
    \Big)^N
    \sqrt{\delta^{\rm SGDw}_{t, 0}}
    \\
    &\;\;\,+\,
    C \sigma^{\rm SGDw}_{n,T} \, t^{-\beta} \, 
    \msum_{j=1}^N \Big(
        1
        -
        \tilde \mu_m C t^{-\beta}
        +
        \mfrac{L^2 C^2}{2} t^{-2\beta}
    \Big)^{j-1}
    \;.
\end{align*}
By noting that $\delta^{\rm SGDw}_{t, 0} = \delta^{\rm SGDw}_{t-1, N}$ almost surely for $t \geq 2$ and using another induction on $t=1,\ldots, T$, 
\begin{align*}
    \sqrt{\delta^{\rm SGDw}_{T, N}}
    \,\leq&\;
    Q_0 \,
    \sqrt{\delta^{\rm SGDw}_{0,0}}
    \,+\,
    C \sigma^{\rm SGDw}_{n,T} 
    \msum_{t=1}^T 
    Q_t \, A_t
    \,
    t^{-\beta}
    \;, \tagaligneq \label{eq:SGDw:before:final}
\end{align*}
where
\begin{align*}
    Q_t 
    \coloneqq
    \mprod_{s=t+1}^T 
    \Big(
        1
        -
        \tilde \mu_m C s^{-\beta}
        +
        \mfrac{L^2 C^2}{2} s^{-2\beta}
    \Big)^N
    \;\text{ and }\;
    A_t 
    \coloneqq
    \msum_{j=1}^N \Big(
        1
        -
        \tilde \mu_m C t^{-\beta}
        +
        \mfrac{L^2 C^2}{2} t^{-2\beta}
    \Big)^{j-1}
    \;.
\end{align*}
Note that by \cref{lem:SGDw:recursion}, we also have that $ 1
-
\tilde \mu_m C t^{-\beta}
+
\mfrac{L^2 C^2}{2} t^{-2\beta} > 0$.
This allows us to apply \cref{lem:error:accumulate} to obtain
\begin{align*}
    \kappa_0 
    \;\leq\; 
    \exp\Big( 
        1 
        - 
        N \tilde \mu_m C \varphi_{1-\beta}(T+1)
        +
        \mfrac{ N L^2 C^2}{2} \varphi_{1-2\beta}(T+1)
    \Big)
    \;=\;
    E_1^{T,N}
    \;.
\end{align*}
To control $\sum_{t=1}^T Q_t A_t t^{-\beta}$, recall that $\beta \in [0,1]$, $\tilde \mu_m > 0$, $\frac{L^2 C^2}{2} > 0$, and that
\begin{align*}
    E_2^{T,N} \;=&\; \exp\Big( 
        - \mfrac{N \tilde \mu_m C}{2} \varphi_{1-\beta}(T+1)
        + 2 N L^2 C^2 \varphi_{1-2\beta}(T+1)
     \Big)\;.
\end{align*}
We can now apply \cref{lem:error:accumulate}: If $\beta \not\in \{\frac{1}{2},1\}$, then
\begin{align*}
    \msum_{t=1}^T Q_t \, A_t  \, t^{-\beta}
    \;\leq&\;
    \mfrac{2^{2\beta+1}}{\tilde \mu_m C} e^{\frac{\tilde \mu_m C}{2(1-\beta)} \frac{N}{(T+1)^\beta}}
    +
    \mfrac{3^\beta (1+\tilde \mu_m C)^{N-1} (T+2)^{\beta} }{L^2C^2}  \,  E_2^{T,N}
     \;.
\end{align*}
If $\beta = \frac{1}{2}$, i.e.~$2\beta=1$, we have
\begin{align*}
    \msum_{t=1}^T Q_t \, A_t  \, t^{-\beta}
    \;\leq&\;
    \mfrac{4}{\tilde \mu_m C} e^{\frac{\tilde \mu_m C  N}{(T+1)^{1/2}}}
    +
    2 N ( 1 + \tilde \mu_m C)^{N-1} 
    \varphi_{\frac{1}{2}-L^2C^2 N}(T+1)
    \, E_2^{T,N}
     \;.
\end{align*}
If $\beta=1$, we get that 
\begin{align*}
    \msum_{t=1}^T Q_t \, A_t  \, t^{-\beta}
    \;\leq&\;
    \mfrac{4}{\tilde \mu_m C}
    +
    \mfrac{3 N \big(1+\frac{L^2C^2}{2}\big)^{N-1} \, e^{2 L^2 C^2 N}  \, \log(T+1)}{(T+1)^{(\tilde \mu_m C N)/2}}
\end{align*}
Substituting the bounds into \eqref{eq:SGDw:before:final}, we get the desired bound that $\sqrt{\delta^{\rm SGDw}_{T, N}}$ is upper bounded by
\begin{align*}
    \begin{cases}
        E_1^{T,N}
        \sqrt{\delta^{\rm SGDw}_{0, 0}}
        \,+\,
        C \sigma^{\rm SGDw}_{n,T}
        \Big( 
            \mfrac{4 e^{\frac{\tilde \mu_m C  N}{(T+1)^{1/2}}}}{\tilde \mu_m C}
            +
            2 N ( 1 + \tilde \mu_m C)^{N-1} 
            \varphi_{\frac{1}{2}-L^2C^2 N}(T+1)
            \, E_2^{T,N}
        \Big)
        \\
        \hspace{34.4em}
        \text{ for } \beta = \mfrac{1}{2} 
        \;,
        \\[.8em]
        E_1^{T,N}
        \sqrt{\delta^{\rm SGDw}_{0, 0}}
        +
        C \sigma^{\rm SGDw}_{n,T}
        \Big(
        \mfrac{4}{\tilde \mu_m C}
        +
        \mfrac{3 N \big(1+\frac{L^2C^2}{2}\big)^{N-1} \, e^{2 L^2 C^2 N}  \, \log(T+1)}{(T+1)^{(\tilde \mu_m C N)/2}}
        \Big)
        \hspace{4em}
        \text{ for } \beta = 1
        \;,
        \\[.8em]
        E_1^{T,N}
        \sqrt{\delta^{\rm SGDw}_{0, 0}}
        \,+\,
        C \sigma^{\rm SGDw}_{n,T} 
        \Big( 
            \mfrac{2^{2\beta+1}}{\tilde \mu_m C} 
            e^{\frac{\tilde \mu_m C}{2(1-\beta)} \frac{N}{(T+1)^\beta}}
            +
            \mfrac{3^\beta (1+\tilde \mu_m C)^{N-1} (T+2)^{\beta} }{L^2C^2}  \,  E_2^{T,N}
        \Big)
        \hspace{.2em}
        \text{ otherwise}
        \;.
    \end{cases}
\end{align*}
\qed 

\subsection{Proof of Theorem \ref{thm:SGDo}} \label{appendix:proof:SGDo}

Recall that $\delta^{\rm SGDo}_{t,j} \coloneqq \mean \, \big\| \psi^{\rm SGDo}_{t,j}  - \psi^* \big\|^2$.
To apply the results from \cref{appendix:L2:approx} to $\psi^{\rm SGDo}_{t,j}$, we condition on $S^o_t$, which in particular fixes $S^o_{t,j}$, the last size-$B$ subset of $[n]$ chosen. We then identify $\theta^{\rm init} = \psi^{\rm SGDo}_{t,j-1}$, $\eta = \eta_t$ and the dataset used as $\cD^o_{t,j} \coloneqq (X_i \,:\, i \in S^o_{t,j})$, which allows us to identify $\psi^{\rm SGDo}_{t,j} = \theta^{\rm GD}_m$, the full-batch gradient descent update using $\cD^o_{t,j}$. Meanwhile, we note that almost surely $\theta^{\rm GD}_m = \theta^{\rm SGDw}_{m,B}$, the SGD-with-replacement iterate that uses the full dataset $\cD^o_{t,j}$. Observe that the proof of \cref{lem:SGDw:recursion} holds with $\delta^{\rm SGDo}_{t, j-1}$ replaced by any random initialization $\theta^{\rm init}$ possibly correlated with $X_1, \ldots, X_n$, which allows us to obtain
\begin{align*}
    \sqrt{ \delta^{\rm SGDo}_{t,j} }
    \;\leq\; 
    \Big(
        1
        -
        \eta_t 
        \Big( 
            \mu 
            - 
            \alpha^m \sigma C_\chi 
            -
            \mfrac{L^2}{2} \eta_t
        \Big)
    \Big)
    \,
    \sqrt{ \delta^{\rm SGDo}_{t, j-1} }
    \,+\,
    \eta_t \Big(   \varepsilon^{\rm SGDw}_{n,m,t;\nu}(\epsilon)  + \mfrac{5 \sigma + 5\kappa_{\nu;m}}{\sqrt{B}} \Big)\;.
\end{align*}
Since the error recursion for $\delta^{\rm SGDo}_{t,j}$ is identical to that of $\delta^{\rm SGDw}_{t,j}$ in \cref{lem:SGDw:recursion}, the proof of \cref{thm:SGDw} follows directly, thereby yielding an identical result for $\delta^{\rm SGDo}_{T,N}$ as with $\delta^{\rm SGDw}_{T,N}$ in \cref{thm:SGDw}.
\qed

\section{Proofs for tail probability bounds in offline SGD} \label{appendix:proofs:SGDw:tail}

We present the proofs for results in \cref{appendix:tail:results} that control the tail probability terms $\vartheta^{\rm SGDw}_{\nu;n,m,T}$ and $\varepsilon^{\rm SGDw}_{\nu;n,m,T}$.

\begin{proof}[Proof of \cref{lem:tail:prob:subexp:full}] For $\delta > 0$, let $N_\delta$ be the $\delta$-covering number of $\Psi$, which satisfies $N_\delta \leq \big( r_\Psi (1+2/\delta) \big)^p$ (Example 5.8, \cite{wainwright2019high}). Note also that by the Jensen's inequality applied to $\mean[\argdot| X_1]$ and Assumption \ref{asst:subexp}, there exist some $\sigma_m, \zeta_m > 0$ such that, for any $z \in \R^p$ with $\|z \| \leq \zeta_m$,
\begin{align*}
    \mean[ e^{z^\top(\mean[\phi(K^m_{\psi^*}(X_1)) | X_1] - \mean[\phi(K^m_{\psi^*}(X_1))] ) } ]
    \;\leq&\;
    \mean[ e^{z^\top(\phi(K^m_{\psi^*}(X_1)) - \mean[\phi(K^m_{\psi^*}(X_1))] ) } ]
    \;\leq\; 
    e^{\sigma_m^2 \| z \|^2 /2} \;.
\end{align*}
Meanwhile under \ref{asst:Jiang}, recycling the proof of Lemma 3.1 of \cite{jiang2018convergence} shows that if $C_m \delta / \sqrt{p} < \sigma^2_m \zeta_m$, 
    \begin{align*}
        \P\Big( \sup_{\psi \in \Psi} \Big\| \mfrac{1}{n} \msum_{i=1}^n \mean\big[ \phi\big(K^m_{\psi}(X_i) \big) \big| X_i \big]  
        - 
        \mean\big[  \phi(K^m_{\psi}( & X_i)) 
        \big] \Big\| 
        >  3 C_m \delta
        \Big) 
        \;\leq\;
        2 N_\delta p \exp \Big( - \mfrac{n C^2_m \delta^2}{2 p \sigma_m^2} \Big)
        \\
        \;\leq&\;
        2 (r_\Psi)^p 
        \exp \Big( p \log(1 + 2 \delta^{-1} ) - \mfrac{n C^2_m \delta^2}{2 p \sigma_m^2} \Big)
        \;.
    \end{align*}
    Since the probability above is an upper bound to $\vartheta^{\rm SGDw}_{n,m,T} (3 C_m \delta)$, we get that if $C_m \delta / \sqrt{p} < \sigma^2_m \zeta_m$, 
    \begin{align*}
        \Big(\varepsilon^{\rm SGDw}_{\nu;n,m,T}(3 & C_m \delta) \Big)^2
        \;=\;
        9 C_m^2 \delta^2
        + 
        \kappa_{\nu;m}^2
        \Big(\vartheta^{\rm SGDw}_{n,m,T}\Big( \mfrac{C_m \delta}{\sqrt{p}} \Big) \Big)^{\frac{\nu-2}{\nu}}
        \\
        \;\leq&\;
        9 C_m^2 \delta^2
        + 
        \kappa_{\nu;m}^2
        2^{\frac{2(\nu-2)}{\nu}}
        (r_\Psi)^{\frac{(\nu-2) p}{\nu}}
        \exp \Big( \mfrac{(\nu-2)p}{\nu} \log(1 + 2 \delta^{-1} ) - \mfrac{n (\nu-2) C^2_m \delta^2}{2 \nu p \sigma_m^2} \Big)
        \;. 
    \end{align*}
    Recall that by assumption, $\frac{\log n}{n} < \frac{\sigma^2_m \zeta^2_m}{p + \nu-2}$. We now choose 
    \begin{align*}
        \delta 
        \;=\; 
        \mfrac{\sigma_m}{C_m} \sqrt{ p  \Big(p + \mfrac{2 \nu}{\nu-2} \Big) \times \mfrac{\log n}{n}}
        \;=\;
        \mfrac{\sigma_m}{C_m} \sqrt{ p \times \mfrac{(\nu - 2) p + 2 \nu}{\nu-2} \times \mfrac{\log n}{n}}
        \;,
    \end{align*}
    which implies 
    \begin{align*}
        &\;
        \inf_{\epsilon > 0}\Big(\varepsilon^{\rm SGDw}_{\nu;n,m,T}(\epsilon)\Big)^2
        \;\leq\;
        \Big(\varepsilon^{\rm SGDw}_{\nu;n,m,T}(3 C_m \delta) \Big)^2
        \\
        &\;\leq\;
        \mfrac{9\sigma_m^2 p ( (\nu-2)p + 2 \nu) \log n }{(\nu - 2) n}
        \,+\,
        \kappa_{\nu;m}^2
        2^{\frac{\nu-2}{\nu}}
        (r_\Psi)^{\frac{(\nu-2) p}{\nu}}
        (1 + 2 \delta^{-1})^{\frac{(\nu-2)p}{\nu}}
        \exp\Big(
            - \mfrac{(\nu -2) p + 2 \nu}{2 \nu } \log n
        \Big)
        \\
        &\;\leq\;
        \mfrac{9\sigma_m^2 p ( (\nu-2)p +2 \nu) \log n }{(\nu - 2) n}
        \\
        &\qquad \,+\,
        \kappa_{\nu;m}^2
        2^{\frac{\nu-2}{\nu}}
        (r_\Psi)^{\frac{(\nu-2) p}{\nu}}
        \Big( 1 + \mfrac{2 C_m (\nu-2)^{1/2}}{\sigma_m p^{1/2} ((\nu-2)p + 2\nu)^{1/2}} \mfrac{\sqrt{n}}{\sqrt{\log n}} \Big)^{\frac{(\nu -2)p}{\nu}}
        n^{- \frac{(\nu -2) p + 2 \nu}{2 \nu }}
        \\
        &\;\leq\;
        \mfrac{9\sigma_m^2 p ( (\nu-2)p +2 \nu) \log n }{(\nu - 2) n}
        \\
        &\qquad \,+\,
        \kappa_{\nu;m}^2
        2^{\frac{\nu-2}{\nu}}
        (r_\Psi)^{\frac{(\nu-2) p}{\nu}}
        \Big( 1 + \mfrac{2 C_m (\nu-2)^{1/2}}{\sigma_m p^{1/2} ((\nu-2)p + 2\nu)^{1/2}} \Big)^{\frac{(\nu -2)p}{\nu}}
        n^{\frac{(\nu-2)p}{2\nu} - \frac{(\nu -2) p + 2\nu}{2 \nu }}
        \\
        &\;\leq\;
        \bigg( 
        \mfrac{9\sigma_m^2 p ( (\nu-2)p +2 \nu) }{\nu - 2}
        +  
        \kappa_{\nu;m}^2
        2^{\frac{(\nu-2)}{\nu}}
        (r_\Psi)^{\frac{(\nu -2)p}{\nu}} 
        \Big( 1 + \mfrac{2 C_m (\nu-2)^{1/2}}{\sigma_m p^{1/2} ((\nu-2)p + 2\nu)^{1/2}} \Big)^{\frac{(\nu -2)p}{\nu}}
        \bigg) 
        \, \mfrac{\log n}{n}
        \;.
    \end{align*}
Taking a squareroot across and using $\sqrt{a+b} \leq \sqrt{a} + \sqrt{b}$ for $a,b >0$ gives the desired bound. The limiting result follows by substituting this bound into \cref{thm:SGDw:full}.
\end{proof}

\vspace{.5em}

\begin{proof}[Proof of \cref{lem:tail:prob:subopt:full}] Let $\delta > 0$ and $N_\delta$ be defined as in the proof of \cref{lem:tail:prob:subexp:full}, with $N_\delta \leq \big( r_\Psi (1+2/\delta) \big)^p$. Let $(\psi_l)_{l=1}^{N_\delta}$ be the centers of the covering $\delta$-balls. The covering-ball argument of the proof of Lemma 3.1 of \cite{jiang2018convergence} shows that 
\begin{align*}
    \vartheta^{\rm SGDw}_{n,m,T}(3 C_m \delta)
    \;\leq&\;\P\Big( \, \sup_{\psi \in \Psi} \Big\| \mfrac{1}{n} \msum_{i=1}^n \mean\big[ \phi\big(K^m_{\psi}(X_i) \big) \big| X_i \big]  
    - 
    \mean\big[  \phi(K^m_{\psi}( X_i)) 
    \big] \Big\| 
    >  3 C_m \delta
    \Big) 
    \\
    \;\leq&\;
    \msum_{l=1}^{N_\delta}
    \P \Big( \Big\| \mfrac{1}{n} \msum_{i=1}^n \mean\big[ \phi\big(K^m_{\psi_l}(X_i) \big) \big| X_i \big]  
    - 
    \mean\big[  \phi(K^m_{\psi_l}( X_i)) 
    \big] \Big\| 
    \geq C_m \delta
    \Big) 
    \;.
\end{align*}
By a Markov's inequality followed by the Burkholder's inequality applied to an average of i.i.d.~summands (see e.g.~\cite{dharmadhikari1968bounds} for $\nu > 2$), there exists a constant $C_\nu > 0$ depending only on $\nu$ such that 
\begin{align*}
    \vartheta^{\rm SGDw}_{n,m,T}(3 C_m \delta)
    \;\leq&\; 
    \msum_{l=1}^{N_\delta}
    \mfrac{\mean \Big\| \msum_{i=1}^n \mean\big[ \phi\big(K^m_{\psi_l}(X_i) \big) \big| X_i \big]  
    - 
    \mean\big[  \phi(K^m_{\psi_l}( X_i)) 
    \big] \Big\|^\nu }
    {n^\nu C_m^\nu \delta^\nu}
    \\
    \;\leq&\;
    \msum_{l=1}^{N_\delta}
    \mfrac{\mean \big\| \mean\big[ \phi\big(K^m_{\psi_l}(X_1) \big) \big| X_i \big]  
    - 
    \mean\big[  \phi(K^m_{\psi_l}( X_1)) 
    \big] \big\|^\nu }
    {n^{\nu/2} C_m^\nu \delta^\nu}
    \\
    \;\leq&\;
    \mfrac{N_\delta  \kappa_{\nu;m}^\nu}{n^{\nu/2} C_m^\nu \delta^\nu}
    \;\leq\;
    \mfrac{( r_\Psi)^p (1+2\delta^{-1})^p \kappa_{\nu;m}^\nu}{n^{\nu/2} C_m^\nu \delta^\nu}
    \;,
\end{align*}
where we have used \cref{asst:Markov:noise} in the last line. This implies  
\begin{align*}
    \Big(\varepsilon^{\rm SGDw}_{\nu;n,m,T}(3C_m \delta) \Big)^2
    \;=&\;
    9 C_m^2 \delta^2
    + 
    \kappa_{\nu;m}^2
    \Big(\vartheta^{\rm SGDw}_{n,m,T}(3C_m \delta) \Big)^{(\nu-2)/\nu}
    \\
    \;\leq&\;
    9 C_m^2 \delta^2
    + 
    \kappa_{\nu;m}^\nu
    (r_\Psi)^{\frac{(\nu-2) p}{\nu}}
    C_m^{-(\nu-2)}
    \times 
    \mfrac{ (1+2\delta^{-1})^{\frac{(\nu-2) p}{\nu}}}{n^{\frac{\nu-2}{2}} \delta^{\nu-2}}
    \;. 
\end{align*}
Choosing $\delta = n^{ - \frac{(\nu-2)\nu}{2(\nu^2+(\nu-2)p)} } \leq 1$, we get that 
\begin{align*}
    \inf_{\epsilon > 0}\Big(\varepsilon^{\rm SGDw}_{\nu;n,m,T} & \, (\epsilon)\Big)^2
    \;\leq\;
    \Big(\varepsilon^{\rm SGDw}_{\nu;n,m,T}(3 C_m \delta) \Big)^2
    \\
    \;\leq&\;
    9 C_m^2  n^{  - \frac{(\nu-2)\nu}{\nu^2+(\nu-2)p} }
    + 
    \kappa_{\nu;m}^\nu
    (r_\Psi)^{\frac{(\nu -2) p}{\nu}}
    C_m^{- (\nu-2)}
    3^{\frac{(\nu -2) p}{\nu}}
    \,  n^{  - \frac{(\nu-2)\nu}{\nu^2+(\nu-2)p} }
    \\
    \;=&\;
    \big(
    9 C_m^2 
    + 
   \kappa_{\nu;m}^\nu
    (r_\Psi)^{\frac{(\nu -2) p}{\nu}}
    C_m^{- (\nu-2)}
    3^{\frac{(\nu-2) p}{\nu}}
    \big) 
    \,
      n^{  - \frac{(\nu-2)\nu}{\nu^2+(\nu-2)p} }
    \;.
\end{align*}
Taking a squareroot across and using $\sqrt{a+b} \leq \sqrt{a} + \sqrt{b}$ for $a,b >0$ gives the desired bound. The limiting result follows by substituting this bound into \cref{thm:SGDw:full}.
\end{proof}

\vspace{.5em}

\begin{proof}[Proof of \cref{lem:tail:prob:ergodicity}] By a Markov's inequality and a Jensen's inequality with respect to the empirical average, we have 
\begin{align*}
    \vartheta^{\rm SGDw}_{n,m,T} (\epsilon) 
    \;\leq&\;
    \sup_{t \in [T], j \in [N]} 
    \mfrac{  \sum_{i=1}^n \mean \Big\| \mean\Big[ \phi\Big(K^m_{\psi^{\rm SGDw}_{t-1, j}}(X_i) \Big) \Big| X_i, \psi^{\rm SGDw}_{t-1,j} \Big]  
    - 
    \mean\Big[ \phi\Big(K^m_{\psi^{\rm SGDw}_{t-1,j}}(X'_1)\Big)  \Big| \psi^{\rm SGDw}_{t-1,j} 
    \Big] \Big\| }{n \epsilon}
    \\
    \;\eqqcolon&\;
    \sup_{t \in [T], j \in [N]}  \mfrac{\sum_{i=1}^n A_{tji}}{n \epsilon}
    \;.
\end{align*}
Meanwhile by a triangle inequality and the ergodicity assumption, we have 
\begin{align*}
    A_{tji}
    \;\leq&\; 
    \mean \Big\| \mean\Big[ \phi\Big(K^m_{\psi^{\rm SGDw}_{t-1, j}}(X_i) \Big) \Big| X_i, \psi^{\rm SGDw}_{t-1,j} \Big]  
    - 
    \mean\Big[ \phi\Big(X^{\psi^{\rm SGDw}_{t-1,j}}_1\Big)  \Big| \psi^{\rm SGDw}_{t-1,j} 
    \Big] \Big\| 
    \\
    &\;
    +
    \mean \Big\| \mean\Big[ \phi\Big(X^{\psi^{\rm SGDw}_{t-1,j}}_1\Big)  \Big| \psi^{\rm SGDw}_{t-1,j} 
    \Big] 
    - 
    \mean\Big[ \phi\Big(K^m_{\psi^{\rm SGDw}_{t-1,j}}(X'_1)\Big)  \Big| \psi^{\rm SGDw}_{t-1,j} 
    \Big] \Big\| 
    \\
    \;\leq&\; 2 \tilde C_K \tilde \alpha^m \;.
\end{align*}
This implies $\vartheta^{\rm SGDw}_{n,m,T} (\epsilon) \leq 2 \tilde C_K \tilde \alpha^m \epsilon^{-1}$, and therefore 
\begin{align*}
    \big(\varepsilon^{\rm SGDw}_{\nu;n,m,T}(\epsilon)\big)^2
    \;\leq\;
    \epsilon^2 
    + 
    2^{\frac{\nu-2}{\nu}}
    \kappa_{\nu;m}^2
    (\tilde C_K)^{\frac{\nu-2 }{\nu}}
    \tilde \alpha^{\frac{(\nu-2) m}{\nu}} 
    \epsilon^{-\frac{\nu -2 }{\nu}}
    \;.
\end{align*}
Choosing $\epsilon = \tilde \alpha^{(\nu-2) m / (3\nu-2) }$ gives 
\begin{align*}
    \minf_{\epsilon > 0} \big(\varepsilon^{\rm SGDw}_{\nu;n,m,T}(\epsilon)\big)^2 
    \;\leq&\; 
    \big( 1
    +
    2^{\frac{\nu-2}{\nu}}
    \kappa_{\nu;m}^2
    (\tilde C_K)^{\frac{\nu - 2 }{\nu}}
    \big) \tilde \alpha^{\frac{2(\nu - 2) m}{3\nu - 2}}\;.
\end{align*}
Taking a squareroot across and using $\sqrt{a+b} \leq \sqrt{a} + \sqrt{b}$ for $a,b >0$ gives the desired bound. The limiting result follows by substituting this bound into \cref{thm:SGDw:full}.
\end{proof}

\newpage
\section*{NeurIPS Paper Checklist}

\begin{enumerate}

\item {\bf Claims}
    \item[] Question: Do the main claims made in the abstract and introduction accurately reflect the paper's contributions and scope?
    \item[] Answer: \answerYes{}
    \item[] Justification: Our abstract summarizes the theoretical results of our paper. We cleary mention that while CD \emph{can} achieve these rates,
            assumptions are needed to guarantee this. We also mention the model class studied in this paper, namely (unnormalized) exponential family distributions.
    \item[] Guidelines:
    \begin{itemize}
        \item The answer NA means that the abstract and introduction do not include the claims made in the paper.
        \item The abstract and/or introduction should clearly state the claims made, including the contributions made in the paper and important assumptions and limitations. A No or NA answer to this question will not be perceived well by the reviewers. 
        \item The claims made should match theoretical and experimental results, and reflect how much the results can be expected to generalize to other settings. 
        \item It is fine to include aspirational goals as motivation as long as it is clear that these goals are not attained by the paper. 
    \end{itemize}

\item {\bf Limitations}
    \item[] Question: Does the paper discuss the limitations of the work performed by the authors?
    \item[] Answer: \answerYes{}
    \item[] Justification: We mention limitations multiple times in the main body of our work: after introducing the assumptions in Section \ref{sec:preliminaries},
            we mention that spectral gaps may not be verified for heavy tail distributions, or act numerically unfavorably for multimodal distributions,
            and that constants may be large in practice.
            In the discussion (Section \ref{sec:discussion}), we reflect back on these limitations, and we mention the limitations of our considered model
            class (unnormalized exponential families).
    \item[] Guidelines:
    \begin{itemize}
        \item The answer NA means that the paper has no limitation while the answer No means that the paper has limitations, but those are not discussed in the paper. 
        \item The authors are encouraged to create a separate "Limitations" section in their paper.
        \item The paper should point out any strong assumptions and how robust the results are to violations of these assumptions (e.g., independence assumptions, noiseless settings, model well-specification, asymptotic approximations only holding locally). The authors should reflect on how these assumptions might be violated in practice and what the implications would be.
        \item The authors should reflect on the scope of the claims made, e.g., if the approach was only tested on a few datasets or with a few runs. In general, empirical results often depend on implicit assumptions, which should be articulated.
        \item The authors should reflect on the factors that influence the performance of the approach. For example, a facial recognition algorithm may perform poorly when image resolution is low or images are taken in low lighting. Or a speech-to-text system might not be used reliably to provide closed captions for online lectures because it fails to handle technical jargon.
        \item The authors should discuss the computational efficiency of the proposed algorithms and how they scale with dataset size.
        \item If applicable, the authors should discuss possible limitations of their approach to address problems of privacy and fairness.
        \item While the authors might fear that complete honesty about limitations might be used by reviewers as grounds for rejection, a worse outcome might be that reviewers discover limitations that aren't acknowledged in the paper. The authors should use their best judgment and recognize that individual actions in favor of transparency play an important role in developing norms that preserve the integrity of the community. Reviewers will be specifically instructed to not penalize honesty concerning limitations.
    \end{itemize}

\item {\bf Theory Assumptions and Proofs}
    \item[] Question: For each theoretical result, does the paper provide the full set of assumptions and a complete (and correct) proof?
    \item[] Answer: \answerYes{}
    \item[] Justification: All of our theoretical results are formally proved in the supplementary material.
        In the main paper, we provide some proof sketches (discussion of the tail term in Section \ref{sec:offline:SGD},
        paragraphs after theorems in section \ref{sec:online_cd}) to provide high level intuitions behind the proofs.
        All of the theoretical results provided by our work are proved in the supplementary material,
        usually by invoking (and referencing) more atomic lemmas, also proved in the supplementary material.
        All of our theorems and lemmas are stated with the assumptions they require. 
    \item[] Guidelines:
    \begin{itemize}
        \item The answer NA means that the paper does not include theoretical results. 
        \item All the theorems, formulas, and proofs in the paper should be numbered and cross-referenced.
        \item All assumptions should be clearly stated or referenced in the statement of any theorems.
        \item The proofs can either appear in the main paper or the supplemental material, but if they appear in the supplemental material, the authors are encouraged to provide a short proof sketch to provide intuition. 
        \item Inversely, any informal proof provided in the core of the paper should be complemented by formal proofs provided in appendix or supplemental material.
        \item Theorems and Lemmas that the proof relies upon should be properly referenced. 
    \end{itemize}

    \item {\bf Experimental Result Reproducibility}
    \item[] Question: Does the paper fully disclose all the information needed to reproduce the main experimental results of the paper to the extent that it affects the main claims and/or conclusions of the paper (regardless of whether the code and data are provided or not)?
    \item[] Answer: \answerNA{}
    \item[] Justification: The paper does not include experiments.
    \item[] Guidelines:
    \begin{itemize}
        \item The answer NA means that the paper does not include experiments.
        \item If the paper includes experiments, a No answer to this question will not be perceived well by the reviewers: Making the paper reproducible is important, regardless of whether the code and data are provided or not.
        \item If the contribution is a dataset and/or model, the authors should describe the steps taken to make their results reproducible or verifiable. 
        \item Depending on the contribution, reproducibility can be accomplished in various ways. For example, if the contribution is a novel architecture, describing the architecture fully might suffice, or if the contribution is a specific model and empirical evaluation, it may be necessary to either make it possible for others to replicate the model with the same dataset, or provide access to the model. In general. releasing code and data is often one good way to accomplish this, but reproducibility can also be provided via detailed instructions for how to replicate the results, access to a hosted model (e.g., in the case of a large language model), releasing of a model checkpoint, or other means that are appropriate to the research performed.
        \item While NeurIPS does not require releasing code, the conference does require all submissions to provide some reasonable avenue for reproducibility, which may depend on the nature of the contribution. For example
        \begin{enumerate}
            \item If the contribution is primarily a new algorithm, the paper should make it clear how to reproduce that algorithm.
            \item If the contribution is primarily a new model architecture, the paper should describe the architecture clearly and fully.
            \item If the contribution is a new model (e.g., a large language model), then there should either be a way to access this model for reproducing the results or a way to reproduce the model (e.g., with an open-source dataset or instructions for how to construct the dataset).
            \item We recognize that reproducibility may be tricky in some cases, in which case authors are welcome to describe the particular way they provide for reproducibility. In the case of closed-source models, it may be that access to the model is limited in some way (e.g., to registered users), but it should be possible for other researchers to have some path to reproducing or verifying the results.
        \end{enumerate}
    \end{itemize}

\item {\bf Open access to data and code}
    \item[] Question: Does the paper provide open access to the data and code, with sufficient instructions to faithfully reproduce the main experimental results, as described in supplemental material?
    \item[] Answer: \answerNA{}
    \item[] Justification: The paper does not include experiments, and thus no code or data.
    \item[] Guidelines:
    \begin{itemize}
        \item The answer NA means that paper does not include experiments requiring code.
        \item Please see the NeurIPS code and data submission guidelines (\url{https://nips.cc/public/guides/CodeSubmissionPolicy}) for more details.
        \item While we encourage the release of code and data, we understand that this might not be possible, so “No” is an acceptable answer. Papers cannot be rejected simply for not including code, unless this is central to the contribution (e.g., for a new open-source benchmark).
        \item The instructions should contain the exact command and environment needed to run to reproduce the results. See the NeurIPS code and data submission guidelines (\url{https://nips.cc/public/guides/CodeSubmissionPolicy}) for more details.
        \item The authors should provide instructions on data access and preparation, including how to access the raw data, preprocessed data, intermediate data, and generated data, etc.
        \item The authors should provide scripts to reproduce all experimental results for the new proposed method and baselines. If only a subset of experiments are reproducible, they should state which ones are omitted from the script and why.
        \item At submission time, to preserve anonymity, the authors should release anonymized versions (if applicable).
        \item Providing as much information as possible in supplemental material (appended to the paper) is recommended, but including URLs to data and code is permitted.
    \end{itemize}

\item {\bf Experimental Setting/Details}
    \item[] Question: Does the paper specify all the training and test details (e.g., data splits, hyperparameters, how they were chosen, type of optimizer, etc.) necessary to understand the results?
    \item[] Answer: \answerNA{}
    \item[] Justification: The paper does not include experiments.
    \item[] Guidelines:
    \begin{itemize}
        \item The answer NA means that the paper does not include experiments.
        \item The experimental setting should be presented in the core of the paper to a level of detail that is necessary to appreciate the results and make sense of them.
        \item The full details can be provided either with the code, in appendix, or as supplemental material.
    \end{itemize}

\item {\bf Experiment Statistical Significance}
    \item[] Question: Does the paper report error bars suitably and correctly defined or other appropriate information about the statistical significance of the experiments?
    \item[] Answer: \answerNA{}
    \item[] Justification: The paper does not include experiments.
    \item[] Guidelines:
    \begin{itemize}
        \item The answer NA means that the paper does not include experiments.
        \item The authors should answer "Yes" if the results are accompanied by error bars, confidence intervals, or statistical significance tests, at least for the experiments that support the main claims of the paper.
        \item The factors of variability that the error bars are capturing should be clearly stated (for example, train/test split, initialization, random drawing of some parameter, or overall run with given experimental conditions).
        \item The method for calculating the error bars should be explained (closed form formula, call to a library function, bootstrap, etc.)
        \item The assumptions made should be given (e.g., Normally distributed errors).
        \item It should be clear whether the error bar is the standard deviation or the standard error of the mean.
        \item It is OK to report 1-sigma error bars, but one should state it. The authors should preferably report a 2-sigma error bar than state that they have a 96\% CI, if the hypothesis of Normality of errors is not verified.
        \item For asymmetric distributions, the authors should be careful not to show in tables or figures symmetric error bars that would yield results that are out of range (e.g. negative error rates).
        \item If error bars are reported in tables or plots, The authors should explain in the text how they were calculated and reference the corresponding figures or tables in the text.
    \end{itemize}

\item {\bf Experiments Compute Resources}
    \item[] Question: For each experiment, does the paper provide sufficient information on the computer resources (type of compute workers, memory, time of execution) needed to reproduce the experiments?
    \item[] Answer: \answerNA{}
    \item[] Justification: The paper does not include experiments.
    \item[] Guidelines:
    \begin{itemize}
        \item The answer NA means that the paper does not include experiments.
        \item The paper should indicate the type of compute workers CPU or GPU, internal cluster, or cloud provider, including relevant memory and storage.
        \item The paper should provide the amount of compute required for each of the individual experimental runs as well as estimate the total compute. 
        \item The paper should disclose whether the full research project required more compute than the experiments reported in the paper (e.g., preliminary or failed experiments that didn't make it into the paper). 
    \end{itemize}
    
\item {\bf Code Of Ethics}
    \item[] Question: Does the research conducted in the paper conform, in every respect, with the NeurIPS Code of Ethics \url{https://neurips.cc/public/EthicsGuidelines}?
    \item[] Answer: \answerYes{}
    \item[] Justification: As a theory paper, we believe that our work does not breach the code of conduct.
        With this work, we aim to advance the understanding of statistically efficient algorithms, 
        a domain which has the potential to improve the overall \emph{computational} efficiency of machine learning algorithms.
    \item[] Guidelines:
    \begin{itemize}
        \item The answer NA means that the authors have not reviewed the NeurIPS Code of Ethics.
        \item If the authors answer No, they should explain the special circumstances that require a deviation from the Code of Ethics.
        \item The authors should make sure to preserve anonymity (e.g., if there is a special consideration due to laws or regulations in their jurisdiction).
    \end{itemize}

\item {\bf Broader Impacts}
    \item[] Question: Does the paper discuss both potential positive societal impacts and negative societal impacts of the work performed?
    \item[] Answer: \answerNA{} % Replace by \answerYes{}, \answerNo{}, or \answerNA{}.
    \item[] Justification: We believe that our work will not have the negative impacts mentioned 
        (e.g., disinformation, surveillance, fairness, privacy, security considerations). We believe that advancing the domain of statistical
        effiency has the potential to improve the overall computational efficiency of machine learning algorithms, which, all else left equal,
        could constitute a positive societal impact.
    \item[] Guidelines:
    \begin{itemize}
        \item The answer NA means that there is no societal impact of the work performed.
        \item If the authors answer NA or No, they should explain why their work has no societal impact or why the paper does not address societal impact.
        \item Examples of negative societal impacts include potential malicious or unintended uses (e.g., disinformation, generating fake profiles, surveillance), fairness considerations (e.g., deployment of technologies that could make decisions that unfairly impact specific groups), privacy considerations, and security considerations.
        \item The conference expects that many papers will be foundational research and not tied to particular applications, let alone deployments. However, if there is a direct path to any negative applications, the authors should point it out. For example, it is legitimate to point out that an improvement in the quality of generative models could be used to generate deepfakes for disinformation. On the other hand, it is not needed to point out that a generic algorithm for optimizing neural networks could enable people to train models that generate Deepfakes faster.
        \item The authors should consider possible harms that could arise when the technology is being used as intended and functioning correctly, harms that could arise when the technology is being used as intended but gives incorrect results, and harms following from (intentional or unintentional) misuse of the technology.
        \item If there are negative societal impacts, the authors could also discuss possible mitigation strategies (e.g., gated release of models, providing defenses in addition to attacks, mechanisms for monitoring misuse, mechanisms to monitor how a system learns from feedback over time, improving the efficiency and accessibility of ML).
    \end{itemize}
    
\item {\bf Safeguards}
    \item[] Question: Does the paper describe safeguards that have been put in place for responsible release of data or models that have a high risk for misuse (e.g., pretrained language models, image generators, or scraped datasets)?
    \item[] Answer: \answerNA{} % Replace by \answerYes{}, \answerNo{}, or \answerNA{}.
    \item[] Justification: Since our work does not include experiments nor data, and only analyzes existing models, we do not believe that our work requires any specific guidelines.
    \item[] Guidelines:
    \begin{itemize}
        \item The answer NA means that the paper poses no such risks.
        \item Released models that have a high risk for misuse or dual-use should be released with necessary safeguards to allow for controlled use of the model, for example by requiring that users adhere to usage guidelines or restrictions to access the model or implementing safety filters. 
        \item Datasets that have been scraped from the Internet could pose safety risks. The authors should describe how they avoided releasing unsafe images.
        \item We recognize that providing effective safeguards is challenging, and many papers do not require this, but we encourage authors to take this into account and make a best faith effort.
    \end{itemize}

\item {\bf Licenses for existing assets}
    \item[] Question: Are the creators or original owners of assets (e.g., code, data, models), used in the paper, properly credited and are the license and terms of use explicitly mentioned and properly respected?
    \item[] Answer: \answerNA{}
    \item[] Justification: The paper does not use existing assets.
    \item[] Guidelines:
    \begin{itemize}
        \item The answer NA means that the paper does not use existing assets.
        \item The authors should cite the original paper that produced the code package or dataset.
        \item The authors should state which version of the asset is used and, if possible, include a URL.
        \item The name of the license (e.g., CC-BY 4.0) should be included for each asset.
        \item For scraped data from a particular source (e.g., website), the copyright and terms of service of that source should be provided.
        \item If assets are released, the license, copyright information, and terms of use in the package should be provided. For popular datasets, \url{paperswithcode.com/datasets} has curated licenses for some datasets. Their licensing guide can help determine the license of a dataset.
        \item For existing datasets that are re-packaged, both the original license and the license of the derived asset (if it has changed) should be provided.
        \item If this information is not available online, the authors are encouraged to reach out to the asset's creators.
    \end{itemize}

\item {\bf New Assets}
    \item[] Question: Are new assets introduced in the paper well documented and is the documentation provided alongside the assets?
    \item[] Answer: \answerNA{}
    \item[] Justification: Our work does not release any assets.
    \item[] Guidelines:
    \begin{itemize}
        \item The answer NA means that the paper does not release new assets.
        \item Researchers should communicate the details of the dataset/code/model as part of their submissions via structured templates. This includes details about training, license, limitations, etc. 
        \item The paper should discuss whether and how consent was obtained from people whose asset is used.
        \item At submission time, remember to anonymize your assets (if applicable). You can either create an anonymized URL or include an anonymized zip file.
    \end{itemize}

\item {\bf Crowdsourcing and Research with Human Subjects}
    \item[] Question: For crowdsourcing experiments and research with human subjects, does the paper include the full text of instructions given to participants and screenshots, if applicable, as well as details about compensation (if any)? 
    \item[] Answer: \answerNA{} % Replace by \answerYes{}, \answerNo{}, or \answerNA{}.
    \item[] Justification: The paper does not involve crowdsourcing nor research with human subjects.
    \item[] Guidelines:
    \begin{itemize}
        \item The answer NA means that the paper does not involve crowdsourcing nor research with human subjects.
        \item Including this information in the supplemental material is fine, but if the main contribution of the paper involves human subjects, then as much detail as possible should be included in the main paper. 
        \item According to the NeurIPS Code of Ethics, workers involved in data collection, curation, or other labor should be paid at least the minimum wage in the country of the data collector. 
    \end{itemize}

\item {\bf Institutional Review Board (IRB) Approvals or Equivalent for Research with Human Subjects}
    \item[] Question: Does the paper describe potential risks incurred by study participants, whether such risks were disclosed to the subjects, and whether Institutional Review Board (IRB) approvals (or an equivalent approval/review based on the requirements of your country or institution) were obtained?
    \item[] Answer: \answerNA{} % Replace by \answerYes{}, \answerNo{}, or \answerNA{}.
    \item[] Justification: The paper does not involve crowdsourcing nor research with human subjects.
    \item[] Guidelines:
    \begin{itemize}
        \item The answer NA means that the paper does not involve crowdsourcing nor research with human subjects.
        \item Depending on the country in which research is conducted, IRB approval (or equivalent) may be required for any human subjects research. If you obtained IRB approval, you should clearly state this in the paper. 
        \item We recognize that the procedures for this may vary significantly between institutions and locations, and we expect authors to adhere to the NeurIPS Code of Ethics and the guidelines for their institution. 
        \item For initial submissions, do not include any information that would break anonymity (if applicable), such as the institution conducting the review.
    \end{itemize}

\end{enumerate}

\end{document}